\documentclass[twoside,11pt]{article}

\usepackage[abbrvbib,preprint]{jmlr2e}
\usepackage{jmlr2e_neutral_author}
\usepackage{xurl}
\usepackage{xcolor}
\usepackage{algorithm}
\usepackage[noend]{algpseudocode}
\usepackage{amsmath}
\usepackage{url}
\usepackage{subcaption}
\usepackage{bbm}
\usepackage{tikz}

\newtheorem{property}[theorem]{Property}
\newtheorem{assumption}[theorem]{Assumption}

\usepackage{algorithm}
\usepackage[noend]{algpseudocode}

  \usepackage{tikz}
  \usetikzlibrary{arrows.meta, positioning, fit, backgrounds, calc}
  
\definecolor{tealborder}{RGB}{20,95,85}
\definecolor{lighttealfill}{RGB}{206,232,226}
\definecolor{palettealfill}{RGB}{227,244,240}
\definecolor{goldfill}{RGB}{255,214,110}
\definecolor{goldborder}{RGB}{158,122,20}
\definecolor{browntext}{RGB}{140,68,20}
\definecolor{arrowpurple}{RGB}{140,40,140}
\usepackage{lastpage}

\firstpageno{1}

\begin{document}

\title{Model-free Bayesian Clustering and Multimodal Posterior Summarisation}


\author{
\vspace{2em}
\begin{tabular}{cc}
\begin{minipage}{0.3\textwidth}
\centering
Qing Li\\
School of Mathematics\\
University of Edinburgh\\
\texttt{qli10@ed.ac.uk}
\end{minipage}
&
\begin{minipage}{0.3\textwidth}
\centering
Sara Wade\\
School of Mathematics\\
University of Edinburgh\\
\texttt{sara.wade@ed.ac.uk}
\end{minipage}
\end{tabular}
}

\maketitle

\begin{abstract}%
We develop a model-free, noise-aware Bayesian clustering framework that constructs a martingale posterior over clusterings from a predictive Bayesian perspective. This avoids explicit likelihood-prior model specification, expensive MCMC sampling, and mode-seeking variational approximations,
while accommodating irregular cluster shapes and enabling parallel generation of posterior draws. 
To interpret the  posterior over high-dimensional, noise-aware clusterings, we introduce a Wasserstein-type summarization method. Using a novel noise-aware loss, this approach compresses the large collection of posterior draws into a small set of weighted representatives, capturing distinct cluster structures and noise assignments.
Further, we marginalise the uncertainty in the density level, avoiding reliance on a specific choice.
We establish theoretical properties of the proposed loss and guarantees for the level-marginal posterior approximation. Results demonstrate the framework's effectiveness in capturing clustering and noise uncertainty.
\end{abstract}


\section{Introduction}
Clustering is a fundamental task in unsupervised learning and exploratory data analysis, and a wide variety of methods has been developed across statistics and machine learning. Broadly, it aims to group observations so that those within the same cluster are more similar than those in different clusters. A key limitation of many classical methods such as k-means and hierarchical clustering is that they return a single point-estimate partition with no uncertainty quantification. In contrast, Bayesian clustering addresses this by constructing a posterior distribution on the partition structure \citep{wade2023bayesian}. Since the number of clusters is rarely known in advance, several classes of random-partition models have been developed \citep{Pitman95, richardson1997bayesian,miller2018mixture}. Nevertheless, when the component family is misspecified, a single non-Gaussian or non-convex cluster may be represented by several components, resulting in cluster splitting \citep{miller2019robust,cai2021finite,buch2024bayesian}. 
In parallel,  generalized Bayes  approaches \citep{rigon2023generalized} replace the likelihood with a loss function for clustering and yield Gibbs posteriors, although these methods still require prior specification and calibration of the learning-rate parameter. 

\begin{figure}
    \centering
\includegraphics[width=0.8\linewidth]{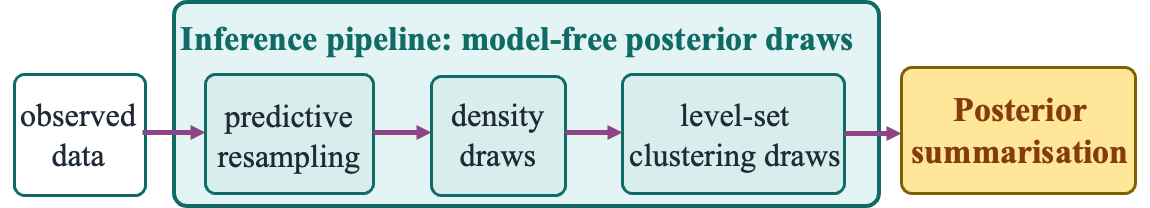}
    \caption{A model-free noise-aware clustering framework with posterior summarisation. }
\label{fig:pipeline}
\end{figure}

A further complication, and a central feature of this paper, arises when the data contain noise.
This occurs, for example, in sensor-based multiple object tracking \citep{li2023adaptive}, where environmental clutter generates false alarms unrelated to physical targets. Similarly, in astronomical sky surveys, 
galaxy clusters are embedded within large populations of unrelated stars and diffuse background light \citep{jang2006nonparametric}. 
In such settings, methods such as $k$-means can be particularly sensitive to outliers and noise, as forcing every observation into a full partition may absorb background observations into genuine clusters or create artificial, spurious groups.
Algorithmic density-based methods such as DBSCAN \citep{schubert2017dbscan} can identify clusters while leaving low-density observations as noise. However, they return only a single point-estimate of the partition and noise structure, 
can be sensitive to tuning parameters, and deteriorate when clusters have varying densities or when observations lie near the boundary between clusters and noise. An alternative is Bayesian level-set clustering \citep[BALLET,][]{buch2024bayesian}, which in contrast to the fixed density-based rule of DBSCAN, treats clustering as a functional of an unknown density and propagates posterior uncertainty in the density to the induced clustering. However, BALLET still requires a Bayesian model for the density $f$, so its induced posterior on clustering inherits the modelling assumptions, hyperparameter choices, and computational demands of the chosen model for $f$.

Beyond posterior computation, another challenge in Bayesian clustering is how to summarise the resulting posterior draws. In practice, inference may produce tens of thousands of clustering draws, while the underlying partition space is discrete, combinatorially enormous, and potentially multimodal \citep{balocchi2025understanding,li2023scalable}. Existing approaches \citep{dahl2022search,buch2024bayesian} typically report a single Bayes estimator, supplemented by tools such as a posterior similarity matrix or credible ball. 
Recent work by \citet{page2025uncertainty} constructs interpretable credible sets through high-probability common sub-partitions; 
however, the underlying posterior remains supported on full partitions and does not include an explicit inactive/noise state.
At the density-model level, \citet{bolfarine2025lower} project a flexible,  overparameterised  posterior onto a lower-dimensional parametric surrogate, preserving uncertainty about the simplified density but not directly targeting distinct modes in the posterior. By contrast, \cite{balocchi2025understanding} proposed the  Wasserstein Approximation for Bayesian Inference (WASABI) that operates directly on the partition posterior, approximating it with a small weighted set of representative clusterings, though its original formulation is restricted to full partitions and therefore cannot represent uncertainty in noise assignment.

Motivated by these limitations, this paper develops a model-free, noise-aware clustering pipeline that delivers uncertainty quantification while preserving the principal multimodal structure of the posterior. 
The general scheme, in Figure~\ref{fig:pipeline}, separates density learning from clustering, where clustering uncertainty is induced by uncertainty in the predictive distribution. 
It is robust to irregular cluster shapes and requires no  mixture-model specification, allowing greater flexibility than other commonly used approaches
(see Figure~\ref{fig:sim2d} for an example).

\begin{figure*}[tp]
    \centering
    \includegraphics[width=1\linewidth]{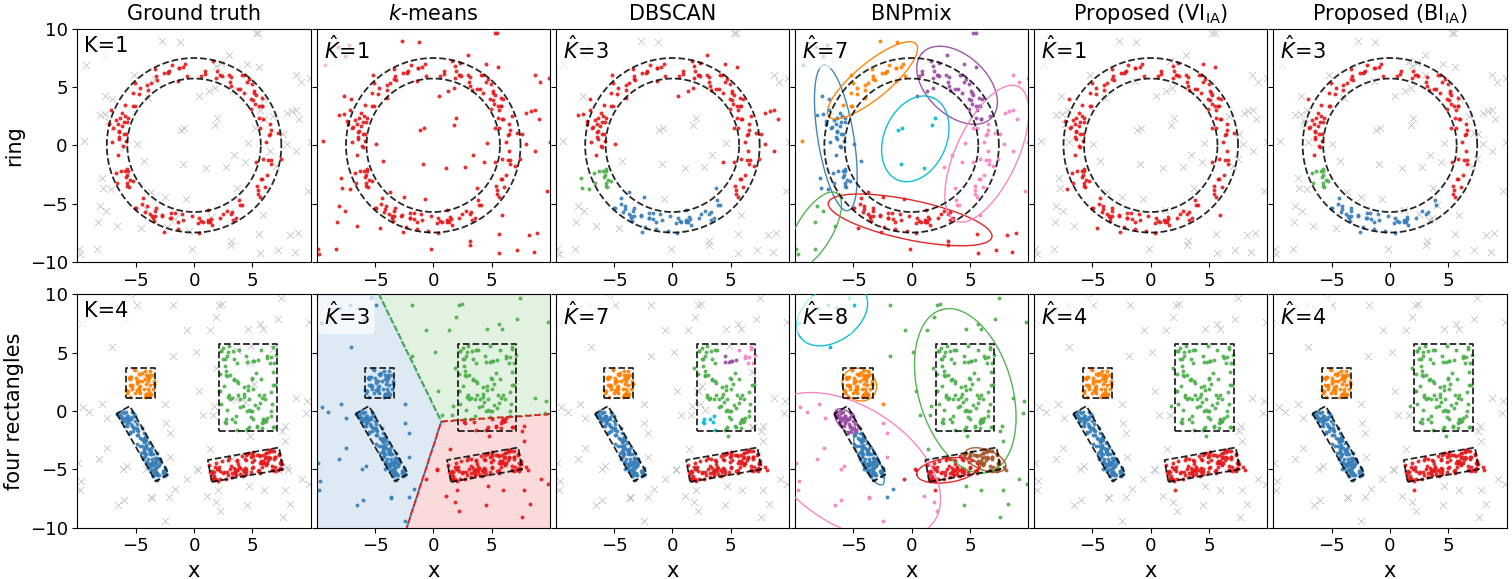}
    \caption{Point-estimate clusterings for the ring (top) and four-rectangle (bottom) examples. 
    Colours indicate active clusters and grey crosses are inactive observations; dashed black boundaries show the true cluster regions. $\hat K$ is the estimated number of  clusters. BNPmix \citep{corradin2021bnpmix} shows the fitted Gaussian components. We present the Bayes estimate of the proposed MDBC under both the VI$_{\mathrm{IA}}$  and BI$_{\mathrm{IA}}$ losses in \eqref{eq:bayes_subpartition_estimator}.
    }
    \label{fig:sim2d}
\end{figure*}

On the \emph{inferential} side, the proposed framework builds on Bayesian predictive approaches \citep{fortini2023prediction,fortini2025exchangeability}, which interpret Bayesian uncertainty through missing future data. We adopt this perspective by constructing a nonparametric martingale posterior distribution (MPD) \citep{fong2023martingale} for clustering, in which predictive resampling allows parallelizable generation of random density realisations, each of which is mapped deterministically to a sub-partition via a sample-based level-set graph rule. This yields a posterior distribution over noise-aware clusterings without introducing a mixture model or prior on the partition structure, while avoiding 
expensive MCMC sampling or reliance on variational approximations \citep{gan2024variational}. 
\cite{rodriguez2025martingale} also extended martingale posteriors to finite-mixture clustering; however, in their empirical analyses, k-means is applied with the number of clusters fixed at the posterior mode, so all clustering draws have the same number of clusters and a common labelling across draws.
Moreover, here, predictive resampling for generating clustering draws is efficient, parallelisable, and modular, allowing off-the-shelf schemes to be readily incorporated \citep{bariletto2026scalable,lee2023martingale,ng2025tabmgp}.

On the \emph{posterior summarisation} side, a main contribution here, we extend WASABI \citep{balocchi2025understanding} to sub-partitions, denoted WASABI-SuP, which introduces new losses and Wasserstein-type summaries that explicitly account for inactive observations. Given sub-partition draws, WASABI-SuP captures distinct posterior modes arising from uncertainty in the clustering structure, the active set, or both. This is complementary to \citet{bariletto2026scalable}, who also combine MPDs with density-based clustering but return full partitions, with uncertainty summaries 
\citep{bariletto2025conformalized} likewise on the full-partition space. 
Moreover, WASABI-SuP is inference-agnostic and applies to sub-partition draws from conventional Bayes, generalised Bayes, or predictive resampling.

\section{Martingale posteriors on sub-partitions}
\label{sec:mpd_ballet}

\begin{figure*}[t]
    \centering
    \includegraphics[width=1\linewidth]{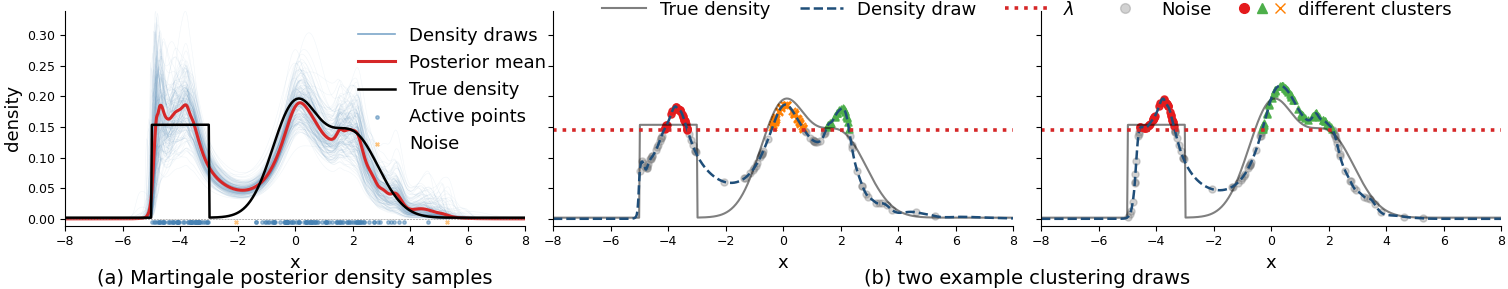}
    \caption{Illustration of predictive-resampling with density draws (blue lines) mapped to clustering draws by level-set rule.}
    \label{fig:mpd_clustering_draws}
\end{figure*}

This section concerns the stage of generating draws from a distribution on noise-aware clusterings, which bypasses using MCMC for sampling  or the need to resort to variational approximations. The proposed method, denoted MDBC, combines martingale posteriors with density-based clustering. 
As shown in the pipeline in Figure \ref{fig:pipeline}, the inferential construction consists of two steps:
the first step represents uncertainty about the unknown data-generating distribution through predictive density draws using predictive resampling from a martingale posterior;  the second step maps each density draw deterministically to a sample noise-aware clustering through a density level-set graph clustering rule. Thus, uncertainty in the resulting noise-aware clustering is induced by the martingale posterior on the underlying density. 
The resulting noise-aware clustering draws from the proposed MDBC form a Monte Carlo approximation to the posterior and are subsequently taken as input for posterior summarisation using credible sets, loss-based point estimates, or the WASABI-SuP in Section~\ref{sec:wasabi noise}.

\subsection{Density level-set graph clustering}\label{sec: graph}

Let $y_{1:n}=(y_1,\ldots,y_n)$ denote the observed data, with
$y_i\in\mathbb{R}^d$. For a candidate density $f$ and threshold
$\lambda>0$, define the density level set
\begin{equation}\notag
L_\lambda(f)
=\{y\in\mathbb{R}^d:f(y)\geq\lambda\}.
\end{equation}
Density level-set clustering regards the connected high-density regions of
$L_\lambda(f)$ as clusters, while observations outside the level set are
treated as inactive.

\paragraph{Sub-partitions.}
Let $[n]=\{1,\ldots,n\}$. A sub-partition
$\rho=\{C_1,\ldots,C_K\}$ consists of non-empty, pairwise disjoint subsets
satisfying $\bigcup_{h=1}^K C_h\subseteq[n]$, allowing some observations to
remain unclustered. We represent $\rho$ by an allocation vector
$\theta=(c_1,\ldots,c_n)\in\{0,1,\ldots,K\}^n$, where $c_i=h$ if
$i\in C_h$ and $c_i=0$ if observation $i$ is inactive.

\paragraph{Sample level-set graph clustering.}
Computing the connected components of the continuous level set
$L_\lambda(f)$ is generally difficult for $d\geq2$. We therefore use a
sample-based graph surrogate defined only on the observation points. For a candidate density $f$, define the
active set
$A_\lambda(f)
=
\{i\in[n]:f(y_i)\geq\lambda\}$.
We construct a graph $G_\lambda(f)$ on the vertices $A_\lambda(f)$ and
take its connected components as the active clusters, assigning label
$0$ to observations outside $A_\lambda(f)$. This defines the deterministic 
mapping $\psi_\lambda(f)$, which returns the vector of cluster allocations $\theta$ (for simplicity, dependence on the chosen graph-connectivity rule and its parameters is suppressed). 
We consider both the $\delta$-neighbourhood graph used in BALLET
\citep{buch2024bayesian} and a reciprocal-$k$NN alternative that mitigates
the chaining behaviour of radius-based connectivity and reduces over-merging (Figure~\ref{fig:rknn_vs_delta}).  
Formal definitions, parameter selection, and implementation details are in Appendix~\ref{app:levelset_graphs}.

\subsection{Predictive resampling for clustering}
\label{sec:mpd_intro}

Suppose that
$Y_{1:n}\overset{\mathrm{i.i.d.}}{\sim}F_0$,
where $F_0$ is an unknown data-generating distribution, and let
$Y_{1:n}=y_{1:n}$ denote the observed data.
Following the predictive view of Bayesian inference,
uncertainty is represented through the unobserved future sequence
$Y_{n+1:\infty}$. Starting from a predictive distribution $P_n$ fitted to the
observed data, predictive resampling sequentially generates future
observations from one-step-ahead predictives and updates the predictive
distribution after each imputation. Under suitable predictive-coherence
conditions, this sequence converges to a random limiting predictive
distribution $P_\infty$, whose distribution defines the martingale posterior.
Formal details are provided in Appendix~\ref{app:mpd_foundations}.

For density-based clustering, we work with the corresponding predictive
density. Suppose that $P_\infty$ is absolutely continuous with density
$p_\infty$. Using the level-set graph map $\psi_\lambda$ defined in
Section~\ref{sec: graph}, the limiting clustering is
$\theta_\infty
=
\psi_\lambda(p_\infty)$.
Hence, the martingale posterior on the predictive density induces a posterior
on sub-partitions through
\begin{equation}
\hspace{-0.5em}\Pi(\theta_\infty\in B | y_{1:n})
= \int \mathbf{1}\{ \psi_\lambda(p_\infty)\in B\} \mathrm{d}\Pi(p_\infty| y_{1:n})
\label{eq:clustering_mpd}
\end{equation}
where $B\subseteq\mathcal{S}_n$, and $\mathcal{S}_n$ denotes the space of sub-partitions of $[n]$.
Thus, we do not construct a posterior directly on cluster allocations: a predictive-density draw is mapped deterministically to a sub-partition, so clustering uncertainty is induced by uncertainty in the predictive
density.

\paragraph{Finite predictive resampling.}
To get posterior draws of clusterings, we repeat predictive resampling independently by: 1) sampling predictive density draws from the martingale posterior; and 2) mapping each density draw deterministically to a sub-partition through the level-set graph. 
In practice, predictive resampling is truncated at a sufficiently large
horizon $N$. Starting from $P_n$, each resampling path recursively
generates $Y_{i+1}\sim P_i$,
$i=n,\ldots,N-1$, and updates $P_i$ after each imputed observation. The resulting terminal
predictive $P_N$ has density $p_N$, which provides a finite approximation to the limiting predictive density $p_\infty$.

We construct the predictive sequence using the recursive Gaussian-copula
update of \citet{fong2023martingale}, described in
Appendix~\ref{app:mpd_gausscop}. The update is sequential, while independent
predictive-resampling paths can be run in parallel. The density-generation
step is also modular: alternative coherent predictive schemes, including the score-based construction \citep{bariletto2026scalable},  neural predictive processes \citep{lee2023martingale}, and transformer-based TabPFN \citep{ng2025tabmgp}, may be
substituted without changing the downstream level-set clustering or posterior summarisation procedure.

Repeating predictive resampling independently $M$ times gives terminal
density draws $p_N^{(m)}
$, $ m=1,\ldots,M$ (Figure~\ref{fig:mpd_clustering_draws}a).
Each draw is then mapped by the level-set graph rule (Figure~\ref{fig:mpd_clustering_draws}b):
\begin{equation}
\theta^{(m)}
=
\psi_\lambda
\left(
p_N^{(m)}
\right),
\qquad
m=1,\ldots,M.
\label{eq:clustering_draws}
\end{equation}
Because the predictive-resampling paths are generated independently, $\{\theta^{(m)}\}_{m=1}^M$ are independent draws from the finite-horizon distribution induced by $p_N$. Hence, for sufficiently large $N$, the empirical distribution of $\{\theta^{(m)}\}_{m=1}^M$ provides a Monte Carlo approximation to the martingale posterior on sub-partitions in \eqref{eq:clustering_mpd}. 
Variation across these draws captures uncertainty both in which observations are active and in how the active observations are clustered. 
These sub-partition draws form the input to the posterior summarisation framework
developed in Section~\ref{sec:wasabi noise}.

\section{Multimodal posterior summarisation and uncertainty quantification under noise}
\label{sec:wasabi noise}

Let $\theta^{(1)},\ldots,\theta^{(M)}\in\mathcal S_n$ denote draws from a posterior $\pi(\theta\mid y_{1:n})$ on sub-partitions, which may be obtained from the predictive-resampling in Section~\ref{sec:mpd_intro}, or from other Bayesian methods that return posterior sub-partition draws. A standard decision-theoretic summary is the Bayes estimator under a loss $D(\theta,\hat\theta)$,
\begin{equation}
\theta^{*}
=\operatorname*{arg\,min}_{\hat{\theta}\in\mathcal S_n}
\mathbb{E}
[D(\theta,\hat{\theta})| y]
\approx
\operatorname*{arg\min}_{\hat{\theta}\in\mathcal S_n}
\frac{1}{M} \sum_{m=1}^{M} D(\theta^{(m)},\hat{\theta}).
\label{eq:bayes_subpartition_estimator}
\end{equation}
Different choices of the loss $D$ encode different notions of clustering discrepancy. For full partitions, commonly used choices include Binder's loss \citep{binder1978bayesian}, based on pairwise co-clustering disagreements, and the variation of information (VI) \citep{meilua2007comparing}, based on entropy and mutual information. Their definitions are reviewed in Appendix~\ref{app:loss function}.

\paragraph{Limitation I: standard partition losses do not identify noise.}
When inactive observations are present, treating label $0$ as an ordinary cluster label loses its special interpretation. For example, $\theta=(0,1,1,1)$ and $\hat\theta=(1,0,0,0)$ have complementary active sets but induce the same unlabeled full partition $\{\{1\},\{2,3,4\}\}$ if label $0$ is treated as an ordinary cluster; standard Binder and VI losses therefore assign zero discrepancy. Noise-aware comparison must instead distinguish disagreement in active status from disagreement in the clustering of observations that are active in both sub-partitions. \citet{buch2024bayesian} introduced the inactive--active Binder (IA-Binder) loss for this purpose. Motivated by the different split--merge behaviour of Binder and VI \citep{wade2018bayesian}, we additionally introduce an inactive--active extension of VI in Section~\ref{subsec:ia_vi_losses}.

\paragraph{Limitation II: a single Bayes estimate cannot represent posterior multimodality.}
Even with an appropriate noise-aware loss, \eqref{eq:bayes_subpartition_estimator} reduces the posterior to one sub-partition. This can obscure qualitatively distinct posterior modes arising from alternative active sets, alternative clustering structures, or both. We therefore seek a compact posterior summary comprising several weighted representative sub-partitions.

\subsection{WASABI-SuP: multimodal summarisation of sub-partition posteriors}
\label{subsec:wasabi_sup}

We propose WASABI-SuP, extending WASABI \citep{balocchi2025understanding} from posteriors on full partitions to sub-partitions. Rather than replacing the posterior by a single Bayes estimate, WASABI-SuP approximates $\pi(\theta\mid y_{1:n})$ by an $L$-particle distribution
\begin{equation}
q(\theta)
=
\sum_{\ell=1}^{L}
w_\ell\,\delta_{\theta_\ell}(\theta),
\quad
w_\ell\geq0,
\quad
\sum_{\ell=1}^{L}w_\ell=1,
\notag
\end{equation}
and $\theta_\ell\in\mathcal S_n$. 
For a non-negative loss $D$ on $\mathcal S_n$, define the transport-type discrepancy
\begin{equation}
\mathsf T_D(\pi,q)
=
\inf_{\gamma\in\Gamma(\pi,q)}
\sum
D(\theta,\theta')\,\gamma(\theta,\theta'),
\notag
\end{equation}
where $\Gamma(\pi,q)$ is the set of couplings of $\pi$ and $q$. 

WASABI-SuP aims to solve  $\min_{q\in\mathcal Q_L}\mathsf T_D(\pi,q)$, 
where $\mathcal Q_L$ is the class of probability measures supported on at most $L$ sub-partitions. When $D$ is a metric, $\mathsf T_D$ is the corresponding $1$-Wasserstein distance denoted by $W_D$. For a general non-negative loss, including VI$_{\mathrm{IA}}$ in \eqref{eq:ia_vi_lin}, the same construction remains a well-defined loss-based optimal-transport compression objective.

In practice, we only have the empirical posterior
\begin{equation}
\widehat\pi_M
=
\sum_{m=1}^{M}\eta_m\,\delta_{\theta^{(m)}},
\quad
\eta_m\geq0,
\quad
\sum_{m=1}^{M}\eta_m=1.
\label{eq:weighted_empirical_posterior}
\end{equation}

For the proposed predictive resampling, or MCMC, the weights are $\eta_m=1/M$; allowing general weights also accommodates the level-marginal construction of Section~\ref{sec:wasabi-lambda}. 
The corresponding empirical WASABI-SuP solves 
\begin{equation}
q^*
\in
\operatorname*{arg\,min}_{q\in\mathcal Q_L}
\mathsf T_D\!\left(\widehat\pi_M,q\right)
\notag
\end{equation} 
using properties of transport-type discrepancies \citep{cuturi2014fast}. Specifically,  given the particles $\theta_\ell^*$,  $\ell=1,\ldots,L$, 
define their Voronoi cells by
\begin{equation}
\mathcal N_\ell
=
\left\{
\theta^{(m)}:
D(\theta^{(m)},\theta_\ell^*)
\leq
D(\theta^{(m)},\theta_j^*)
\ \text{for all }j\neq\ell
\right\},
\notag
\end{equation}
with ties resolved arbitrarily to one cell.
The particle weight and within-region loss are
\begin{equation}
w_\ell=\hspace{-0.5em}\sum_{\theta^{(m)}\in\mathcal N_\ell}\hspace{-0.5em}\eta_m,
\quad
l_\ell
=
\frac{1}{w_\ell}
\sum_{\theta^{(m)}\in\mathcal N_\ell}
\hspace{-0.5em}\eta_m\,
D(\theta^{(m)},\theta_\ell^*),  
\label{eq:wasabi_region_weight_loss}
\end{equation}
for $ w_\ell>0$, with the convention $l_\ell=0$ when $w_\ell=0$.
Hence the empirical compression objective is
\begin{align}
\widehat{\mathcal T}_D &= \mathsf T_D (\widehat\pi_M, q^*) =
\sum_{\ell=1}^{L}w_\ell l_\ell
=
\sum_{m=1}^{M}  \eta_m
\min_{1\leq\ell\leq L} 
D(\theta^{(m)},\theta_\ell^*).
\label{eq:wasabi_sup_objective}
\end{align}

Thus, each particle represents one region of the posterior in the sub-partition space, $w_\ell$ measures the posterior mass of  particle $\ell$, and $l_\ell$ quantifies the residual posterior spread within the region. For equal posterior weights, $w_\ell=|\mathcal N_\ell|/M$. When $L=1$, \eqref{eq:wasabi_sup_objective} reduces to 
\eqref{eq:bayes_subpartition_estimator}; WASABI-SuP therefore generalises the usual Bayes estimator from one global representative to several locally representative posterior summaries.

\paragraph{Optimisation.}
We use an alternating $k$-means-like procedure to find the particles $\theta_1^*,\ldots,\theta_L^*$: i) given the particles, the $\mathcal N$-update assigns each posterior draw to its closest particle under $D$, and ii) given the Voronoi cells, each particle $\ell=1,\ldots,L$ is updated by finding the local Bayes estimator
\begin{equation}
\theta_\ell^*
\leftarrow
\operatorname*{arg\,min}_{\theta\in\mathcal S_n}
\sum_{\theta^{(m)}\in\mathcal N_\ell}
\eta_m\,
D(\theta^{(m)},\theta),
\label{eq:particle_update}
\end{equation}
which we approximate using a modified SALSO search \citep{dahl2022search}. Hence each WASABI-SuP particle is a Bayes representative of a local posterior region.
The two steps are repeated until the change in \eqref{eq:wasabi_sup_objective} is below a prescribed threshold. Pseudocode, 
initialisation strategies, and complexity analysis are given in Appendix~\ref{app:wasabi_initialisation} and Appendix~\ref{app:complexity}.

\subsection{Noise-aware losses}
\label{subsec:ia_vi_losses}

WASABI-SuP requires a loss that respects the active--inactive disagreements. One available choice is IA-Binder (BI$_{\mathrm{IA}}$) \citep{buch2024bayesian}, which augments Binder's pairwise loss with explicit active--inactive penalties; its definition is recalled in Appendix~\ref{app:loss function}. We additionally propose an inactive--active variation of information (IA-VI) loss, denoted as VI$_{\mathrm{IA}}$, providing an information-theoretic alternative for sub-partitions. The general construction of WASABI-SuP can be easily adapted for alternative losses on sub-partitions.

Consider two sub-partitions $\theta$ and $\hat\theta$. Define their active and inactive sets 
$A=\{i:c_i\neq0\}$,
$\hat A=\{i:\hat c_i\neq0\}$,
$I=[n]\setminus A$,
$\hat I=[n]\setminus\hat A$,
and let $B=A\cap\hat A$ be the common active set. Write $\rho_B$ and $\hat\rho_B$ for the ordinary partitions induced by $\theta$ and $\hat\theta$ on $B$.
For $n\geq2$, define
\begin{align}
\mathrm{VI}_{\mathrm{IA}}(\theta,\hat\theta)
&=
\frac{\tau\log_2 n}{n}
\left(
|A\cap\hat I|
+
|I\cap\hat A|
\right)
+
\mathrm{VI}(\rho_B,\hat\rho_B).
\label{eq:ia_vi_lin}
\end{align}
where the first term measures the active-inactive disagreement of observations and $\tau>0$ is a common mismatch penalty; an asymmetric extension allowing different penalties for active-to-inactive and inactive-to-active mismatches is given in 
Appendix~\ref{app:IA-VI loss}. The second term $\mathrm{VI}(\rho_B,\hat\rho_B)$ compares the clustering structure conditional on being active in both sub-partitions, and  explicit expressions are given in 
Appendix~\ref{app:IA-VI loss}.
The factor $\log_2 n/n$ places the activity penalty on the same logarithmic scale as VI. 

\begin{proposition}[Basic properties of VI$_{\mathrm{IA}}$]
\label{prop:ia_vi_lin_basic}
For $\tau>0$,
$\mathrm{VI}_{\mathrm{IA}}$ is non-negative, symmetric, and invariant to relabelling of positive cluster labels. Moreover,
\[
\mathrm{VI}_{\mathrm{IA}}(\theta,\hat\theta)=0
\quad\Longleftrightarrow\quad
\theta=\hat\theta
\]
up to relabelling of the positive labels. If $A=\hat A$, then
\begin{equation}
\mathrm{VI}_{\mathrm{IA}}(\theta,\hat\theta)
=
\mathrm{VI}(\rho_A,\hat\rho_A),
\label{eq:iavi_reduces_vi}
\end{equation}
so VI$_{\mathrm{IA}}$ reduces exactly to ordinary VI when the active sets coincide. Finally,
\begin{equation}
0
\leq
\mathrm{VI}_{\mathrm{IA}}(\theta,\hat\theta)
\leq
(\tau+1)\log_2 n.
\label{eq:iavi_bound}
\end{equation}
\end{proposition}
A tight upper bound (compared to  \eqref{eq:iavi_bound}), which 
characterizes the most distant sub-partitions, together with a more general asymmetric extension allowing distinct active-to-inactive and inactive-to-active penalties, is given in Appendix~\ref{app:iavi_properties}. 
Beyond these basic properties, Appendix~\ref{app:comparative_properties}
develops the geometry of VI$_{\mathrm{IA}}$ on the sub-partition lattice and shows
that the loss satisfies an additive meet identity. We further compare
VI$_{\mathrm{IA}}$ and BI$_{\mathrm{IA}}$ under several canonical sub-partition configurations
in Appendix~\ref{app:comparative_properties}. These results show that Binder-type losses can favour splitting,
whereas VI$_{\mathrm{IA}}$ gives a more balanced treatment of splitting and merging. 
The analysis also explores the role of the active--inactive penalty $\tau$ in defining the discrepancy between different sub-partition structures, showing that $\tau=1$ is a reasonable default choice.

\paragraph{Posterior expected VI$_{\mathrm{IA}}$ loss.}
VI$_{\mathrm{IA}}$ can be used both in the global Bayes estimator
\eqref{eq:bayes_subpartition_estimator} and in the local particle
update \eqref{eq:particle_update}. Unlike BI$_{\mathrm{IA}}$, whose posterior
expected loss can be expressed through marginal activity probabilities
and pairwise active co-clustering probabilities, VI$_{\mathrm{IA}}$ is not pairwise
additive: its clustering component depends on the full contingency
table between the sampled and candidate partitions on their common
active set. A decomposition into an activity-penalty term and a
restricted-VI term, together with weighted Monte Carlo evaluation for
both global and local risks, is given in
Appendix~\ref{app:expected_iavi}.

\paragraph{Losses used in this work.}
We consider both VI$_{\mathrm{IA}}$ and BI$_{\mathrm{IA}}$ as losses within WASABI-SuP, allowing the effect of the underlying loss geometry on posterior compression to be examined explicitly. The general construction of WASABI-SuP can be easily adapted for alternative losses on sub-partitions.

\subsection{WASABI-SuP with level uncertainty}
\label{sec:wasabi-lambda}

The level set approach fixes a density threshold $\lambda$. When no preferred level is available, different plausible values of $\lambda$ can induce different active sets and clustering structures. Conditioning on a single level may therefore understate an important source of clustering uncertainty. We account for this by placing a distribution $H$ on a prespecified range $\Lambda\subset(0,\infty)$ of plausible density levels and marginalising over $\lambda$.

For a density $f$ and level $\lambda$, let
$\theta_\lambda
=
\psi_\lambda(f)
\in\mathcal S_n$
denote the sample level-set graph sub-partition defined in Section~\ref{sec: graph}. If $\Pi_n(df\mid y_{1:n})$ denotes a posterior or martingale-posterior distribution over densities, the fixed-level clustering posterior is the pushforward
\begin{equation}
\Pi_{n,\lambda}
=
\int
\delta_{\theta_\lambda}\,
\Pi_n(df\mid y_{1:n}).
\notag
\end{equation}
We define the \emph{level-marginal clustering posterior} by
\begin{equation}
\Pi_n^\Lambda
=
\int_\Lambda \Pi_{n,\lambda}H(d\lambda)
=
\int_\Lambda\int
\delta_{\theta_\lambda}
\Pi_n(df|y_{1:n})H(d\lambda).
\notag
\end{equation}
Equivalently, a draw from $\Pi_n^\Lambda$ is obtained by drawing $f\sim\Pi_n(\cdot\mid y)$ and $\lambda\sim H$, then computing $\theta_\lambda = \psi_\lambda(f)$. Thus, $\Pi_n^\Lambda$ propagates uncertainty in both the density and the level into uncertainty in the active set and clustering structure. The fixed-level formulation is recovered as the special case $H=\delta_{\lambda_0}$.

\paragraph{Empirical approximation.}
Given density draws $f^{(1)},\ldots,f^{(M)}$ and a weighted level grid $\{(\lambda_g,\omega_g)\}_{g=1}^{G}$ with $\sum_g\omega_g=1$, define
$\theta^{(m,g)}
=\psi_{\lambda_g}\!\left(f^{(m)}\right)$,
$m=1,\ldots,M$,
$g=1,\ldots,G$.
The corresponding empirical level-marginal posterior is
\begin{equation}
\widehat\pi_{M,G}^{\Lambda}
=
\frac{1}{M}
\sum_{m=1}^{M}
\sum_{g=1}^{G}
\omega_g\,
\delta_{\theta^{(m,g)}}.
\label{eq:hat-Pi-Lambda}
\end{equation}
Alternatively, density--level pairs may be sampled directly from $\Pi_n(\cdot\mid y_{1:n})\otimes H$, giving equally weighted sub-partition draws.

Because \eqref{eq:hat-Pi-Lambda} is simply a weighted empirical distribution on $\mathcal S_n$, the WASABI-SuP solution is:
\begin{equation}
\widehat q_L^{\,\Lambda}
\in
\operatorname*{arg\,min}_{q\in\mathcal Q_L}
\mathsf T_D\!\left(
\widehat\pi_{M,G}^{\Lambda},q
\right),
\notag
\end{equation}
with atom weights $\eta_{m,g}=\omega_g/M$ in \eqref{eq:weighted_empirical_posterior}. Hence, level marginalisation changes the posterior supplied to WASABI-SuP, but not the  algorithm itself. Importantly, a resulting particle need not correspond to a particular value of $\lambda$. A single Voronoi cell may collect posterior mass from several density--level combinations whenever they induce similar sub-partitions. WASABI-SuP therefore compresses uncertainty jointly over density estimation, level selection, active-set membership, and clustering structure.

\paragraph{Theoretical guarantees.}
For the theoretical analysis, we distinguish the computational sample-graph
clustering from the population level-set clustering.
Let $P_{0,n}^{\Lambda}$ denote the level-marginal distribution of the
sub-partitions of the observed data induced by the connected components of the
true level sets $L_\lambda(f_0)$, with $\lambda\sim H$; its precise definition
is given in Appendix~\ref{ssec:level-marginal-theory}. Under uniform sup-norm
contraction of the density posterior and a level-marginal stability condition
linking the sample-graph surrogate to the population level-set clustering, we
show that $\Pi_n^\Lambda$ converges to $P_{0,n}^\Lambda$ under the corresponding
bounded transport cost. When the ground loss is a metric, this yields
Wasserstein convergence. We further establish an oracle bound showing that the
empirical $L$-particle WASABI-SuP summary is asymptotically as accurate as the
best $L$-particle population approximation whenever the empirical
level-marginal measure consistently approximates its population target. For
VI$_{\mathrm{IA}}$, an analogous oracle result is obtained under uniform convergence of the
corresponding compression risk.

\paragraph{Practical  choice of the level distribution.}
For the fixed-level analysis, 
we avoid manual tuning of the  density level $\lambda$ or specification of the ground-truth 
noise fraction through a calibrated approach described  in Appendix~\ref{app:level-calibration}. 
For the marginal-level analysis, 
uncertainty in this calibrated choice is
represented by a distribution $H$ over plausible levels.  Our experiments use local grids around the calibrated choice, either directly on the density-level scale or through the inactive-fraction parameterisation. 
Additional choices for $H$ and grid constructions 
are in Appendix~\ref{APP:h}.

\section{Results}
\begin{table*}[tp]
\caption{Point-estimate clustering accuracy.
$\hat K$ and $\hat{\text{N}}\text{oise}$ are the estimated number of clusters and inactive data points, respectively.
Lower is better for the normalized $\widetilde{\mathrm{VI}}_{\mathrm{IA}}$ and
$\widetilde{\mathrm{BI}}_{\mathrm{IA}}$; higher is better for
$\mathrm{ARI}$. 
Point estimates of MDBC use the same posterior draws and differ in the loss ($\text{V}$ for VI$_{\mathrm{IA}}$ and $\text{B}$ for BI$_{\mathrm{IA}}$).}
\label{tab:results2d}
\begin{center}
\small
\setlength{\tabcolsep}{2pt}
\renewcommand{\arraystretch}{1.25}
\begin{tabular}{
l
ccccc
@{\hspace{8pt}}
ccccc
}
\hline
&
\multicolumn{5}{c}{
{Ring} $(n=250,\ K=1,\ \mathrm{Noise}=90)$
}
&
\multicolumn{5}{c}{
{Four rectangles} $(n=480,\ K=4,\ \mathrm{Noise}=86)$
}
\\
\cline{2-6}\cline{7-11}
\textbf{Metric}
& \textbf{$k$-means}
& \textbf{DBSCAN}
& \textbf{BNPmix}
& \textbf{MDBC$_{\text{V}}$}
& \textbf{MDBC$_{\text{B}}$}
& \textbf{$k$-means}
& \textbf{DBSCAN}
& \textbf{BNPmix}
& \textbf{MDBC$_{\text{V}}$}
& \textbf{MDBC$_{\text{B}}$}
\\
\hline
$\hat K$
& 1 & 3 & 7 & \textbf{1} & 3
& 3 & 7 & 8 & \textbf{4} & \textbf{4}
\\
$\hat{\text{N}}\text{oise}$
& 0 & 44 & 0 & 73 & 88
& 0 & 79 & 0 & 68 & 68
\\
$\widetilde{\mathrm{VI}}_{\mathrm{IA}} \downarrow$
& 0.239 & 0.225 & 0.428 & \textbf{0.082} & 0.178
& 0.174 & 0.050 & 0.158 & \textbf{0.024} & \textbf{0.024}
\\
$\widetilde{\mathrm{BI}}_{\mathrm{IA}} \downarrow$
& 0.360 & 0.396 & 0.684 & \textbf{0.124} & 0.292
& 0.273 & 0.062 & 0.236 & \textbf{0.037} & \textbf{0.037}
\\
$\mathrm{ARI} \uparrow$
& 0.000 & 0.183 & 0.019 & \textbf{0.558} & 0.324
& 0.483 & 0.859 & 0.604 & \textbf{0.917} & \textbf{0.917}
\\

\hline
\end{tabular}
\end{center}
\end{table*}

\subsection{Point-estimate clustering performance} \label{sec:sim2d} We first evaluate point-estimate clustering performance under noise; see Figure~\ref{fig:sim2d}. Both the \emph{ring} and the \emph{four rectangles} contain non-Gaussian clusters embedded in homogeneous clutter. The four-rectangle example additionally stresses heterogeneous density. 
We compare with $k$-means, DBSCAN, and a Pitman--Yor Gaussian mixture fitted using \texttt{BNPmix} \citep{corradin2021bnpmix}. For $k$-means, $K$ is selected by the gap statistic \citep{tibshirani2001estimating}; for DBSCAN, the neighbourhood radius adopts the $k$-distance curve; and for BNPmix, we use its VI-based clustering estimate. 
For the proposed MDBC, our default point estimate is the VI$_{\mathrm{IA}}$ Bayes estimator in
\eqref{eq:bayes_subpartition_estimator}, but to assess sensitivity to the choice of loss, we additionally compute the BI$_{\mathrm{IA}}$ Bayes estimate. 
We denote the two summaries by MDBC$_{\mathrm{V}}$ and MDBC$_{\mathrm{B}}$, respectively. 
We report the standard adjusted Rand index (ARI) \citep{hubert1985comparing} as a reference metric, though it does not explicitly distinguish active--inactive disagreement. Thus, 
we evaluate noise-aware clustering
accuracy using a normalised version of VI$_{\mathrm{IA}}$ in \eqref{eq:ia_vi_lin} and BI$_{\mathrm{IA}}$ in \eqref{eq:ia-binder} (see Appendix \ref{app:Normalised losses}); 
lower values indicate better recovery and $0$ corresponds to exact agreement.

As shown in Figure~\ref{fig:sim2d} and
Table~\ref{tab:results2d}, MDBC$_{\mathrm{V}}$ recovers the intended cluster structure in both cases and achieves the best accuracy. For the four-rectangle case, MDBC$_{\mathrm{V}}$ and
MDBC$_{\mathrm{B}}$ have the same $L=1$ sub-partition estimate, as is visible in Figure~\ref{fig:sim2d}; consequently, their metrics in Table~\ref{tab:results2d} are identical. In contrast, MDBC$_{\mathrm{B}}$ splits the ring into three clusters; this behaviour is consistent
with the split-favouring behavior of Binder loss  \citep{wade2023bayesian}. 
In contrast,  $k$-means forces all clutter into clusters, DBSCAN over-fragments the data, and BNPmix uses multiple elliptical components to represent the  non-Gaussian clusters. Additional details and comparisons with BALLET-type methods \citep{buch2024bayesian} are  in Appendix~\ref{app:sim2d}.

\subsection{Robustness analysis}
\label{sec:robustness}
\paragraph{Varying cluster separation.}
First, we highlight the utility of WASABI-SuP under clustering ambiguity. We
consider two Gaussian components with increasing
separation $d$ (see Figure~\ref{fig:separation-data} in Appendix~\ref{app:robustness-separation}). 
Figure~\ref{fig:separation}(a) shows the  VI$_{\mathrm{IA}}$ compression loss  in \eqref{eq:wasabi_sup_objective} for $L=1$ across values of 
$d$. This loss peaks at the intermediate separation of $d=2.5$ and
decreases when the components are clearly separated, indicating the
greatest posterior uncertainty when clusters are weakly separated. Figure~\ref{fig:separation}(b) shows the VI$_{\mathrm{IA}}$ compression loss as the number of
particles $L$ increases, suggesting the choice of the $L=2$ summary at the elbow. This $L=2$ summary at the ambiguous separation $d=2.5$ is shown in  Figure~\ref{fig:separation}(c)--(d), where the dominant particle ($w_1=0.879$) represents
a single merged cluster, whereas the second ($w_2=0.121$) represents a
two-cluster split. Thus, rather than forcing 
a single clustering, WASABI-SuP retains a compact summary of  multiple competing modes of clustering and quantifies their posterior  mass. Appendix~\ref{app:robustness-separation} gives more details on data generation  and the decomposition of  uncertainty into inactive status and active clustering. Robustness to increasing clutter is further examined in Appendix~\ref{app:Robustness to increasing clutter}; see Figures~\ref{fig:noise_level_robustness} and~\ref{fig:noise_level_data}.

\begin{figure}[tp!]
    \centering
    \includegraphics[width=0.6\linewidth]{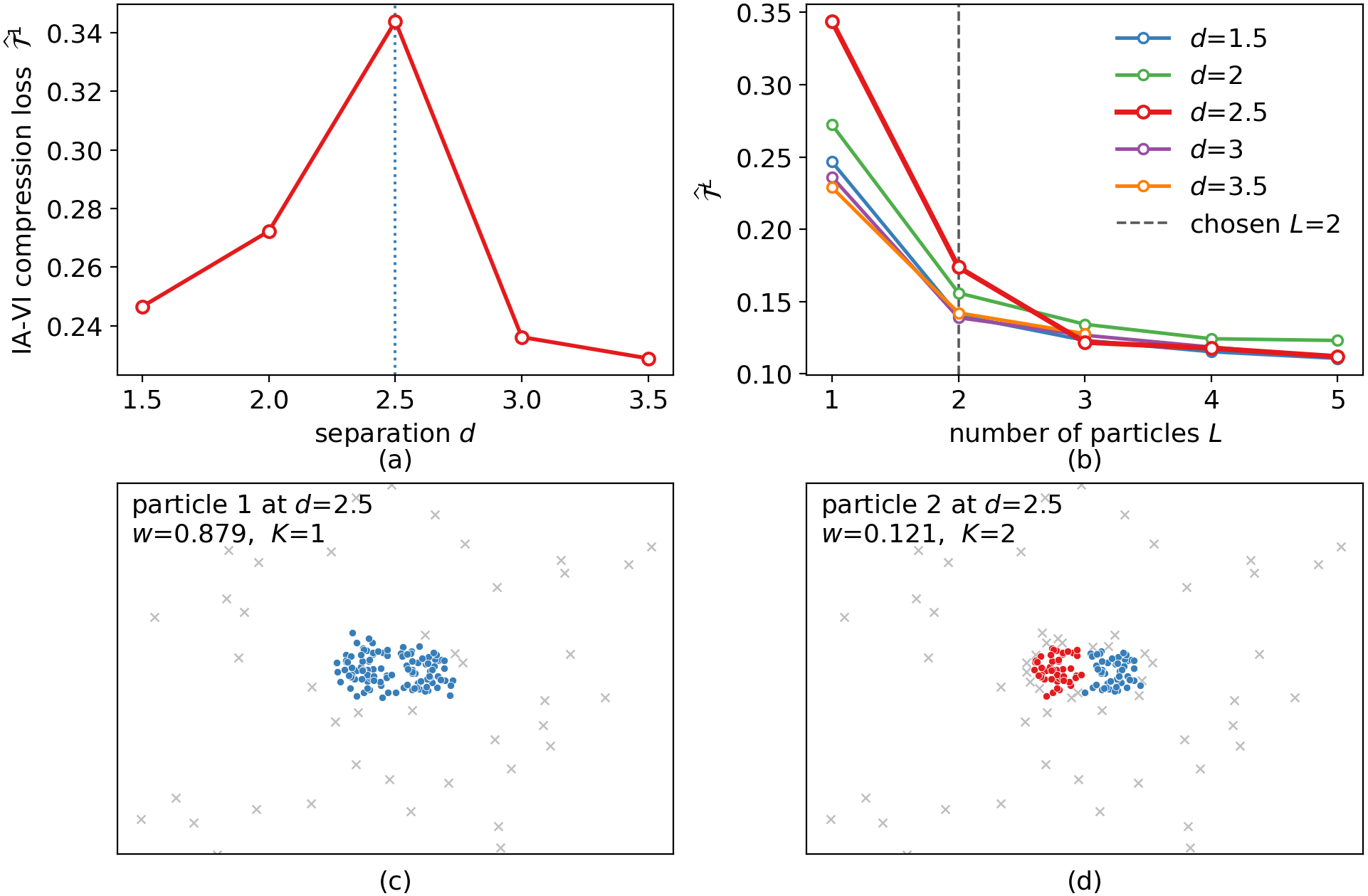}
    \caption{
(a) VI$_{\mathrm{IA}}$ compression loss 
for $L=1$  and increasing separation $d$. 
(b) Elbow plot for choosing $L$; the dashed line marks the selected $L=2$.
(c)--(d) The two particles at  $d=2.5$. Grey crosses are inactive observations.
}
    \label{fig:separation}
\end{figure}

\begin{figure}[htp]
    \centering
    \includegraphics[width=\linewidth]{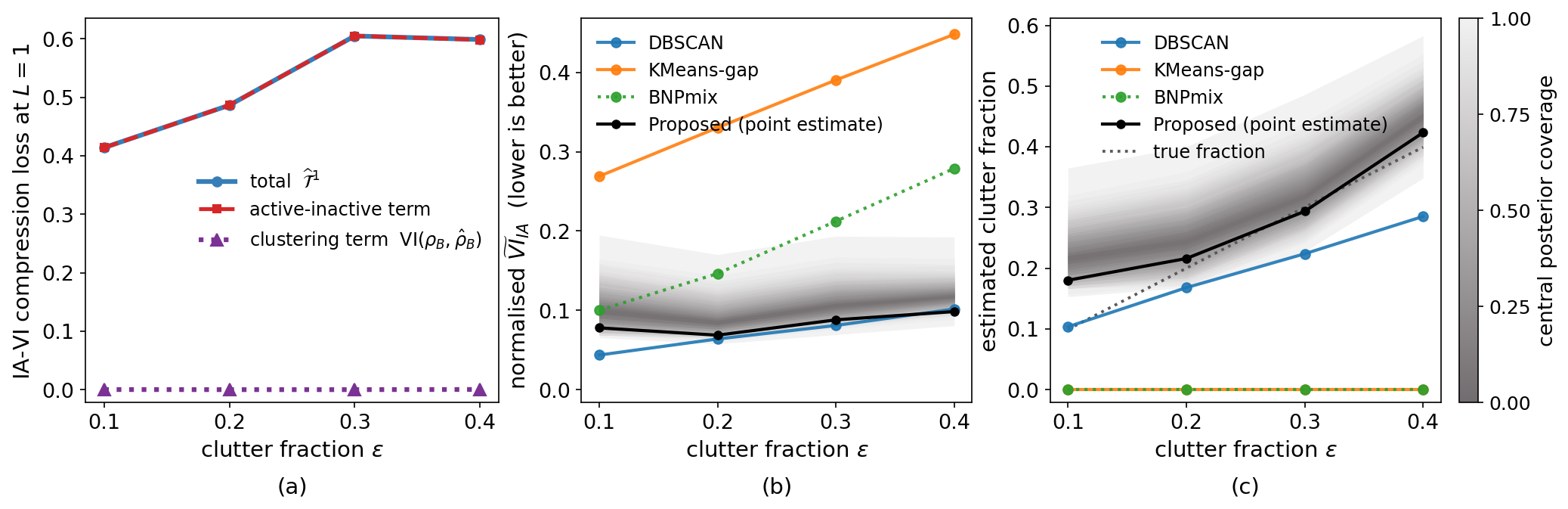}
    \caption{Robustness to increasing clutter.
    (a) One-particle $\mathrm{VI}_{\mathrm{IA}}$ compression loss and its
    decomposition into active--inactive and within-active clustering components.
    (b) Normalised $\widetilde{\mathrm{VI}}_{\mathrm{IA}}$ against the
    ground-truth clustering (lower is better).
    (c) Estimated inactive fraction, with the dotted line indicating the true
    clutter fraction.
    In (b)--(c), the black line denotes the $L=1$ point estimate of the proposed
    method, while the shaded fan summarises the distribution across posterior
    sub-partition draws, with shading indicating central posterior coverage.
    }
    \label{fig:noise_level_robustness}
\end{figure}
\paragraph{Increasing clutter level.}
We next examine robustness as the proportion of inactive observations increases, with dataset and clustering result analysis in Appendix \ref{app:Robustness to increasing clutter}.
Figure~\ref{fig:noise_level_robustness}(a) examines how posterior uncertainty changes with contamination level. The one-particle ($L=1$) $\mathrm{VI}_{\mathrm{IA}}$ compression loss increases as the clutter level grows.
Its decomposition shows that this increase is almost entirely driven by the active--inactive component, while the within-active clustering contribution
remains negligible. Thus, in this example, the clustering structure among observations classified
as active is highly stable, and the dominant posterior uncertainty arises from disagreement over active versus inactive status.
Figure~\ref{fig:noise_level_robustness}(b) compares the proposed method with DBSCAN, gap-selected $k$-means, and BNPmix. For the proposed method, the shaded fan represents the distribution of
truth-referenced $\mathrm{VI}_{\mathrm{IA}}$ losses across individual posterior sub-partition draws, while the black line gives the loss of the $L=1$ point
estimate and therefore need not lie at the centre of this distribution. Across the clutter levels, the proposed method maintains a relatively stable
normalised $\widetilde{\mathrm{VI}}_{\mathrm{IA}}$ loss. In contrast, the full-partition methods deteriorate as clutter increases. DBSCAN performs slightly better at lower contamination levels, but its loss increases steadily and reaches $0.102$ at $\epsilon=0.4$. Additionally, as shown in Figure~\ref{fig:noise_level_data} in
Appendix~\ref{app:Robustness to increasing clutter}, DBSCAN recovers most clusters but consistently merges a pair of neighbouring clusters. In contrast, the proposed method preserves the main clustering structure while progressively classifying more observations as inactive as clutter increases.

Figure~\ref{fig:noise_level_robustness}(c) shows the corresponding estimated inactive fractions.
The proposed method tracks the increasing clutter level closely while also quantifying uncertainty in this estimate, whereas DBSCAN increasingly
underestimates the inactive fraction at higher contamination levels. As expected, the full-partition methods, $k$-means and BNPmix, return zero
inactive observations throughout.

\subsection{Evaluating uncertainty with WASABI-SuP on harbour radar data}
\label{sec:harbour}
\begin{figure}
    \centering
    \includegraphics[width=0.5\linewidth]{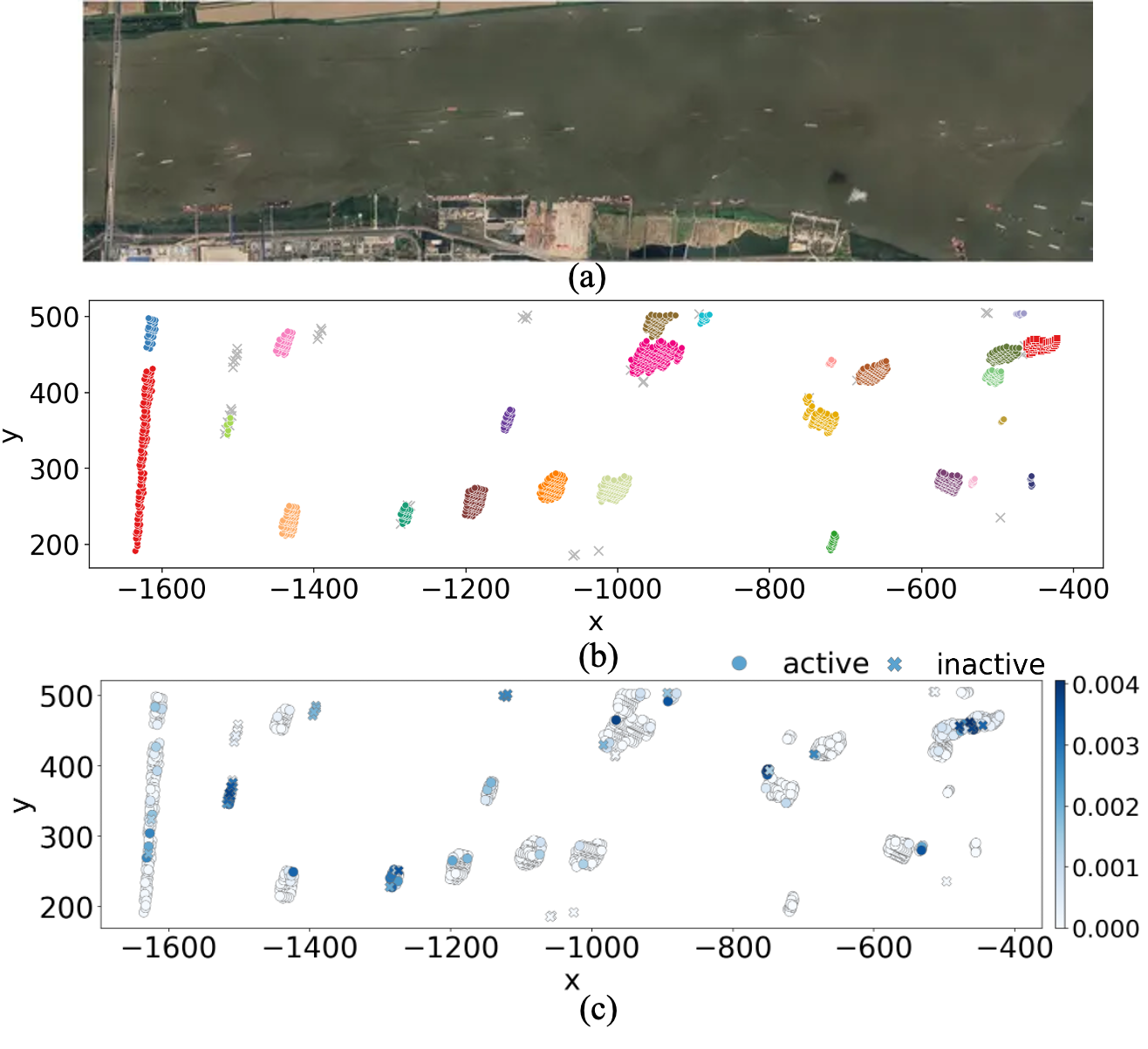}
    \caption{Harbour data. (a) Satellite map of the site.  (b) VI$_{\mathrm{IA}}$ point estimate; colours denote cluster membership and grey crosses are clutter. (c) Expected VI$_{\mathrm{IA}}$ contribution around point estimate; marker shape indicates active/inactive status and darker colours indicate greater uncertainty.}
    \label{fig:harbour-main}
\end{figure}

Finally, we examine the posterior uncertainty represented by WASABI-SuP on harbour radar data. 
Figure~\ref{fig:harbour-main}(a) shows the satellite view of the scene, while Figure~\ref{fig:harbour-main}(b) shows the VI$_{\mathrm{IA}}$
summary for $L=1$.  Figure~\ref{fig:harbour-main}(c) localises the uncertainty around
this estimate by colouring each observation according to its contribution to the posterior expected VI$_{\mathrm{IA}}$ loss (defined in Appendix~\ref{app:EVIC contribution}).  Pale points are stable across posterior draws, whereas darker points indicate greater uncertainty in their clustering or the active/inactive status.  
 
Appendix~\ref{app:harbour} validates WASABI-SuP supports multimodal posterior summaries to recover alternative plausible clustering structures. Moreover, increasing the number of WASABI-SuP particles decreases the IA-VI compression loss. The elbow plot selects $L=3$ and the particle weights are $(0.400,\;0.333,\;0.267)$. Hence no single mode dominates the posterior approximation. 
Additional  experimental settings and  comparisons with the BI$_{\mathrm{IA}}$ are also given in Appendix~\ref{app:harbour}.

\section{Conclusion}
We introduced an efficient Bayesian density-based clustering framework for data with
noise using martingale posteriors without requiring a
parametric mixture model. We further developed WASABI-SuP to summarise
a posterior over sub-partitions by a small set of weighted representative clusterings, retaining uncertainty in both cluster structure and active/inactive status. Experiments demonstrate accurate
noise-aware clustering estimation and show that WASABI-SuP preserves alternative posterior modes, with increasingly faithful summaries as
the number of particles grows. While the experiments and applications presented are limited to moderate dimensions, higher dimensional settings can be accommodated through suitable copula constructions \citep{fong2023martingale} or latent embeddings. 
Future work includes adaptive density-level selection, more scalable optimisation, and extensions to dynamic clustering.

\vskip 0.2in
\bibliography{ref}

\newpage

\appendix
\begin{center}
{\Large\bfseries Model-free Bayesian Clustering and Multimodal Posterior Summarisation: 
Supplementary Materials}
\end{center}

\section{Martingale-posterior predictive foundations}
\label{app:mpd_foundations}

This section provides additional details underlying the predictive-resampling
construction in Section~\ref{sec:mpd_intro}. We follow the martingale-posterior
framework of \citet{fong2023martingale}, in which uncertainty about a
quantity of interest is induced through a predictive distribution on
unobserved future data.

\subsection{Martingale posteriors from missing future observations}

Suppose that
$Y_{1:n}\overset{\mathrm{i.i.d.}}{\sim}F_0$
and that $Y_{1:n}=y_{1:n}$ has been observed. Consider completing the
observed sample with an infinite sequence of unobserved future observations,
$Y_{n+1},Y_{n+2},\ldots .$
Let $F_\infty$ denote the limiting empirical distribution of the completed
sample, whenever this limit exists, and let
$\vartheta_\infty=\vartheta(F_\infty)$
denote a generic quantity of interest that can be recovered from the
completed information. The martingale posterior is the distribution induced
on $\vartheta_\infty$ by the predictive distribution of the missing future:
\begin{equation}
\Pi_\infty
\left(
\vartheta_\infty\in B
\mid y_{1:n}
\right)
=
\int
\mathbf{1}
\left\{
\vartheta(F_\infty)\in B
\right\}
\,\mathrm{d}\Pi
\left(
F_\infty\mid y_{1:n}
\right),
\notag
\end{equation}
for measurable $B$ in the space of $\vartheta$.

Rather than deriving the predictive distribution from a likelihood, prior,
and posterior, martingale-posterior specifies the sequence of
one-step-ahead predictive distributions directly. After observing
$Y_{1:i}=y_{1:i}$, define
\begin{equation}
P_i(y)
:=
\Pr
\left(
Y_{i+1}\leq y
\mid y_{1:i}
\right),
\qquad
i\geq n,
\label{eq:Pi_def}
\end{equation}
with predictive density $p_i$ whenever it exists. The initial predictive
$P_n$ summarises the information in the observed sample and any externally
specified predictive information.

\subsection{Predictive coherence and the limiting predictive distribution}

The predictive sequence must satisfy a coherence condition so that repeated
forward simulation yields a well-defined limiting random distribution. A
convenient sufficient condition is that the future sequence
$Y_{n+1},Y_{n+2},\ldots$ is \emph{conditionally identically distributed}
(c.i.d.) \citep{fong2023martingale}. In particular, conditional on the
currently observed and imputed sequence, all subsequent observations have
the same one-step-ahead predictive distribution.

Equivalently, for each fixed $y$, the predictive sequence satisfies the
martingale condition
\begin{equation}
\mathbb{E}
\left[
P_i(y)
\mid
Y_{1:i-1}
\right]
=
P_{i-1}(y),
\qquad
i\geq n+1,
\label{eq:Pi_martingale}
\end{equation}
almost surely. Hence,
$P_n(y),P_{n+1}(y),\ldots$
forms a martingale for each $y$ and converges almost surely to a random
limiting predictive distribution $P_\infty$ under the usual regularity
conditions. Moreover,
\begin{equation}
\mathbb{E}
\left[
P_\infty(y)
\mid y_{1:n}
\right]
=
P_n(y),
\notag
\end{equation}
so predictive resampling propagates uncertainty around the current
predictive assessment without systematically shifting its conditional mean.

For c.i.d.\ sequences, the limiting predictive distribution $P_\infty$
coincides almost surely with the limiting empirical distribution
$F_\infty$ under the conditions described by
\citet{fong2023martingale}. Thus, when a functional is well defined for
both objects, it can equivalently be expressed as
$\vartheta(F_\infty)$ or $\vartheta(P_\infty)$.

\subsection{Density-valued functionals and finite truncation}

For the clustering construction in this paper, the target depends on a
density rather than solely on a distribution function. At a finite
predictive-resampling horizon $N$, the empirical distribution
\[
F_N
=
\frac{1}{N}
\left(
\sum_{i=1}^{n}\delta_{y_i}
+
\sum_{i=n+1}^{N}\delta_{Y_i}
\right)
\]
is discrete and therefore does not provide a suitable density for
level-set clustering. Instead, following
\citet{fong2023martingale}, we work with the updated predictive
distribution $P_N$, which remains continuous under the copula construction,
and use its density $p_N$.
More generally, when the finite martingale posterior is represented through an empirical distribution, as in TabMGP \citep{ng2025tabmgp}, kernel density estimation could be used to obtain a smooth density before applying the level-set map.

The random predictive $P_N$ provides a finite approximation to
$P_\infty$. Under the c.i.d.\ condition, $P_N$ converges weakly to a
random probability measure $P_\infty$ almost surely. In the univariate
Gaussian-copula construction, stronger conditions established by
\citet{fong2023martingale} ensure that $P_\infty$ is absolutely continuous
with density $p_\infty$ and that $P_N$ converges to $P_\infty$ in total
variation. Corresponding density convergence is not required for the finite
implementation used here: for sufficiently large finite $N$, each terminal
density $p_N$ can be mapped directly to a sub-partition through
$\psi_\lambda$.

Accordingly, define the finite-horizon clustering
\begin{equation}
\theta_N
=
\psi_\lambda(p_N).
\notag
\end{equation}
Repeating the predictive-resampling procedure independently produces
finite-horizon draws
$\theta_N^{(1)},\ldots,\theta_N^{(M)}$.
For sufficiently large $N$, these provide the practical approximation used
for the limiting clustering posterior in
\eqref{eq:clustering_mpd}.

\section{Elicitation of the predictive sequence with recursive Gaussian copulas}
\label{app:mpd_gausscop}

We now describe the predictive update used to generate the density draws in
Section~\ref{sec:mpd_intro}. Following
\citet{fong2023martingale}, the predictive sequence is constructed through
recursive Gaussian-copula updates. The construction is particularly suited
to predictive resampling because it is sequential, admits direct forward
simulation, and satisfies the martingale condition
\eqref{eq:Pi_martingale}, thereby yielding a c.i.d.\ predictive sequence.

\subsection{Univariate construction}

For clarity, we first describe the univariate update. Let $P_i$ denote the
predictive distribution function in \eqref{eq:Pi_def}, with density $p_i$.
Let $(U,V)\in(0,1)^2$ follow a Gaussian copula with correlation parameter
$\rho\in(0,1)$, and write
$c_\rho(u,v)$
for its density and
$H_\rho(u,v)
:=
\Pr(U\leq u\mid V=v)$
for the associated conditional distribution function.

For the Gaussian copula,
\begin{equation}
H_\rho(u,v)
=
\Phi
\left(
\frac{
\Phi^{-1}(u)
-
\rho\,\Phi^{-1}(v)
}{
\sqrt{1-\rho^2}
}
\right),
\notag
\end{equation}
where $\Phi$ is the standard normal distribution function. The corresponding
copula density is
\begin{equation}
c_\rho(u,v)
=
\frac{
\phi_\rho
\left(
\Phi^{-1}(u),
\Phi^{-1}(v)
\right)
}{
\phi\!\left(\Phi^{-1}(u)\right)
\phi\!\left(\Phi^{-1}(v)\right)
},
\notag
\end{equation}
where $\phi$ is the standard normal density and $\phi_\rho$ is the
bivariate standard normal density with correlation $\rho$.

Given a newly observed or imputed value $Y_{i+1}$, define
$V_i
=
P_i(Y_{i+1})$.
The predictive distribution and density are updated recursively as
\begin{align}
P_{i+1}(y)
&=
(1-\alpha_{i+1})P_i(y)
+
\alpha_{i+1}
H_\rho
\left(
P_i(y),V_i
\right),
\label{eq:mpd_cdf_update}
\\
p_{i+1}(y)
&=
\left[
1-\alpha_{i+1}
+
\alpha_{i+1}
c_\rho
\left(
P_i(y),V_i
\right)
\right]
p_i(y),
\label{eq:mpd_density_update}
\end{align}
where $\alpha_{i+1}\in(0,1)$ controls the magnitude of the update.
The sequence $\{\alpha_i\}$ is chosen to decrease with $i$ so that the
predictive update becomes progressively smaller as information accumulates.
We use the step-size sequence of \citet{fong2023martingale},
$\alpha_i
=
\left(2-\frac{1}{i}\right)\frac{1}{i+1}$,
$i\geq1$,
which is $O(i^{-1})$ and decreases to zero. The same sequence is continued during predictive resampling, so the first forward update after observing $y_{1:n}$ uses $\alpha_{n+1}$.
The parameter $\rho$ controls the smoothness of the copula update. Larger
values produce a more concentrated copula kernel around the newly observed
or imputed value, while smaller values produce smoother updates. \cite{hahn2018recursive} report the analysis of choosing $\rho$ and give the suggested $\rho$ values.  As in
\citet{fong2023martingale}, $\rho$ may be selected using predictive performance, for example through a prequential log score. 

\paragraph{Predictive resampling in PIT space.}
If $Y_{i+1}\sim P_i$ and $P_i$ is continuous, then
$V_i=P_i(Y_{i+1})
\sim
\mathrm{Unif}(0,1)$.
Consequently, predictive resampling does not require explicitly generating
$Y_{i+1}$ and subsequently evaluating its predictive probability integral
transform. Instead, we may draw
$V_i\sim\mathrm{Unif}(0,1)$
directly and substitute $V_i$ into
\eqref{eq:mpd_cdf_update}--\eqref{eq:mpd_density_update}. Starting from
$(P_n,p_n)$, one predictive-resampling path iterates these updates for
$i=n,\ldots,N-1$ to obtain a random terminal predictive density $p_N$.

Repeating the procedure independently $M$ times produces
$p_N^{(1)},\ldots,p_N^{(M)}$,
which are the density draws
$f^{(1)},\ldots,f^{(M)}$
used in \eqref{eq:clustering_draws}. Because the resampling paths are
independent, the resulting finite-horizon density draws can be generated
in parallel.

\subsection{Multivariate construction}
\label{app:mpd_multivariate}

For multivariate observations
$Y=(Y^1,\ldots,Y^d)$, applying the univariate update independently to each
coordinate does not in general preserve the joint martingale condition.
We therefore use the coherent multivariate construction of
\citet{fong2023martingale}, which uses ordered conditional probability integral transforms to construct a product copula factor within a single
joint predictive-density update.

For a fixed ordering of the coordinates, define
\begin{align}\notag
u_i^j
&=
P_i
\left(
y^j
\mid
y^{1:j-1}
\right), \qquad
v_i^j
=
P_i
\left(
Y_{i+1}^j
\mid
Y_{i+1}^{1:j-1}
\right),
\qquad
j=1,\ldots,d.
\end{align}
The resulting joint predictive-density update takes the form
\begin{equation}
p_{i+1}(y)
=
\left[
1-\alpha_{i+1}
+
\alpha_{i+1}
\prod_{j=1}^{d}
c_{\rho_j}
\left(
u_i^j,v_i^j
\right)
\right]
p_i(y),
\notag
\end{equation}
where separate copula parameters $\rho_j$ may be used for different
coordinates. 
This construction, rather than independent coordinatewise copula updates,
preserves the required martingale condition; see \citet{fong2023martingale} for the corresponding recursive updates of the conditional distributions.

Under predictive resampling, the conditional PITs satisfy
$v_i^j
\sim
\mathrm{Unif}(0,1)$,
$j=1,\ldots,d$,
so each forward step can be generated using uniform random variables
followed by the deterministic predictive update.

For the clustering procedure, only the terminal predictive-density
evaluations at the observed locations are required:
$\left\{
p_N^{(m)}(y_i)
:
i=1,\ldots,n
\right\}$,
$m=1,\ldots,M$.
These values determine the active sets and hence the level-set graph
clusterings in \eqref{eq:clustering_draws}; no latent mixture allocations
are introduced at any stage.

\paragraph{Higher-dimensional construction.}
Although the experiments considered here are two-dimensional, the multivariate
predictive construction is not restricted to $d=2$. The conditional-copula
factorisation of \citet{fong2023martingale} is defined for
$\mathbf y\in\mathbb R^d$, with tractable updates requiring no numerical
integration; after the initial construction, predictive resampling scales
linearly with $d$. The method has also been evaluated for multivariate density
estimation with dimensions up to $d=32$ in \citet{fong2023martingale}.

\subsection{Modularity of the predictive update}

The Gaussian-copula construction is used here to obtain an efficient,
coherent sequence of predictive-density draws, but the downstream
clustering framework does not depend on this particular choice. Any
predictive-resampling mechanism that supplies evaluable density draws at
the observed locations can be combined with the mapping $\psi_\lambda$.
Examples include score-based predictive constructions
\citep{bariletto2026scalable}, martingale-posterior neural processes
\citep{lee2023martingale}, and transformer-based predictive methods
\citep{ng2025tabmgp}. Replacing the density-generation module therefore
does not alter either the level-set graph construction or the WASABI-SuP
posterior summarisation developed subsequently.

\section{Sample level-set graph constructions}
\label{app:levelset_graphs}

Section~\ref{sec: graph} defines the generic sample level-set clustering map
$\psi_\lambda(f)$ by first identifying the active observations
\[
A_\lambda(f)
=
\{i\in[n]:f(y_i)\geq\lambda\},
\]
and then partitioning them according to the connected components of a graph
$G_\lambda(f)$. Observations outside $A_\lambda(f)$ remain inactive and are
assigned label $0$.
The graph construction provides a discrete surrogate for the connected
components of the continuous level set $L_\lambda(f)$. 
Below we describe the two connectivity rules considered in this paper.

\subsection{\texorpdfstring{$\delta$}{delta}-neighbourhood graph}
\label{app:delta-neighborhood graph}

A common choice, also employed in BALLET \citep{buch2024bayesian}, is the
$\delta$-neighbourhood graph. For a radius $\delta>0$, define
\[
G_{\delta,\lambda}(f)
=
\bigl(A_\lambda(f),E_{\delta,\lambda}(f)\bigr),
\]
where
\begin{equation}
(i,m)\in E_{\delta,\lambda}(f)
\quad\Longleftrightarrow\quad
i,m\in A_\lambda(f)
\quad\text{and}\quad
\|y_i-y_m\|\leq\delta ,
\notag
\end{equation}
with $\|\cdot\|$ denoting the chosen metric, taken to be Euclidean in our
experiments.

The connected components of $G_{\delta,\lambda}(f)$ define the active
clusters. Let $K(f)$ denote their number and
$C_1(f),\ldots,C_{K(f)}(f)$ the corresponding sets of observation indices.
The induced allocation vector is
\begin{equation}
\theta(f)
=
\bigl(c_1(f),\ldots,c_n(f)\bigr)
\in
\{0,1,\ldots,K(f)\}^n,
\notag
\end{equation}
where
\begin{equation}
c_i(f)=
\begin{cases}
h, & \text{if } i\in C_h(f)
\text{ for some } h\in\{1,\ldots,K(f)\},\\[1mm]
0, & \text{if } i\notin A_\lambda(f).
\end{cases}
\notag
\end{equation}
Thus, label $0$ identifies inactive observations. We denote the resulting
sample level-set clustering map by
\begin{equation}
\psi_{\delta,\lambda}(f):=\theta(f).
\label{eq:psi_def}
\end{equation}

The two parameters play distinct roles. The density threshold $\lambda$
determines which observations are active: increasing $\lambda$ focuses on
denser cluster cores and typically leaves more observations inactive,
whereas decreasing $\lambda$ includes more low-density boundary observations.
The radius $\delta$ determines connectivity among active observations:
smaller values favour local connectivity but may fragment elongated
structures, whereas larger values can connect nearby regions through
intermediate observations.
Once $(\lambda,\delta)$ are fixed, $\psi_{\delta,\lambda}$ is deterministic;
hence, in our predictive-resampling framework, uncertainty in the resulting
sub-partition is induced by uncertainty in the density draw $f$.

\subsection{Reciprocal-\texorpdfstring{$\text{k}$}{k}NN graph}
\label{app:rkNN}

The $\delta$-neighbourhood graph uses a single global distance threshold and
therefore inherits the chaining behaviour of single-linkage-type
connectivity: two dense regions may be joined by intermediate active
observations connected through short edges. This can be problematic when
local sampling density varies substantially or when sparse bridge points
occur between otherwise distinct clusters. Figure~\ref{fig:rknn_vs_delta}
illustrates this behaviour and how reciprocal-$k$NN connectivity can suppress
such bridge connections.

\begin{figure}
    \centering
    \includegraphics[width=1\linewidth]{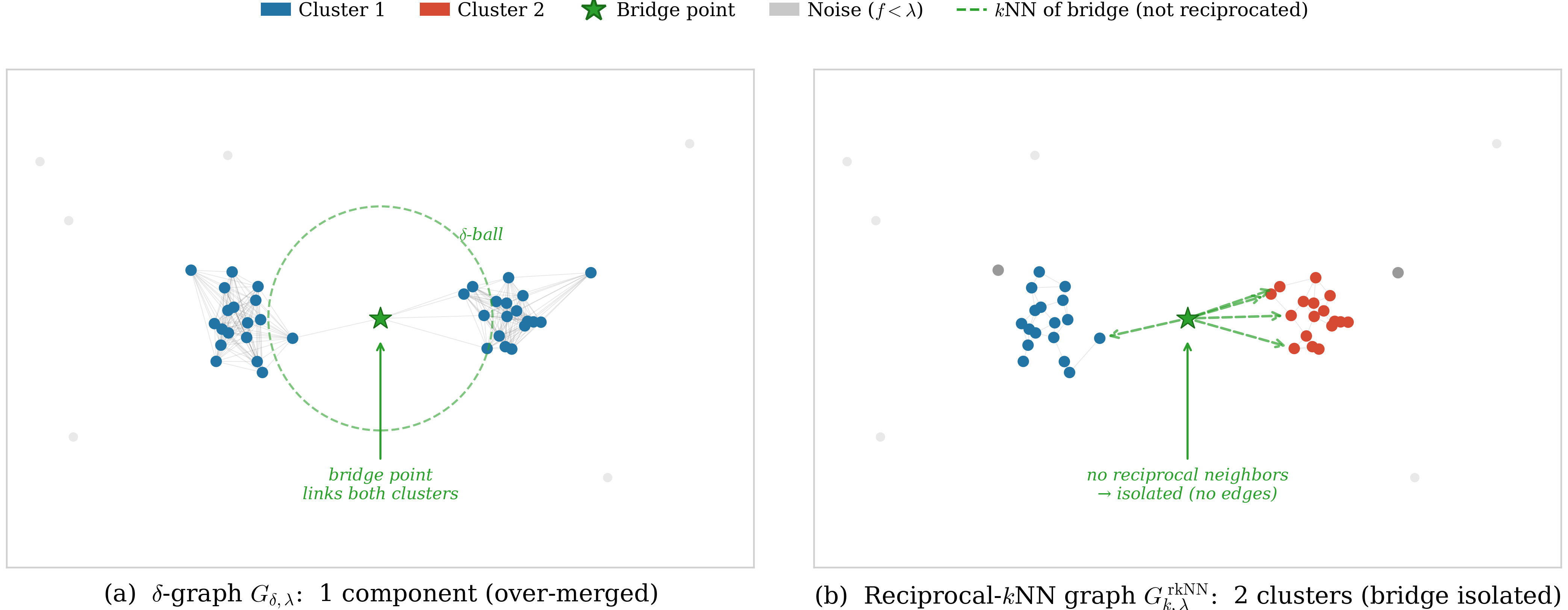}
    \caption{
    Illustration of chaining in the $\delta$-graph and its suppression by
    reciprocal-$k$NN connectivity. Two dense clusters are separated by a
    single bridge point (green star) lying above the density threshold
    $\lambda$; grey points lie below the threshold and are inactive.
    \textbf{(a)} In the $\delta$-graph $G_{\delta,\lambda}(f)$, the bridge is
    connected to active observations in both clusters, merging them into a
    single connected component.
    \textbf{(b)} In the reciprocal-$k$NN graph
    $G^{\mathrm{rkNN}}_{k,\lambda}(f)$ with $k=5$, the bridge selects nearby
    observations, but they do not select it in return. The bridge therefore
    remains isolated and no longer connects the two dense components.
    }
    \label{fig:rknn_vs_delta}
\end{figure}

To mitigate this behaviour, we also consider a locally adaptive
reciprocal-$k$NN graph. The active set $A_\lambda(f)$ remains unchanged, so
the two constructions differ only in how edges are formed among active
observations. 
For a given active set $A_\lambda(f)$, define
$k_f =\min\!\left\{
k,\,|A_\lambda(f)|-1\right\}$.
For each $i\in A_\lambda(f)$, let
$N_{k_f}(i;f)\subseteq A_\lambda(f)\setminus\{i\}$
denote the indices of its $k_f$ nearest active neighbours under the chosen
metric, using a fixed deterministic rule to resolve equal-distance ties. If
$|A_\lambda(f)|\leq1$, we set $N_{k_f}(i;f)=\varnothing$.

We define
\[
G^{\mathrm{rkNN}}_{k,\lambda}(f)
=
\bigl(
A_\lambda(f),
E^{\mathrm{rkNN}}_{k,\lambda}(f)
\bigr),
\]
with
\begin{equation}
(i,m)\in E^{\mathrm{rkNN}}_{k,\lambda}(f)
\quad\Longleftrightarrow\quad
m\in N_{k_f}(i;f)
\quad\text{and}\quad
i\in N_{k_f}(m;f).
\notag
\end{equation}

Thus, two active observations are connected only when each is among the
other's $k_f$ nearest active neighbours. As in the $\delta$-neighbourhood
construction, the connected components define the active clusters, while
observations outside $A_\lambda(f)$ are assigned label $0$. We denote the
resulting clustering map by
$\psi^{\mathrm{rkNN}}_{k,\lambda}(f)$.

The reciprocal condition can suppress spurious bridge connections. A point
in a sparse region may regard a nearby dense-cluster point as one of its
$k$ nearest neighbours, while the dense-cluster point need not select the
sparse point in return because it has closer neighbours within its own local
region. Hence, the reciprocal-$k$NN construction adapts connectivity to local
sample geometry and can reduce over-merging caused by chaining. It therefore
defines a distinct level-set graph surrogate from
$\psi_{\delta,\lambda}$, while preserving the same active-set definition and
sub-partition representation.

\section{Existing partition and sub-partition losses}
\label{app:loss function}

Throughout this section, let
$\theta=(c_1,\ldots,c_n)$ and
$\hat\theta=(\hat c_1,\ldots,\hat c_n)$
denote two allocations of the observations in $[n]$.
For ordinary partition losses, all observations are active and the
allocations induce full partitions
$\rho=\{C_1,\ldots,C_K\}$ and
$\hat\rho=\{\hat C_1,\ldots,\hat C_{\hat K}\}$.
For sub-partitions, label $0$ is reserved for inactive observations.

\paragraph{Binder loss.}
Binder's loss \citep{binder1978bayesian} is the pairwise loss
\begin{equation}
B(\theta,\hat\theta)
=
\sum_{1\leq i<j\leq n}
\left[
a\,\mathbb I\{c_i=c_j,\ \hat c_i\neq\hat c_j\}
+
b\,\mathbb I\{c_i\neq c_j,\ \hat c_i=\hat c_j\}
\right],
\notag
\end{equation}
where $a,b>0$ penalise splitting and merging disagreements,
respectively.
Equal penalties are common, although Binder-based summaries may favour
over-fragmented partitions; unequal penalties can be used to control
this split--merge trade-off \citep{dahl2022search,wade2018bayesian}.
Because the loss is pairwise additive, its posterior expected loss can
be evaluated from the posterior similarity matrix.

\paragraph{Variation of information.}
The variation of information (VI) \citep{meilua2007comparing} between
the full partitions $\rho$ and $\hat\rho$ is
\begin{equation}
\mathrm{VI}(\rho,\hat\rho)
=
H(\rho)+H(\hat\rho)-2I(\rho,\hat\rho),
\label{eq:vi_appendix}
\end{equation}
where entropy $H$ and mutual information $I$  are computed as
\begin{equation}\notag
H(\rho)
=
-\sum_{k=1}^{K}
\frac{|C_k|}{n}
\log_2\frac{|C_k|}{n}, \qquad
I(\rho,\hat\rho)
=
\sum_{k=1}^{K}
\sum_{j=1}^{\hat K}
\frac{|C_k\cap\hat C_j|}{n}
\log_2
\frac{n|C_k\cap\hat C_j|}
{|C_k|\,|\hat C_j|},
\end{equation}
with zero-count terms interpreted as zero.
VI is invariant to permutations of cluster labels and is a metric on
the space of full partitions, with
$0\leq \mathrm{VI}(\rho,\hat\rho)\leq\log_2 n$.

\paragraph{Inactive–active Binder loss for sub-partitions.} \cite{buch2024bayesian} proposed the \emph{inactive–active Binder} (BI$_{\mathrm{IA}}$) loss, which augments Binder’s loss with explicit penalties for mismatches between inactive and active labels.  
For sub-partitions, define the active sets
$A=\{i:c_i\neq0\}$,
$\hat A=\{i:\hat c_i\neq0\}$,
and corresponding inactive sets
$I=[n]\setminus A$ and
$\hat I=[n]\setminus\hat A$.
The BI$_{\mathrm{IA}}$ loss is
\begin{align}
\mathrm{BI}_{\mathrm{IA}}(\theta,\hat\theta)
=&
(n-1)
\left(
m_{ai}|A\cap\hat I|
+ m_{ia}|I\cap\hat A|
\right)\notag\\
&+
\sum_{\substack{1\leq i<j\leq n\\ i,j\in A\cap\hat A}}
\left[
a\,\mathbb I\{c_i=c_j,\ \hat c_i\neq\hat c_j\}
+
b\,\mathbb I\{c_i\neq c_j,\ \hat c_i=\hat c_j\}
\right].
\label{eq:ia-binder}
\end{align}
The pairwise Binder term is evaluated only for observations active in both sub-partitions, while the first term explicitly penalises disagreement in active status. Under suitable parameter choices, including $a=b$ and $m_{ai}=m_{ia}\geq a/2$, BI$_{\mathrm{IA}}$ satisfies metric properties on the space of sub-partitions \citep{buch2024bayesian}.

\section{IA-VI loss:  generalised form, properties and proofs}\label{app:IA-VI loss}
The symmetric VI$_{\mathrm{IA}}$ loss in \eqref{eq:ia_vi_lin} is a special case of
the more general asymmetric form of VI$_{IA}$ loss with $m_{ai}>0$ and $m_{ia}>0$ controlling the penalties for active-to-inactive and inactive-to-active mismatches, respectively.
For $n\geq2$, define
\begin{align}
\mathrm{VI}_{\mathrm{IA}}(\theta,\hat\theta)
&=
\frac{\log_2 n}{n}
\left(
m_{ai}|A\cap\hat I|
+
m_{ia}|I\cap\hat A|
\right) +
\mathrm{VI}(\rho_B,\hat\rho_B),
\label{eq:ia_vi_lin general}
\end{align}
The first term measures disagreement about which observations participate in the clustering, while the second compares the clustering structure conditional on being active in both sub-partitions. The definition in \eqref{eq:ia_vi_lin} is recovered by
setting $m_{ai}=m_{ia}=\tau$.

\paragraph{Explicit expressions for the restricted VI term}
For $|B|>0$, write the non-empty clusters of the restricted partitions on $B$ as $\rho_B=\{D_1,\ldots,D_{K_B}\}$ and $\hat\rho_B =\{\hat D_1,\ldots,\hat D_{\hat K_B}\}$. Then $\mathrm{VI}(\rho_B,\hat\rho_B)$ is computed using  standard VI in \eqref{eq:vi_appendix}, applied to the restricted partitions on $B$.
\begin{align}
\label{eq:ia-vi-expanded}
\mathrm{VI}(\rho_B,\hat\rho_B)
&=
H(\rho_B)+H(\hat\rho_B)-2I(\rho_B,\hat\rho_B) \\\notag
&=
-\sum_{k=1}^{K_B}
\frac{|D_k|}{|B|}
\log_2\!\left(\frac{|D_k|}{|B|}\right)
-\sum_{j=1}^{\hat K_B}
\frac{|\hat D_j|}{|B|}
\log_2\!\left(\frac{|\hat D_j|}{|B|}\right)
-2\sum_{k=1}^{K_B}\sum_{j=1}^{\hat K_B}
\frac{|D_k\cap \hat D_j|}{|B|}
\log_2\!\left(
\frac{|D_k\cap \hat D_j|\,|B|}
{|D_k|\,|\hat D_j|}
\right).
\end{align}
Equivalently, after combining the entropy and mutual-information terms, we may write
\begin{equation}\notag
\mathrm{VI}(\rho_B,\hat\rho_B)
=
-\sum_{k=1}^{K_B}\sum_{j=1}^{\hat K_B}
\frac{|D_k\cap \hat D_j|}{|B|}
\log_2\!\left(
\frac{|D_k\cap \hat D_j|^2}{|D_k|\,|\hat D_j|}
\right),
\end{equation}
with the usual convention that the summand is zero whenever $|D_k\cap \hat D_j|=0$.
When $B=\varnothing$, we set $\mathrm{VI}(\rho_B,\hat\rho_B)=0$.

\subsection{Proof of the basic VI$_{\mathrm{IA}}$ properties}
\label{app:iavi_properties}
\begin{proof}[Proof of Proposition~\ref{prop:ia_vi_lin_basic}]
Non-negativity is immediate from \eqref{eq:ia_vi_lin}, since $\tau>0$,
$|A\cap \hat{I}|, |I\cap \hat{A}|\ge 0$, $\log_2(n)>0$, and the ordinary VI is non-negative.

Symmetry follows because
$|A\cap\hat I|+|I\cap\hat A|$
is symmetric in $A$ and $\hat A$, and ordinary VI is symmetric.
Therefore
$\mathrm{VI}_{\mathrm{IA}}(\theta,\hat\theta)
=
\mathrm{VI}_{\mathrm{IA}}(\hat\theta,\theta)$.

Suppose $\mathrm{VI}_{\mathrm{IA}}(\theta,\hat\theta)=0$.
Both terms in \eqref{eq:ia_vi_lin} are non-negative, so each term must vanish separately.
Hence $\tau|A\cap \hat{I}|+\tau|I\cap \hat{A}|=0$.
Since $\tau>0$, this implies
$|A\cap \hat{I}|=0$,
$|I\cap \hat{A}|=0$.
Therefore $A\subseteq \hat{A}$ and $\hat{A}\subseteq A$, so $A=\hat{A}$, and consequently $I=\hat{I}$.
Thus  $B=A\cap \hat{A}=A$, and 
$\mathrm{VI}(\rho_B,\hat\rho_B)=0$.
Ordinary VI vanishes if and only if the two partitions coincide up to relabelling of cluster
labels. Hence the restrictions of $\theta$ and $\hat\theta$ to the active set induce the same
partition up to relabelling of the positive labels. Since the inactive sets also coincide, it
follows that $\theta=\hat\theta$ up to relabelling of the positive labels.
Conversely, if $\theta=\hat\theta$ up to relabelling of the positive labels, then the zero
entries coincide, so $A=\hat{A}$ and $I=\hat{I}$. Hence
$|A\cap \hat{I}|=|I\cap \hat{A}|=0$,
and the support penalty vanishes. Moreover, the induced partitions on
$B=A=\hat{A}$ are identical up to relabelling, so
$\mathrm{VI}(\rho_B,\hat\rho_B)=0$.
Thus $\mathrm{VI}_{\mathrm{IA}}(\theta,\hat\theta)=0$.

Next, relabelling the positive labels does not change which entries are zero and which are
non-zero, so it leaves $A$, $\hat{A}$, $I$, $\hat{I}$, and therefore the support mismatch counts,
unchanged. On $B=A\cap \hat{A}$, such relabellings do not alter the induced set partition,
only the names of the cluster labels. Ordinary VI is invariant under relabelling of cluster
labels. Hence $\mathrm{VI}_{\mathrm{IA}}$ is invariant under relabelling of positive labels.

Further, if $A=\hat A$, then $I=\hat I$ and hence the activity-mismatch term
vanishes. Moreover, $B=A\cap\hat A=A$, giving
\[
\mathrm{VI}_{\mathrm{IA}}(\theta,\hat\theta)
=
\mathrm{VI}(\rho_A,\hat\rho_A),
\]
which proves \eqref{eq:iavi_reduces_vi}.

Finally, since
$|A\cap\hat I|+|I\cap\hat A|\le n$,
the activity-mismatch term is bounded above by
$\tau\log_2 n$. For $B\neq\varnothing$,
\[
\mathrm{VI}(\rho_B,\hat\rho_B)
\le \log_2|B|
\le \log_2 n,
\]
while for $B=\varnothing$ the VI term is zero by convention.
Consequently,
\[
0
\le
\mathrm{VI}_{\mathrm{IA}}(\theta,\hat\theta)
\le
(\tau+1)\log_2 n,
\]
which proves \eqref{eq:iavi_bound}.

\paragraph{Asymmetric extension and tighter upper bound.}
For the more general asymmetric loss in \eqref{eq:ia_vi_lin general}, a tighter upper bound can be obtained by optimising over the common active-set size $|B|$.

Let $m=\max(m_{ai},m_{ia})$.
First suppose $|B|>0$. 
Ordinary VI satisfies
$
\mathrm{VI}(\rho_B,\hat\rho_B)\le \log_2|B|, 
$
with this upper bound achieved when restricting to $B$, all points are all singletons in one partition and all together in one large cluster in the other partition. 
Since the number of points that are  not active in at least one of the partitions is equal to $n - |B|$, we have that  $|A\cap\hat I|+|I\cap\hat A|
\leq n-|B|$. It follows that
\begin{align*}
m_{ai}|A\cap\hat I|+m_{ia}|I\cap\hat A|
&\leq
m\bigl(|A\cap\hat I|+|I\cap\hat A|\bigr)
\leq
m(n-|B|),
\end{align*}
This bound is achieved when all points outside $B$ are active in one partition and inactive in the other, in the direction associated with
the larger of $m_{ai}$ and $m_{ia}$. If $m_{ai}=m_{ia}$, either direction gives the same bound.
Combining these bounds proves 
\begin{align}\label{eq:bound deriv}
    \mathrm{VI}_{\mathrm{IA}}(\theta,\hat\theta) \leq \begin{cases} 
    m\left(1-\frac{|B|}{n}\right)\log_2 n
     + \log_2(|B|), & \text{ if } |B| >0\\
    m\log_2 n, & \text{ if } |B|=0
    \end{cases}.
\end{align}
Defining 
$
p=\frac{|B|}{n},
p\in\left\{\frac1n,\frac2n,\ldots,1\right\},
$
and relaxing the discrete domain to $p\in[1/n,1]$, 
the first case in \eqref{eq:bound deriv} can be written as:
\begin{align*}
     f(p) = m(1- p)\log_2 n + \log_2 p + \log_2 n.
\end{align*}
Differentiating with respect to $p$ yields
\begin{align*}
    f'(p) &= -m\log_2 n + \frac{1}{p\ln2},\qquad
    f''(p) = -\frac{1}{p^2 \ln 2}<0.
\end{align*}
Thus, the function is strictly concave. 
The maximiser over $[1/n,1]$ is
\begin{align*}
    p^*= \begin{cases}
    \frac{1}{n} & \text{ if } m \geq  \frac{n}{\ln(n)} \\
        \frac{1}{m\ln{n}} & \text{ if } \frac{1}{\ln(n)} < m < \frac{n}{\ln(n)}\\
        1 & \text{ if }  m \leq \frac{1}{\ln(n)} \\
    \end{cases},
\end{align*}
Note also that for $1/\ln(n) < m < n/\ln(n)$, comparing $f(p^*)$ with the $|B|=0$ bound gives
\begin{align*}
    f(p^*) = \frac{m \ln(n)}{\ln(2)} + \frac{\ln(n)- 1 - \ln(m) - \ln(\ln(n))}{\ln(2)} < m \ln(n)/\ln(2),
\end{align*}
that is, when
    $\frac{n}{e \ln(n)} < m$.
Hence, for $n\geq3$,
\begin{align}\label{eq:iavi_upper_bound}
   \mathrm{VI}_{\mathrm{IA}}(\theta,\hat\theta) \leq  \begin{cases}
  m \log_2(n) & \text{ if } m \geq  \frac{n}{e\ln(n)} \\
         (m+1) \log_2(n)  - \log_2(e) -\log_2(m) - \log_2(\ln(n)) & \text{ if } \frac{1}{\ln(n)} < m < \frac{n}{e\ln(n)}\\
       \log_2(n) & \text{ if }  m \leq \frac{1}{\ln(n)} \\
    \end{cases}.
\end{align}
For $n=2$, the corresponding bound simplifies to
\[
\mathrm{VI}_{\mathrm{IA}}(\theta,\hat\theta)
\leq
\max\{1,m\}\log_2 n.
\]
\end{proof}
It is also interesting to examine the sub-partitions that achieve equality with the upper bound, as these characterize the set of sub-partitions that incur the largest loss. For very large $m \geq n/(e\ln(n))$, the upper bound is attained with $|B|=0$, we have the following cases:
\begin{itemize}
    \item if $m_{ia}=m_{ai}$, the most extreme pairs of sub-partitions are the set of $(\theta, \hat\theta )$ with no overlap in their inactive and active sets, i.e. $|A \cap \hat{I}| + |\hat{A} \cap I| = n$;
    \item  if $m_{ia}< m_{ai}$, the most extreme pairs are the $(\theta, \hat\theta )$ with $\hat\theta= (0,\ldots, 0)$ being the sub-partition of all inactive points and $\theta$ being any partition with all points active; 
    \item  if $m_{ia}>m_{ai}$, we have the reverse of the above: the most extreme pairs are the $(\theta, \hat\theta )$ with $\theta= (0,\ldots, 0)$ being the sub-partition of all inactive points and $\hat \theta$ being any partition with all points active.
\end{itemize} 
 For very small $m \leq 1/\ln(n)$, the emphasis is on the VI in the active set and the most extreme pair is the partition of the all-singletons clusters and the partition of one large cluster, with all points active. 
 
In most cases, $m$ will be between these values, and the most extreme pairs of sub-partitions will have $n/(m\ln(n))$ active points in common, if this quantity is an integer. Among these active points, one of the partitions will create all singletons and the other will contain one cluster, while among the other $n-n/(m\ln(n))$ points, 
every point is active in exactly one of the two sub-partitions. If $m_{ai}>m_{ia}$, these observations are active in $\theta$ and inactive in $\hat\theta$; if $m_{ia}>m_{ai}$, the reverse holds; and if $m_{ai}=m_{ia}$, either mismatch direction gives the same loss.

If $n/(m\ln n)$ is not an integer, the continuous optimum cannot be
attained exactly because $|B|$ must be integer-valued. The best
positive common active-set size is therefore one of the two neighbouring
integers; its loss should also be compared with the $|B|=0$ case,
which has loss $m\log_2 n$, to obtain the exact discrete maximiser.

\subsection{Lattice geometric and further  comparative properties of VI$_{\mathrm{IA}}$ and BI$_{\mathrm{IA}}$}\label{app:comparative_properties}
\paragraph{The lattice of sub-partitions.}\label{app:lattice}
Let \([n]=\{1,\ldots,n\}\). A sub-partition of $[n]$ is a collection
$\rho=\{C_1,\ldots,C_K\}$,
where \(K\geq 0\), and the sets \(C_1,\ldots,C_K\) are non-empty and pairwise disjoint
subsets of \([n]\).  We refer to \(C_1,\ldots,C_K\) as the blocks or active clusters of \(\rho\). The active set \(A(\rho)\) records the observations included in the clustering  \(\rho\),
$A(\rho):=\bigcup_{h=1}^K C_h$,
and the inactive set \(I(\rho)\) records the observations that are treated as noise,
$I(\rho):=[n]\setminus A(\rho)$.
\(K=0\) gives the empty sub-partition \(\varnothing\), in which all observations are inactive 
\begin{equation}\label{eq:empty}
  \varnothing \quad\text{with}\quad  A(\varnothing)=\varnothing, \qquad
I(\varnothing)=[n].
\end{equation}
We write $1$ for the
one-cluster partition and $0$ for the all-singletons partition when all observations are active
\begin{equation}\label{eq:1 and 0}
    1:=\{[n]\},
\qquad
0:=\bigl\{\{1\},\ldots,\{n\}\bigr\}.
\end{equation}
More generally,
on a fixed active set \(A\subseteq[n]\), we write
\begin{equation}\label{eq:1a and 0a}
    1_A:=\{A\},
\qquad
0_A:=\bigl\{\{i\}:i\in A\bigr\}.
\end{equation}
Notice that \(0\) is the bottom element only among fully active partitions; in the full
sub-partition lattice, the bottom element is \(\varnothing\).

We equip the space of sub-partitions with the refinement order. In particular, we order sub-partitions by declaring that a sub-partition is below another one if it is
more refined and has no more active observations. Thus, moving upward in the order
corresponds to making the sub-partition less specific: clusters may be merged and
previously inactive observations may be added to the active set.
For two sub-partitions \(\rho\) and \(\hat\rho\), we write \(\rho\preceq\hat\rho\) if every block of \(\rho\) is contained in a block of
\(\hat\rho\); that is, for every \(h\in\{1,\ldots,K\}\),  there exists
\(j\in\{1,\ldots,\hat K\}\) such that
$C_h\subseteq \hat C_j$. This condition implies 
$A(\rho)\subseteq A(\hat\rho)$, 
and ensures that observations clustered together under $\rho$ remain
clustered together under $\hat\rho$. Thus, moving upward in the order
corresponds to merging existing blocks and/or activating previously
inactive observations, without splitting any existing block.

This partial order makes the space of sub-partitions into a lattice. 
For two sub-partitions
\(\rho\) and \(\hat\rho\), the meet \(\rho\wedge\hat\rho\) is the greatest lower bound. It is
given by all non-empty intersections of their blocks:
\[
\rho\wedge\hat\rho
=
\bigl\{
C_h\cap \hat C_j:
h=1,\ldots,K,\ 
j=1,\ldots,\hat K,\ 
C_h\cap \hat C_j\neq\varnothing
\bigr\}.
\]
Consequently,
$A(\rho\wedge\hat\rho)=A(\rho)\cap A(\hat\rho)$. The meet is therefore the coarsest sub-partition that is below both \(\rho\) and \(\hat\rho\):
it keeps only observations that are active in both sub-partitions, and refines their active
clusters by taking all non-empty intersections.

The join \(\rho\vee\hat\rho\) is the least upper bound. Intuitively, it is the finest sub-partition that is above both \(\rho\) and
\(\hat\rho\): it activates all observations that are active in either sub-partition, and
merges any observations that are connected through the grouping structure of either
sub-partition.
Specifically, consider the graph with vertex set $A(\rho)\cup A(\hat\rho)$ which includes all observations that are active in at least one of the two sub-partitions as vertices.
Two vertices are connected whenever they belong to the same block of
\(\rho\) or to the same block of \(\hat\rho\). The connected components of this graph form
\(\rho\vee\hat\rho\). Hence, 
$A(\rho\vee\hat\rho)=A(\rho)\cup A(\hat\rho)$.

The cover relation gives the edges of the Hasse diagram. We write
\(\rho\prec\hat\rho\) if
$\rho\preceq\hat\rho, \rho\neq\hat\rho$,
and there is no sub-partition \(\rho'\) such that
$\rho\preceq\rho'\preceq\hat\rho,$
$\rho'\neq\rho$, $ \rho'\neq\hat\rho$. Equivalently, \(\hat\rho\) is obtained from \(\rho\) by exactly one elementary upward move:
either two blocks of \(\rho\) are merged, or one inactive observation is added as a new
singleton block.

The Hasse diagram in Figure~\ref{fig:subpartition_lattice}
illustrates this order for a small example. For instance,
$\bigl\{\{1\},\{2\}\bigr\}
\prec
\bigl\{\{1,2\}\bigr\}$
is a cover by merging two blocks, while
$\bigl\{\{1,2\}\bigr\}
\prec
\bigl\{\{1,2\},\{3\}\bigr\}$
is a cover by activating the previously inactive observation \(3\) as a singleton block. Consider
$\rho=\bigl\{\{1,2\},\{3\}\bigr\},
\hat\rho=\bigl\{\{2,3\},\{4\}\bigr\}$.
Then, the meet is 
$\rho\wedge\hat\rho
=
\bigl\{\{2\},\{3\}\bigr\}$.
The join is 
$\rho\vee\hat\rho
=
\bigl\{\{1,2,3\},\{4\}\bigr\}$.

\begin{figure}
    \centering
\includegraphics[width=0.8\linewidth]{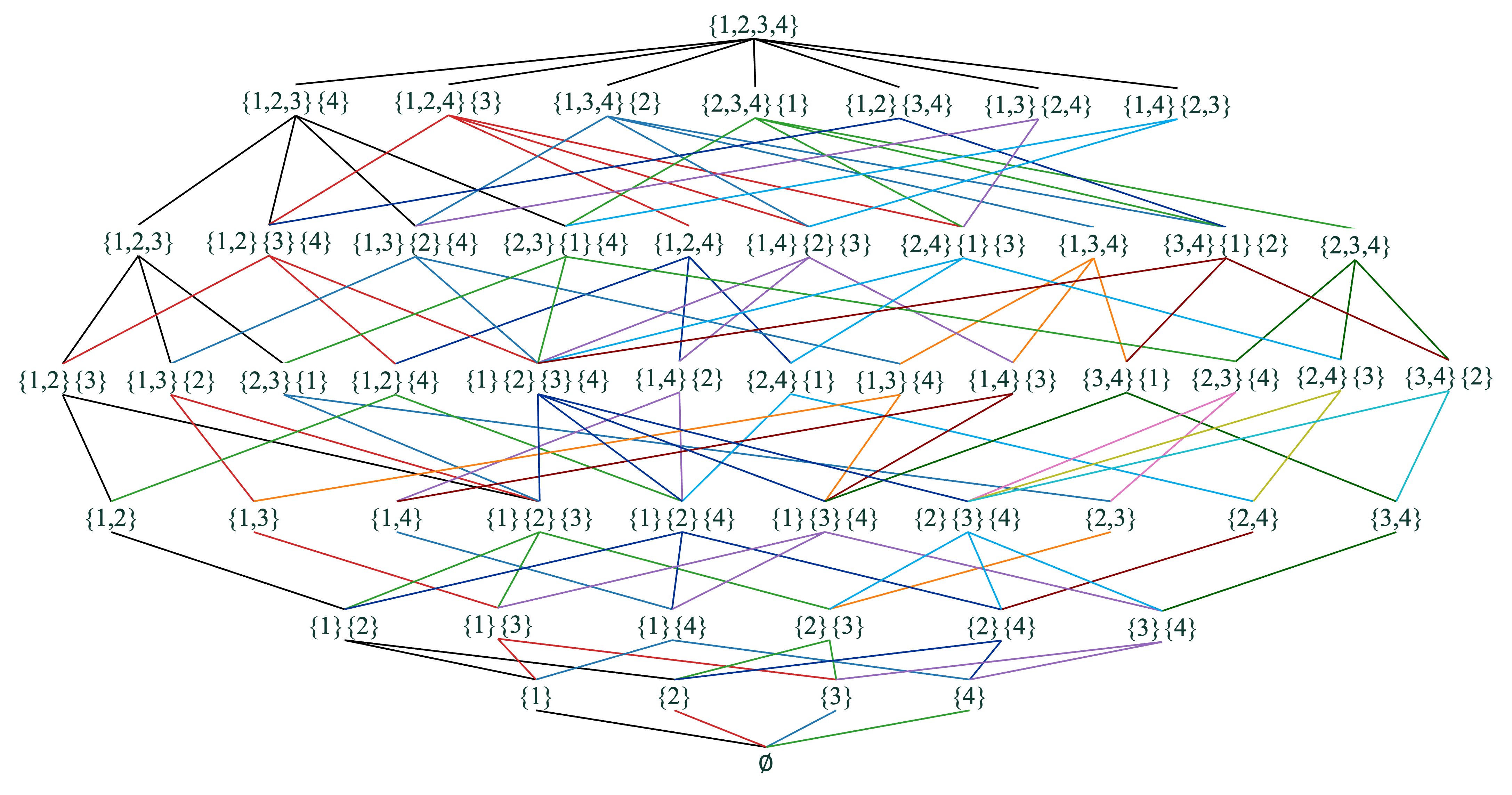}
    \caption{
    Hasse diagram of the sub-partition lattice. Moving upward corresponds either to merging two active blocks or to activating one previously inactive observation as a singleton active block. The bottom element is the empty sub-partition $\varnothing$.
    }
    \label{fig:subpartition_lattice}
\end{figure}

\begin{property}[Meet identity]
\label{prop:ia_vi_meet_identity}
Assume \(m_{ai}=m_{ia}=\tau\) for VI$_{\mathrm{IA}}$ loss in \eqref{eq:ia_vi_lin general}. Let \(\rho\) and \(\hat\rho\) be two sub-partitions, and let
$\rho\wedge\hat\rho$
be their meet in the sub-partition lattice. Then
\begin{equation}\notag
\mathrm{VI}_{\mathrm{IA}}(\rho,\hat\rho)
=
\mathrm{VI}_{\mathrm{IA}}(\rho,\rho\wedge\hat\rho)
+
\mathrm{VI}_{\mathrm{IA}}(\hat\rho,\rho\wedge\hat\rho).
\end{equation}
\end{property}

\begin{example}
Consider $n=4$, with
$\rho=\bigl\{\{1,2\},\{3\}\bigr\}$ and
$\hat\rho=\bigl\{\{2,3\},\{4\}\bigr\}$.
In the Hasse diagram in Figure~\ref{fig:subpartition_lattice}, the meet is
obtained by moving downward from both $\rho$ and $\hat\rho$ and taking their
highest common node,
$\mu=\rho\wedge\hat\rho
=
\bigl\{\{2\},\{3\}\bigr\}$.
We can verify that
$\mathrm{VI}_{\mathrm{IA}}(\rho,\hat\rho)=\tau+1$,
$\mathrm{VI}_{\mathrm{IA}}(\rho,\mu)=\frac{\tau}{2}$,
$\mathrm{VI}_{\mathrm{IA}}(\hat\rho,\mu)=\frac{\tau}{2}+1$.
Hence,
$\mathrm{VI}_{\mathrm{IA}}(\rho,\mu)
+
\mathrm{VI}_{\mathrm{IA}}(\hat\rho,\mu)
=
\mathrm{VI}_{\mathrm{IA}}(\rho,\hat\rho)$.
\end{example}

\begin{proof}[Proof of Property~\ref{prop:ia_vi_meet_identity}]
Let $A=A(\rho), \hat A=A(\hat\rho), B=A\cap \hat A$,
and write $\lambda_n:=\log_2 n/n$.
Let
$
\mu=\rho\wedge\hat\rho.
$
The active set of the meet is
\[
A(\mu)=A(\rho\wedge\hat\rho)=A\cap \hat A=B.
\]
Therefore, for the pair $(\rho,\hat\rho)$, the active--inactive part of
the VI$_{\mathrm{IA}}$ loss is
$
\lambda_n\tau
\bigl(|A\setminus \hat A|+|\hat A\setminus A|\bigr).
$

For the pair $(\rho,\mu)$, the active sets are $A$ and $B$, with
$B\subseteq A$. Hence
$
A\setminus B=A\setminus \hat A, 
 \,\,
B\setminus A=\varnothing,
$
so the active--inactive contribution to
$\mathrm{VI}_{\mathrm{IA}}(\rho,\mu)$ is
$
\lambda_n\tau |A\setminus B|
=
\lambda_n\tau |A\setminus \hat A|.
$

Similarly, for the pair $(\hat\rho,\mu)$,
$
\hat A\setminus B=\hat A\setminus A,
\,\,
B\setminus \hat A=\varnothing,
$
and hence 
$
\lambda_n\tau |\hat A\setminus B|
=
\lambda_n\tau |\hat A\setminus A|.
$

Thus,
\[
\lambda_n\tau
\bigl(|A\setminus \hat A|+|\hat A\setminus A|\bigr)
=
\lambda_n\tau|A\setminus B|
+
\lambda_n\tau|\hat A\setminus B|.
\]

It remains to decompose the ordinary VI term on the common active set.
The restrictions of $\rho$ and $\hat\rho$ to $B$ are denoted by
$\rho_B$ and $\hat\rho_B$. Since $\mu$ is their meet, the partition
induced by $\mu$ on $B$, denoted by $\mu_B$, satisfies
$
\mu_B=\rho_B\wedge\hat\rho_B.
$
Ordinary VI satisfies the meet identity on the partition lattice, so
\[
\mathrm{VI}(\rho_B,\hat\rho_B)
=
\mathrm{VI}(\rho_B,\mu_B)
+
\mathrm{VI}(\hat\rho_B,\mu_B).
\]
Combining the decompositions of the active--inactive and VI terms gives
\[
\mathrm{VI}_{\mathrm{IA}}(\rho,\hat\rho)
=
\mathrm{VI}_{\mathrm{IA}}(\rho,\mu)
+
\mathrm{VI}_{\mathrm{IA}}(\hat\rho,\mu).
\]

When $B=\varnothing$, all restricted VI terms are zero by convention,
so the identity follows directly from the active--inactive
decomposition above.
\end{proof}

\begin{property}[Equal-cluster comparisons]
\label{prop:ia_vi_equal_clusters}
Let $A\subseteq[n]$ be a non-empty active set with $n_A=|A|$.
Assume that $k$ divides $n_A$, and let $c_{k,A}$ denote the
sub-partition with active set $A$ in which the $n_A$ active
observations are divided into $k$ equal-sized clusters. The remaining
observations in $[n]\setminus A$ are inactive. The sub-partitions
$\varnothing$, $1_A$, and $0_A$ are defined in
\eqref{eq:empty} and \eqref{eq:1a and 0a}.
Then the following properties hold.

\begin{enumerate}
\item[(i)] \textit{Comparison with the active-set extremes.}

Under VI$_{\mathrm{IA}}$ defined in \eqref{eq:ia_vi_lin},
\begin{equation}\notag
\mathrm{VI}_{\mathrm{IA}}(1_A,c_{k,A})
\geq
\mathrm{VI}_{\mathrm{IA}}(0_A,c_{k,A})
\qquad\text{for}\qquad
k\geq\sqrt{n_A}.
\end{equation}
Under BI$_{\mathrm{IA}}$ defined in \eqref{eq:ia-binder} with $a=b$,
\begin{equation}\notag
\mathrm{BI}_{\mathrm{IA}}(1_A,c_{k,A})
>
\mathrm{BI}_{\mathrm{IA}}(0_A,c_{k,A})
\qquad\text{for}\qquad
k>1.
\end{equation}

\item[(ii)] \textit{Comparison with the empty sub-partition.}

Under the symmetric VI$_{\mathrm{IA}}$ setting $m_{ai}=m_{ia}=\tau$,
\begin{equation}\notag
\mathrm{VI}_{\mathrm{IA}}(1_A,c_{k,A})
<
\mathrm{VI}_{\mathrm{IA}}(\varnothing,c_{k,A})
\qquad\text{for}\qquad
k<n^{\tau n_A/n}.
\end{equation}
This condition is equivalent to 
$\tau>\frac{n}{n_A}\log_n k$, in which  $\frac{n}{n_A}\log_n k \in [0, \frac{n}{n_A}]$.

Under symmetric BI$_{\mathrm{IA}}$ with $m_{ai}=m_{ia}=\tau$,
\begin{equation}\notag
\mathrm{BI}_{\mathrm{IA}}(1_A,c_{k,A})
\leq
\mathrm{BI}_{\mathrm{IA}}(\varnothing,c_{k,A})
\qquad\text{for}\qquad
k>1,\quad \tau\geq a/2.
\end{equation}
Equality holds if and only if
$\tau=a/2$ and $
k=n_A=n.$
The condition $\tau\geq a/2$ corresponds to the parameter choice
considered by \citet{buch2024bayesian}.
\end{enumerate}
\end{property}

\begin{remark}
Property~\ref{prop:ia_vi_equal_clusters}(i) illustrates the difference
between VI$_{\mathrm{IA}}$ and BI$_{\mathrm{IA}}$ on a fixed active set $A$. The behaviour of VI$_{\mathrm{IA}}$ is intuitive: small \(k\) gives fewer, larger clusters,
which are naturally closer to the one-cluster extreme \(1_A\), while large \(k\) gives
many smaller clusters, which are closer to the all-singletons extreme \(0_A\). BI$_{\mathrm{IA}}$,
in contrast, disproportionately penalises merging and therefore systematically favours
the all-singletons configuration under equal penalties, regardless of \(k\). For example, \(c_{2,A}\) is closer
to \(1_A\) than to \(0_A\) under VI$_{\mathrm{IA}}$ whenever \(n_A>4\), whereas equal-penalty
BI$_{\mathrm{IA}}$ places \(c_{2,A}\) closer to \(0_A\).

Property~\ref{prop:ia_vi_equal_clusters}(ii) illustrates the role of
the inactive--active penalty by comparing the active $k$-cluster
sub-partition $c_{k,A}$ with the empty sub-partition $\varnothing$.
Under symmetric BI$_{\mathrm{IA}}$, choosing $\tau>a/2$ makes $c_{k,A}$ closer
to $1_A$ than to $\varnothing$ for every non-trivial case $k>1$.
Thus, this choice strongly discourages all-inactive summaries,
regardless of the number of active clusters. By contrast, under VI$_{\mathrm{IA}}$
the comparison depends on $k$, $\tau$, and the active-set size $n_A$,
allowing $\varnothing$ to remain competitive when the inactive--active
penalty is sufficiently small.
\end{remark}

\begin{example}
Let $n_A=6$, $A=\{1,\ldots,6\}$, and
$
c_{2,A}
=
\{\{1,2,3\},\{4,5,6\}\}.
$
Then
$
\mathrm{VI}_{\mathrm{IA}}(1_A,c_{2,A})
=
1
<
\mathrm{VI}_{\mathrm{IA}}(0_A,c_{2,A})
=
\log_2 3
\approx1.58,
$
so $c_{2,A}$ is closer to $1_A$ than to $0_A$ under VI$_{\mathrm{IA}}$.
For BI$_{\mathrm{IA}}$ with $a=b$,
$
\mathrm{BI}_{\mathrm{IA}}(1_A,c_{2,A})
=
9a
>
\mathrm{BI}_{\mathrm{IA}}(0_A,c_{2,A})
=
6a,
$
so $c_{2,A}$ is instead closer to $0_A$.
Thus it shows BI$_{\mathrm{IA}}$'s split and VI$_{\mathrm{IA}}$'s merge behaviours.
\end{example}

\begin{proof}[Proof of Property~\ref{prop:ia_vi_equal_clusters}]\label{app: Proof equal vs full one }
For part (i), the sub-partitions $1_A$, $0_A$, and $c_{k,A}$ all have
the same active set $A$. Hence the active--inactive term in VI$_{\mathrm{IA}}$
vanishes, and VI$_{\mathrm{IA}}$ reduces to ordinary VI on $A$. Therefore,
\begin{align}
\mathrm{VI}_{\mathrm{IA}}(1_A,c_{k,A})
&=
\mathrm{VI}(1_A,c_{k,A})
=
\log_2 k,
\label{eq:VI_1_A_c_k}
\\
\mathrm{VI}_{\mathrm{IA}}(0_A,c_{k,A})
&=
\mathrm{VI}(0_A,c_{k,A})
=
\log_2 n_A-\log_2 k.
\label{eq:VI_0_A_c_k}
\end{align}
It follows that
$
\mathrm{VI}_{\mathrm{IA}}(1_A,c_{k,A})
\geq
\mathrm{VI}_{\mathrm{IA}}(0_A,c_{k,A})
$
if and only if
$
\log_2 k
\geq
\log_2 n_A-\log_2 k,
$
or equivalently,
\[
2\log_2 k\geq\log_2 n_A
\quad\Longleftrightarrow\quad
k^2\geq n_A
\quad\Longleftrightarrow\quad
k\geq\sqrt{n_A}.
\]
We next consider the corresponding BI$_{\mathrm{IA}}$ comparison. Since the
active sets agree, only the pairwise Binder term contributes. For
$\mathrm{BI}_{\mathrm{IA}}(1_A,c_{k,A})$, all pairs of active
observations are together under $1_A$, while pairs belonging to
different clusters of $c_{k,A}$ are separated. For
$\mathrm{BI}_{\mathrm{IA}}(0_A,c_{k,A})$, no pair is together under
$0_A$, while pairs within the same cluster of $c_{k,A}$ are together.

Let
$
s=n_A/k
$
denote the common cluster size. Then
\begin{align}
\mathrm{BI}_{\mathrm{IA}}(1_A,c_{k,A})
&=
a\left\{
\binom{n_A}{2}
-
k\binom{s}{2}
\right\}
=
\frac{a n_A^2}{2}
\left(1-\frac{1}{k}\right),
\label{eq:B_1_A_c_k}
\\
\mathrm{BI}_{\mathrm{IA}}(0_A,c_{k,A})
&=
b\,k\binom{s}{2}
=
\frac{b n_A^2}{2}
\left(
\frac{1}{k}-\frac{1}{n_A}
\right).
\label{eq:B_0_A_c_k}
\end{align}
When $a=b$,
$
\mathrm{BI}_{\mathrm{IA}}(1_A,c_{k,A})
<
\mathrm{BI}_{\mathrm{IA}}(0_A,c_{k,A})
$
if and only if
\(1-1/k<1/k-1/n_A\),
which is equivalent to
\(k<2n_A/(n_A+1)\).
Since \(2n_A/(n_A+1)<2\), this condition cannot hold for any integer
\(k>1\). Therefore,
$\mathrm{BI}_{\mathrm{IA}}(1_A,c_{k,A})
>\mathrm{BI}_{\mathrm{IA}}(0_A,c_{k,A})
$ for all $k>1.$

For part (ii), consider first VI$_{\mathrm{IA}}$. The active set of $\varnothing$
is empty, whereas that of $c_{k,A}$ is $A$. Hence their common active
set is empty and the ordinary VI term vanishes. The only contribution
comes from observations that are inactive under $\varnothing$ but
active under $c_{k,A}$. Thus, under $m_{ai}=m_{ia}=\tau$,
\begin{equation}\notag
\mathrm{VI}_{\mathrm{IA}}(\varnothing,c_{k,A})
=
\frac{\log_2 n}{n}\,\tau n_A
=
\frac{n_A}{n}\tau\log_2 n.
\end{equation}
Using \eqref{eq:VI_1_A_c_k},
$
\mathrm{VI}_{\mathrm{IA}}(1_A,c_{k,A})
<
\mathrm{VI}_{\mathrm{IA}}(\varnothing,c_{k,A})
$
if and only if
$
\log_2 k
<
({n_A}/{n})\tau\log_2 n.
$
Exponentiating gives
$
k<n^{\tau n_A/n}.
$
Equivalently,
\[
\tau
>
\frac{n\log_2 k}{n_A\log_2 n}
=
\frac{n}{n_A}\log_n k.
\]

For BI$_{\mathrm{IA}}$,
$
\mathrm{BI}_{\mathrm{IA}}(\varnothing,c_{k,A})
=(n-1)\tau n_A.
$
Using \eqref{eq:B_1_A_c_k},
$
\mathrm{BI}_{\mathrm{IA}}(1_A,c_{k,A})
<
\mathrm{BI}_{\mathrm{IA}}(\varnothing,c_{k,A})
$
if and only if
$
({a n_A^2}/{2})
\left(1-{1}/{k}\right)
<
(n-1)\tau n_A.
$
Since $n_A>0$ and $k>1$, this is equivalent to
\[
\frac{n_A}{n-1}
<
\frac{2\tau}{a}\frac{k}{k-1}.
\]

Suppose first that $\tau>a/2$. Then
$
\frac{2\tau}{a}>1.
$
Since $k>1$, and $k\leq n_A\leq n$, thus
\[
\frac{2\tau}{a}\frac{k}{k-1}
>
\frac{k}{k-1}, \qquad \frac{n_A}{n-1}
\leq
\frac{k}{k-1},
\]
as
$
k(n-1)-n_A(k-1)
=
k(n-n_A)+(n_A-k)\geq0.
$
Hence
\[
\frac{n_A}{n-1}
<
\frac{2\tau}{a}\frac{k}{k-1},
\]
thus
$
\mathrm{BI}_{\mathrm{IA}}(1_A,c_{k,A})
<
\mathrm{BI}_{\mathrm{IA}}(\varnothing,c_{k,A})
$
for every admissible $k>1$.

Now suppose that $\tau=a/2$. The comparison reduces to
\[
\frac{n_A}{n-1}
\leq
\frac{k}{k-1},
\]
or equivalently,
$
k(n-n_A)+(n_A-k)\geq0,
$
which always holds since $k\leq n_A\leq n$. Equality occurs if and
only if both non-negative terms vanish, namely
$n_A=n $ and
$k=n_A.$
Thus, for $\tau\geq a/2$,
$
\mathrm{BI}_{\mathrm{IA}}(1_A,c_{k,A})
\leq
\mathrm{BI}_{\mathrm{IA}}(\varnothing,c_{k,A}),
$
with equality if and only if
$\tau=a/2 $ and
$k=n_A=n.$
\end{proof}

\begin{property}[Empty versus fully active one-cluster comparison]
\label{prop:empty_vs_full_one}
Let $A\subseteq[n]$ with $n_A=|A|>1$, and consider
$0_A$, $\varnothing$, and $1$, as defined in
\eqref{eq:1a and 0a}, \eqref{eq:empty}, and \eqref{eq:1 and 0},
respectively. For each loss, consider the symmetric active--inactive
penalty setting $m_{ai}=m_{ia}=\tau$.
Then, for VI$_{\mathrm{IA}}$,
\begin{equation}\notag
\mathrm{VI}_{\mathrm{IA}}(\varnothing,0_A)
<
\mathrm{VI}_{\mathrm{IA}}(1,0_A)
\end{equation}
\begin{equation}\notag
\text{iff either} \quad n_A \le \frac{n}{2},
\qquad\text{or}\qquad
n_A>\frac{n}{2}
\ \text{and}\
\tau<
\frac{n\log_2 n_A}{(2n_A-n)\log_2 n}.
\end{equation}
In particular, when $A=[n]$ and $\tau=1$,
 $ \mathrm{VI}_{\mathrm{IA}}(\varnothing,0)
=\mathrm{VI}_{\mathrm{IA}}(1,0). $

For BI$_{\mathrm{IA}}$,
\begin{equation}\notag
\mathrm{BI}_{\mathrm{IA}}(\varnothing,0_A)
<
\mathrm{BI}_{\mathrm{IA}}(1,0_A)
\end{equation}
\begin{equation}\notag
\text{iff either} \quad n_A \le \frac{n}{2},
\qquad\text{or}\qquad
n_A>\frac{n}{2}
\ \text{and}\
\tau<
\frac{a n_A(n_A-1)}
{2(n-1)(2n_A-n)} .
\end{equation}
In particular, when $A=[n]$ and $\tau=a/2$,
$
\mathrm{BI}_{\mathrm{IA}}(\varnothing,0)
=
\mathrm{BI}_{\mathrm{IA}}(1,0).
$
\end{property}

\begin{remark}
\label{rem:empty_vs_full_one_interpretation}
For both VI$_{\mathrm{IA}}$ and BI$_{\mathrm{IA}}$, if $n_A\leq n/2$, then
$\varnothing$ is closer to $0_A$ than the fully active one-cluster
sub-partition $1$, independently of the active--inactive penalty.
When $n_A>n/2$, the comparison instead depends on the penalty
parameter $\tau$.

At the fully active extreme $A=[n]$, the two reference
sub-partitions $\varnothing$ and $1$ are equally distant from $0$
when $\tau=1$ for VI$_{\mathrm{IA}}$ and when $\tau=a/2$ for BI$_{\mathrm{IA}}$.
\end{remark}

\begin{example}
Let $n=4$ as in Figure~\ref{fig:subpartition_lattice}, and take
$A=\{1,2\}$, $0_A=\{\{1\},\{2\}\}.$
For VI$_{\mathrm{IA}}$ with $\tau=1$,
$
\mathrm{VI}_{\mathrm{IA}}(\varnothing,0_A)
=\frac{2}{4}\log_2 4 =1,
$
whereas
$
\mathrm{VI}_{\mathrm{IA}}(1,0_A)
= \frac{2}{4}\log_2 4+\log_2 2 =2.
$
For BI$_{\mathrm{IA}}$ with $\tau=a/2$,
$
\mathrm{BI}_{\mathrm{IA}}(\varnothing,0_A)
= 3\cdot\frac{a}{2}\cdot2
= 3a,
$
whereas
$
\mathrm{BI}_{\mathrm{IA}}(1,0_A)
= 3\cdot\frac{a}{2}\cdot2
+ a\binom{2}{2}
= 4a.
$
Thus, in this case $n_A=n/2$, under both losses, the all-inactive summary is preferred over forcing them into a  single active cluster.
\end{example}

\begin{proof}[Proof of Property~\ref{prop:empty_vs_full_one}]
\label{app:Proof empty vs full one}
We first consider VI$_{\mathrm{IA}}$. The sub-partition $0_A$ has active set $A$,
whereas $\varnothing$ has empty active set. Their common active set is
therefore empty, so the ordinary VI term vanishes. Hence
\[
\mathrm{VI}_{\mathrm{IA}}(\varnothing,0_A)
=
\frac{n_A}{n}\tau\log_2 n.
\]

Next consider $1$ and $0_A$. The active set of $1$ is $[n]$, whereas
that of $0_A$ is $A$. Thus the active--inactive mismatch consists of
the $n-n_A$ observations in $[n]\setminus A$. On the common active set
$A$, the corresponding restrictions are $1_A$ and $0_A$. Therefore
\[
\mathrm{VI}_{\mathrm{IA}}(1,0_A)
=
\frac{n-n_A}{n}\tau\log_2 n
+
\mathrm{VI}(1_A,0_A).
\]
Since 
$
\mathrm{VI}(1_A,0_A)=\log_2 n_A,
$
we obtain
$
\mathrm{VI}_{\mathrm{IA}}(1,0_A)
=
(({n-n_A})/{n})\tau\log_2 n
+
\log_2 n_A.
$
Consequently,
$
\mathrm{VI}_{\mathrm{IA}}(\varnothing,0_A)
<
\mathrm{VI}_{\mathrm{IA}}(1,0_A)
$
if and only if
\[
\frac{2n_A-n}{n}\tau\log_2 n
<
\log_2 n_A.
\]
If $n_A\leq n/2$, the left-hand side is non-positive, whereas
$\log_2 n_A>0$ because $n_A>1$, so the inequality always holds.
If $n_A>n/2$, then $2n_A-n>0$, and the inequality is equivalent to
\[
\tau<
\frac{n\log_2 n_A}
     {(2n_A-n)\log_2 n}.
\]
This proves the VI$_{\mathrm{IA}}$ result. When $A=[n]$, the threshold reduces to
$\tau<1$, so equality holds at $\tau=1$.

We now consider BI$_{\mathrm{IA}}$. Since $0_A$ has active set $A$ and
$\varnothing$ has empty active set,
\[
\mathrm{BI}_{\mathrm{IA}}(\varnothing,0_A)
=
(n-1)\tau n_A.
\]

For the pair $(1,0_A)$, the $n-n_A$ observations in $[n]\setminus A$
are active under $1$ but inactive under $0_A$, giving the
active--inactive contribution
$
(n-1)\tau(n-n_A).
$
On the common active set $A$, every pair of observations is together
under $1_A$ but separated under $0_A$, giving the Binder contribution
$
a\binom{n_A}{2}.
$
Hence
\[
\mathrm{BI}_{\mathrm{IA}}(1,0_A)
=
(n-1)\tau(n-n_A)
+
a\binom{n_A}{2}.
\]

Therefore,
$
\mathrm{BI}_{\mathrm{IA}}(\varnothing,0_A)
<
\mathrm{BI}_{\mathrm{IA}}(1,0_A)
$
if and only if
\[
(n-1)\tau(2n_A-n)
<
\frac{a n_A(n_A-1)}{2}.
\]
If $n_A\leq n/2$, the left-hand side is non-positive, whereas the
right-hand side is positive because $n_A>1$, so the inequality always
holds. If $n_A>n/2$, it is equivalent to
\[
\tau<
\frac{a n_A(n_A-1)}
     {2(n-1)(2n_A-n)}.
\]
This proves the BI$_{\mathrm{IA}}$ result. When $A=[n]$, the threshold reduces
to $\tau<a/2$, so equality holds at $\tau=a/2$.
\end{proof}

\begin{property}[Two-cluster comparisons]
\label{prop:two_cluster_fixed_active}
Let $A\subseteq[n]$ be a non-empty active set with $n_A=|A|>1$, and let
$c_{A,r}$ be the sub-partition with active set $A$ whose restriction to $A$
consists of two clusters of sizes $r$ and $n_A-r$, where
$1\leq r\leq n_A-1$. Write
\[
p=\frac{r}{n_A},
\qquad
H_2(p)
=
-p\log_2 p-(1-p)\log_2(1-p).
\]
Let $1_A$, $0_A$, and $\varnothing$ be defined as in
\eqref{eq:1a and 0a} and \eqref{eq:empty}. Then the following
comparisons hold.

\begin{enumerate}
\item[(i)] \emph{Comparison with the active-set extremes.}

Under VI$_{\mathrm{IA}}$,
\begin{equation}\notag
\mathrm{VI}_{\mathrm{IA}}(1_A,c_{A,r})
<
\mathrm{VI}_{\mathrm{IA}}(0_A,c_{A,r})
\quad\Longleftrightarrow\quad
H_2(p)<\frac{1}{2}\log_2 n_A.
\end{equation}
In particular, this inequality holds for every admissible $r$ whenever
$n_A\geq5$.

Under BI$_{\mathrm{IA}}$ with equal pairwise penalties $a=b$,
\begin{equation}\notag
\mathrm{BI}_{\mathrm{IA}}(1_A,c_{A,r})
<
\mathrm{BI}_{\mathrm{IA}}(0_A,c_{A,r})
\quad\Longleftrightarrow\quad
\bigl|p - 1/2\bigr| > 1/{2\sqrt{n_A}}.
\end{equation}
Thus BI$_{\mathrm{IA}}$ favours \(1_A\) only for sufficiently imbalanced two-cluster splits.
In particular, for an equal-size split \(p=1/2\), 
\(c_{A,r}\) is closer to \(0_A\) than to \(1_A\) for every $n_A$.

\item[(ii)] \emph{Comparison with the empty sub-partition.}

Under the symmetric VI$_{\mathrm{IA}}$ setting $m_{ai}=m_{ia}=\tau$,
\begin{equation}\notag
\mathrm{VI}_{\mathrm{IA}}(1_A,c_{A,r})
<
\mathrm{VI}_{\mathrm{IA}}(\varnothing,c_{A,r})
\quad\Longleftrightarrow\quad
\tau>
\frac{nH_2(p)}{n_A\log_2 n}.
\end{equation}
A sufficient condition for this inequality to hold for every admissible
$r$ is
$\tau>\frac{n}{n_A\log_2 n}$.

Under symmetric BI$_{\mathrm{IA}}$ with $m_{ai}=m_{ia}=\tau$,
\begin{equation}\notag
\mathrm{BI}_{\mathrm{IA}}(1_A,c_{A,r})
<
\mathrm{BI}_{\mathrm{IA}}(\varnothing,c_{A,r})
\quad\Longleftrightarrow\quad
\tau>
\frac{a n_A}{n-1}p(1-p).
\end{equation}
A sufficient condition for this inequality to hold for every admissible
$r$ is
$\tau>
\frac{a n_A}{4(n-1)}$.
In particular, for $n>2$, the choice $\tau\geq a/2$ implies
\[
\mathrm{BI}_{\mathrm{IA}}(1_A,c_{A,r})
<
\mathrm{BI}_{\mathrm{IA}}(\varnothing,c_{A,r})
\qquad
\text{for all }1\leq r\leq n_A-1.
\]
\end{enumerate}
\end{property}

\begin{remark}
\label{rem:two_cluster_fixed_active_interpretation}
Property~\ref{prop:two_cluster_fixed_active}(i) reveals a notable contrast between the two losses on two-cluster splits. VI$_{\mathrm{IA}}$ tends to favour merging once the active set $ n_A \geq 5$, which is intuitively the right behaviour, since a two-cluster partition is structurally much closer to a single merged cluster. Equal-penalty BI$_{\mathrm{IA}}$ behaves less reasonably: for two equal-sized clusters ($p = 1/2$), all-singletons $0_A$ is still preferred for every $n_A$. This mirrors the splitting tendency of Binder-type losses documented in \cite{wade2018bayesian} and shows that it persists in cases with noise points.
Property~\ref{prop:two_cluster_fixed_active}(ii) isolates the role of
the active--inactive penalty. For VI$_{\mathrm{IA}}$, the threshold depends on the
binary entropy $H_2(p)$ and therefore on the balance of the two
clusters. For BI$_{\mathrm{IA}}$, when $n>2$, the choice $\tau\geq a/2$ is
sufficient to place every two-cluster sub-partition $c_{A,r}$ closer
to $1_A$ than to $\varnothing$.
\end{remark}

\begin{proof}[Proof of Property~\ref{prop:two_cluster_fixed_active}]
\label{app:proof_two_cluster_fixed_active}
For part (i), the sub-partitions $1_A$, $0_A$, and $c_{A,r}$ all have
the same active set $A$. Hence the active--inactive term in VI$_{\mathrm{IA}}$
vanishes, and VI$_{\mathrm{IA}}$ reduces to ordinary VI on $A$.
The entropy of the two-cluster partition $c_{A,r}$ is
\[
H(c_{A,r})
=
-\frac{r}{n_A}\log_2\frac{r}{n_A}
-\frac{n_A-r}{n_A}\log_2\frac{n_A-r}{n_A}
=
H_2(p).
\]
Since $1_A$ consists of a single cluster,
$
\mathrm{VI}_{\mathrm{IA}}(1_A,c_{A,r})
=
\mathrm{VI}(1_A,c_{A,r})
=
H_2(p).
$
Since $0_A$ is the all-singletons partition,
$
\mathrm{VI}_{\mathrm{IA}}(0_A,c_{A,r})
=
\mathrm{VI}(0_A,c_{A,r})
=
\log_2 n_A-H_2(p).
$
Therefore
$
\mathrm{VI}_{\mathrm{IA}}(1_A,c_{A,r})
<
\mathrm{VI}_{\mathrm{IA}}(0_A,c_{A,r})
$
if and only if
$
H_2(p)<\log_2 n_A-H_2(p),
$
or equivalently,
$
H_2(p)<({1}/{2})\log_2 n_A.
$
Since $H_2(p)\leq1$, this condition holds for every $p$ whenever
$
({1}/{2})\log_2 n_A>1,
$
which is equivalent to $n_A>4$. Hence it holds for all admissible $r$
whenever $n_A\geq5$.

We next consider BI$_{\mathrm{IA}}$. For
$\mathrm{BI}_{\mathrm{IA}}(1_A,c_{A,r})$, the disagreement consists of
pairs that are together under $1_A$ but separated under $c_{A,r}$.
There are $r(n_A-r)$ such pairs, so
$
\mathrm{BI}_{\mathrm{IA}}(1_A,c_{A,r})
=
a\,r(n_A-r).
$
For $\mathrm{BI}_{\mathrm{IA}}(0_A,c_{A,r})$, the disagreement
consists of pairs that are together under $c_{A,r}$ but separated
under $0_A$. Hence
\[
\mathrm{BI}_{\mathrm{IA}}(0_A,c_{A,r})
=
b\left\{
\binom{r}{2}
+
\binom{n_A-r}{2}
\right\}.
\]
When $a=b$,
$
\mathrm{BI}_{\mathrm{IA}}(1_A,c_{A,r})
<
\mathrm{BI}_{\mathrm{IA}}(0_A,c_{A,r})
$
if and only if
\[
r(n_A-r)
<
\binom{r}{2}
+
\binom{n_A-r}{2}.
\]
Substituting $r=n_Ap$ and simplifying gives
$
p(1-p)
<
({n_A-1})/{4n_A}.
$
Using
$
p(1-p)
=
{1}/{4}
-
\left(p-{1}/{2}\right)^2,
$
this is equivalent to
$
\left(p-{1}/{2}\right)^2
>
{1}/{4n_A},
$
and hence
$
\left|p-{1}/{2}\right|
>
{1}/{2\sqrt{n_A}}.
$
This proves part (i).

For part (ii), consider first VI$_{\mathrm{IA}}$. The common active set of
$\varnothing$ and $c_{A,r}$ is empty, so the ordinary VI term
vanishes. Under $m_{ai}=m_{ia}=\tau$,
\[
\mathrm{VI}_{\mathrm{IA}}(\varnothing,c_{A,r})
=
\frac{n_A}{n}\tau\log_2 n, \qquad
\mathrm{VI}_{\mathrm{IA}}(1_A,c_{A,r})
=
H_2(p),
\]
we obtain
$
\mathrm{VI}_{\mathrm{IA}}(1_A,c_{A,r})
<
\mathrm{VI}_{\mathrm{IA}}(\varnothing,c_{A,r})
$
if and only if
$
H_2(p)
<
({n_A}/{n})\tau\log_2 n,
$
or equivalently,
$
\tau>
({nH_2(p)})/({n_A\log_2 n}).
$
Since $H_2(p)\leq1$, the sufficient condition
\[
\tau>\frac{n}{n_A\log_2 n}
\]
ensures the inequality for every admissible $r$.

Finally, for symmetric BI$_{\mathrm{IA}}$,
$
\mathrm{BI}_{\mathrm{IA}}(\varnothing,c_{A,r})
=
(n-1)\tau n_A.
$
Therefore
$
\mathrm{BI}_{\mathrm{IA}}(1_A,c_{A,r})
<
\mathrm{BI}_{\mathrm{IA}}(\varnothing,c_{A,r})
$
if and only if
$
a\,r(n_A-r)
<
(n-1)\tau n_A.
$
Using $r=n_Ap$, this becomes
$
a n_A^2p(1-p)
<
(n-1)\tau n_A,
$
or equivalently,
\[
\tau>
\frac{a n_A}{n-1}p(1-p).
\]
Since $p(1-p)\leq1/4$, the sufficient condition
$
\tau>
{a n_A}/{4(n-1)}
$
ensures the inequality for every admissible $r$.

Moreover, for $n>2$ and $n_A\leq n$,
\[
\frac{a n_A}{4(n-1)}
\leq
\frac{a n}{4(n-1)}
<
\frac{a}{2}.
\]
so $\tau\geq a/2$ implies
$
\tau>
{a n_A}/{4(n-1)},
$
and thus
$
\mathrm{BI}_{\mathrm{IA}}(1_A,c_{A,r})
<
\mathrm{BI}_{\mathrm{IA}}(\varnothing,c_{A,r})
$
for every admissible $r$.
\end{proof}

\section{Choice of the level distribution \texorpdfstring{$H$}{H}}
\label{APP:h}

The level distribution $H$ may be specified directly on $\lambda$, for
example as a uniform distribution over a plausible interval
$[\lambda_{\min},\lambda_{\max}]$, or as a discrete distribution over a set
of candidate levels. For computation, we approximate $H$ using a finite grid
$\{\lambda_g\}_{g=1}^G$ with associated weights
$\{\omega_g\}_{g=1}^G$, where $\omega_g\geq0$ and
$\sum_{g=1}^G\omega_g=1$. Combined with $M$ posterior density draws, this
gives the weighted empirical level-marginal posterior with atom weights
$\eta_{m,g}=\omega_g/M$.

Alternatively, uncertainty may be specified through the inactive fraction
$\nu\in(0,1)$. Given a reference density $\bar f$, define
\begin{equation}
\lambda_\nu
=
q_\nu\!\left(
\{\bar f(y_i):i=1,\ldots,n\}
\right),
\label{eq:lambda_from_nu}
\end{equation}
where $q_\nu$ denotes the empirical $\nu$-quantile. A distribution $H_\nu$
on plausible inactive fractions therefore induces a distribution $H$ on
density levels through the mapping $\nu\mapsto\lambda_\nu$. This
parameterisation is scale-free and can be more interpretable when uncertainty
is naturally expressed in terms of the proportion of observations regarded
as background or noise.

For a fixed density draw $f$ evaluated on the observed sample, the active set
\[
A_\lambda(f)=\{i:f(y_i)\geq\lambda\}
\]
changes only when $\lambda$ crosses one of the values
$f(y_1),\ldots,f(y_n)$. Consequently, the sample level-set clustering
$\lambda\mapsto\psi_\lambda(f)$ is piecewise constant between consecutive
ordered density values, apart from any additional dependence of the graph
construction on $\lambda$. This makes a finite-grid approximation natural.
In practice, the grid should cover the relevant support of $H$; its adequacy
can be assessed by increasing $G$ and checking the stability of the resulting
particle summaries and weights.

All level grids used in our experiments are equally weighted,
$\omega_g=1/G$. In the synthetic separation experiment, the local grid is
specified in inactive-fraction space and mapped to density levels through
\eqref{eq:lambda_from_nu}. For the harbour analysis, we instead use a narrow
local grid directly on the density-level scale to assess sensitivity to the
level-set threshold; the corresponding inactive fractions are reported in
Appendix~\ref{app:harbour} for interpretation.

\section{Implementation and initialisation of WASABI-SuP}\label{app:wasabi_initialisation}

Algorithm~\ref{alg:wasabi} gives the implementation of WASABI-SuP. The
algorithm follows an alternating optimisation scheme, with an
\emph{$\mathcal N$-update} step and a \emph{particle-search} step. Given the
current particles, the $\mathcal N$-update assigns each posterior
sub-partition draw to its nearest particle under the ground loss $D$, with a
fixed rule used to break ties. This defines the empirical regions of
attraction $\{\mathcal N_\ell\}_{\ell=1}^L$. For the VI$_{\mathrm{IA}}$ loss, the regions of attraction reflect both sources of sub-partition disagreement. In particular, two draws may induce similar
clusterings on their common active observations while remaining distant
because they disagree on which observations are active. WASABI-SuP can
therefore separate posterior modes that differ primarily in their active sets, as well as those that differ in clustering structure.

In the particle-search step, each particle is updated by approximately
minimising the posterior loss within its current Voronoi cell using
the modified SALSO search described in Appendix~\ref{app: modified salso}. Importantly,
the updated particle is not restricted to be one of the posterior draws in
that region. Instead, SALSO searches over the space of sub-partitions and can
construct a new representative clustering structure that better summarises
the assigned draws. Since this search is approximate, a proposed update is
retained only when it does not increase the corresponding regional objective.
For region $\mathcal N_\ell$, the resulting particle weight and within-region
loss are
\[
w_\ell
=
\sum_{m\in\mathcal N_\ell}\eta_m,
\qquad
l_\ell
=
\frac{1}{w_\ell}
\sum_{m\in\mathcal N_\ell}
\eta_m D(\theta^{(m)},\theta_\ell^*),
\]
for $w_\ell>0$, with $l_\ell=0$ when $w_\ell=0$. Hence the empirical
WASABI-SuP objective is $\sum_{\ell=1}^L w_\ell l_\ell$, as in
\eqref{eq:wasabi_sup_objective}.

\paragraph{Initialisation.} 
As with other \(k\)-means-like procedures, the optimisation of \eqref{eq:wasabi_sup_objective} is non-convex and convergence is only guaranteed to a local optimum, therefore sensitive to initialisation.
We consequently run the algorithm from multiple initialisations and retain the solution with the smallest objective value.
Our default is a $++$-type initialisation analogous to
$k$-means++ \citep{vassilvitskii2006k}. The first particle is sampled from
the posterior draws according to their weights $\eta_m$. Given
$\theta_1^*,\ldots,\theta_{\ell-1}^*$, the next particle is sampled with
probability proportional to
$\eta_m
\min_{j<\ell}
D\!\left(\theta^{(m)},\theta_j^*\right)$.
Thus, posterior draws that carry substantial mass and are far from the
particles already selected are more likely to initialise a new particle.
For VI$_{\mathrm{IA}}$, this promotes coverage of configurations differing either in
active-set membership or in clustering structure. The procedure requires
only distances to the currently selected particles, rather than the full
pairwise distance matrix. In all experiments, we combine this initialisation
with multiple restarts and report the solution attaining the smallest
objective.

\begin{algorithm}[htp!]
\caption{WASABI-SuP: find the approximate posterior under a loss $D$}
\label{alg:wasabi}
\begin{algorithmic}
\State \textbf{Input:} weighted posterior samples
$\{(\theta^{(m)},\eta_m)\}_{m=1}^M$, with
$\eta_m\geq0$ and $\sum_{m=1}^M\eta_m=1$;
loss $D$, number of particles $L$, initialization method $\texttt{init}$,
tolerance $\epsilon$, maximum number of iterations $I_{\max}$
\State \textbf{Initialize} $\theta^*_{1:L}$ using method $\texttt{init}$; set $i=0$
\Repeat \qquad $i\gets i+1$
    \State \textbf{\textit{$\mathcal N$-update} step:}
    \State Compute $D(\theta^{(m)},\theta^*_\ell)$ for all $m$ and $\ell$
    \State Assign each $\theta^{(m)}$ to the closest particle $\theta^*_\ell$
    (fixed rule to break ties) and update the Voronoi cells $\mathcal N_\ell$
    \If{any Voronoi cell $\mathcal N_\ell$ is empty}
        \State Replace $\theta^*_\ell$ with a distant posterior sample
        $\theta^{(m)}$ with $\eta_m>0$ and repeat the $\mathcal N$-update
    \EndIf

    \State \textbf{\textit{Particle-search} step:}
    \For{$\ell$ in $1,\ldots,L$}
        \State Update particle:
        $\theta^*_\ell \approx
        \arg\min_{\theta\in\mathcal S_n}
        \sum_{\theta^{(m)}\in\mathcal N_\ell}
        \eta_m D(\theta^{(m)},\theta)$
        \Comment{modified SALSO}

        \State Retain the proposed particle only if the weighted regional objective does not increase
         
        \State Update weight:
        $w_\ell \gets
        \sum_{\theta^{(m)}\in\mathcal N_\ell}\eta_m$

        \State Update within-cell loss:
        $l_\ell \gets
        \frac{1}{w_\ell}
        \sum_{\theta^{(m)}\in\mathcal N_\ell}
        \eta_m D(\theta^{(m)},\theta^*_\ell)$
    \EndFor

    \State Update the objective:
    $\widehat{\mathcal T}_D^{(i)}
    \gets
    \sum_{\ell=1}^L w_\ell l_\ell$
\Until{change of $\widehat{\mathcal T}_D$ is less than $\epsilon$ or
$i=I_{\max}$}
\State \textbf{Return}
$\{\theta^*_1,\ldots,\theta^*_L\}$,
$\mathbf w=(w_1,\ldots,w_L)$.
\end{algorithmic}
\end{algorithm}

\subsection{Modified SALSO search.}\label{app: modified salso}
The particle-search step adapts customised SALSO \citep{dahl2022search} to minimise VI$_{\mathrm{IA}}$ over sub-partitions. Following the representation used by \citet{buch2024bayesian}, a sub-partition is encoded by an allocation vector in which label $0$ denotes an inactive observation and positive integer labels
denote active clusters.
The search otherwise retains the standard SALSO structure of \cite{dahl2022search}: observations are updated greedily in random order, followed by repeated sweetening sweeps until no further improving single-observation move is found, and multiple random restarts are used with the lowest-loss solution retained. 

For general weighted posterior draws, the regional objective is the weighted posterior loss over the draws assigned to the particle.
For equally weighted posterior draws, as used in the experiments below, the
regional objective is the empirical mean VI$_{\mathrm{IA}}$ over the draws assigned to the
particle.  The loss evaluation,
however, differs from standard VI-based SALSO. For a candidate
sub-partition $\theta$ and posterior draw $\theta^{(m)}$, VI$_{\mathrm{IA}}$ decomposes as
\[
\mathrm{VI}_{\mathrm{IA}}
\!\left(\theta^{(m)},\theta\right)
=
\frac{\log_2 n}{n}
\left[
m_{ai}\lvert A_m\cap I\rvert
+
m_{ia}\lvert I_m\cap A\rvert
\right]
+
\mathrm{VI}
\!\left(
\theta^{(m)}_{B_m},
\theta_{B_m}
\right),
\]
where $A_m$ and $A$ are the respective active sets,
$I_m=A_m^c$, $I=A^c$, and
$B_m=A_m\cap A$ is the common active set.

This form permits the active--inactive contribution of a proposed move to be
updated directly from the corresponding mismatch counts. 
The restricted-VI term is instead evaluated separately for each posterior draw on its current common active set $B_m$. It cannot be reduced to a fixed posterior similarity
matrix: changing the activity status of an observation also changes $B_m$, and hence
the partition on which VI is evaluated.
We consequently compute the restricted VI exactly from the
draw-specific contingency table (as in Appendix \ref{app:Contingency-table}) for each candidate move. Cached mismatch
counts and current restricted-VI values are maintained during a sweep so that
only the quantities affected by the proposed move need to be recomputed.

\section{Posterior expected VI$_{\mathrm{IA}}$ loss}
\label{app:expected_iavi}

For a candidate sub-partition $\hat\theta$ with active set $\hat A$,
define the posterior activity probabilities
\[
p_i
=
\Pr(c_i\neq0\mid y),
\qquad i=1,\ldots,n.
\]
Using the asymmetric VI$_{\mathrm{IA}}$ formulation in
\eqref{eq:ia_vi_lin general}, taking posterior expectations gives
\begin{align}
R_{\mathrm{IA\text{-}VI}}(\hat\theta)
:={}&
\mathbb E\!\left[
\mathrm{VI}_{\mathrm{IA}}(\theta,\hat\theta)
\mid y
\right]
\notag\\
={}&
\frac{\log_2 n}{n}
\sum_{i=1}^{n}
\left[
m_{ai}p_i\,\mathbb I(\hat c_i=0)
+
m_{ia}(1-p_i)\,\mathbb I(\hat c_i\neq0)
\right]
+
\mathbb E\!\left[
\mathrm{VI}(\rho_B,\hat\rho_B)
\mid y
\right],
\qquad
B=A\cap\hat A.
\notag
\end{align}
Thus, the active--inactive contribution depends only on the marginal
activity probabilities, whereas the clustering contribution depends on
the sampled partition through the restricted VI term.

For a weighted empirical posterior
\[\widehat\pi_M
=
\sum_{m=1}^{M}\eta_m\delta_{\theta^{(m)}},
\qquad
\eta_m\geq0,
\qquad
\sum_{m=1}^{M}\eta_m=1.\]
Define
\[\widehat p_i
=
\sum_{m=1}^{M}
\eta_m\,\mathbb I(c_i^{(m)}\neq0),
\qquad
B^{(m)}=A^{(m)}\cap\hat A.\]
Then
\begin{align}
\widehat R_{\mathrm{IA\text{-}VI}}(\hat\theta)
={}&
\frac{\log_2 n}{n}
\sum_{i=1}^{n}
\left[
m_{ai}\widehat p_i\,\mathbb I(\hat c_i=0)
+
m_{ia}(1-\widehat p_i)\,\mathbb I(\hat c_i\neq0)
\right]
+ \sum_{m=1}^{M}
\eta_m\,
\mathrm{VI}\!\left(
\rho_{B^{(m)}}^{(m)},
\hat\rho_{B^{(m)}}
\right).
\notag
\end{align}
Equivalently,
\begin{equation}
    \widehat R_{\mathrm{IA\text{-}VI}}(\hat\theta)
=
\sum_{m=1}^{M}
\eta_m\,
\mathrm{VI}_{\mathrm{IA}}
\!\left(\theta^{(m)},\hat\theta\right).
\notag
\end{equation}
For equally weighted posterior samples, $\eta_m=1/M$. 
Within a WASABI-SuP region $\mathcal N_\ell$, let
\[
w_\ell
=
\sum_{m:\theta^{(m)}\in\mathcal N_\ell}\eta_m.
\]
The same expression is applied to the conditional empirical posterior
within the region, using normalised weights
$\eta_m/w_\ell$ for $\theta^{(m)}\in\mathcal N_\ell$.
The corresponding regional activity probabilities must therefore be
recomputed using only the draws assigned to that region.

This decomposition highlights an important computational distinction
from BI$_{\mathrm{IA}}$. The posterior expected BI$_{\mathrm{IA}}$ loss can be expressed
through marginal activity probabilities and pairwise active
co-clustering probabilities. In contrast, the restricted VI term
depends on the full contingency table between each sampled partition
and the candidate partition on their common active set. Consequently,
the posterior expected VI$_{\mathrm{IA}}$ loss cannot in general be evaluated from
a posterior similarity matrix alone and requires sample-level
partition information.

\paragraph{Contingency-table evaluation.}
\label{app:Contingency-table}

For each posterior sample $\theta^{(m)}$, define the common-active set
$B^{(m)}
=
A^{(m)}\cap\hat A$
and the corresponding restricted contingency counts
\[
n_{kj}^{(m)}
=
\Bigl|
\bigl\{
i\in B^{(m)}
:
c_i^{(m)}=k,\;
\hat c_i=j
\bigr\}
\Bigr|,
\qquad
k\geq1,\;j\geq1.
\]
The row sums, column sums, and common-active-set size are
\[
n_{k+}^{(m)}
=
\sum_{j\geq1}n_{kj}^{(m)},
\qquad
n_{+j}^{(m)}
=
\sum_{k\geq1}n_{kj}^{(m)},
\qquad
n_B^{(m)}
=
\sum_{k\geq1}\sum_{j\geq1}n_{kj}^{(m)}
=
|B^{(m)}|.
\]
Whenever $n_B^{(m)}>0$, the restricted VI term is
\begin{align}
\mathrm{VI}\!\left(
\rho_{B^{(m)}}^{(m)},
\hat\rho_{B^{(m)}}
\right)
={}&
-\sum_{k\geq1}
\frac{n_{k+}^{(m)}}{n_B^{(m)}}
\log_2
\frac{n_{k+}^{(m)}}{n_B^{(m)}}
-
\sum_{j\geq1}
\frac{n_{+j}^{(m)}}{n_B^{(m)}}
\log_2
\frac{n_{+j}^{(m)}}{n_B^{(m)}}
-
2\sum_{k\geq1}\sum_{j\geq1}
\frac{n_{kj}^{(m)}}{n_B^{(m)}}
\log_2
\frac{n_{kj}^{(m)}n_B^{(m)}}
     {n_{k+}^{(m)}n_{+j}^{(m)}}.
\notag
\end{align}
Terms with $n_{kj}^{(m)}=0$ are taken to contribute zero. When
$n_B^{(m)}=0$, we set
$\mathrm{VI}\!\left(
\rho_{B^{(m)}}^{(m)},
\hat\rho_{B^{(m)}}
\right)
=0.$

\section{Computational complexity}
\label{app:complexity}

We summarise the computational cost of the two principal stages of the
proposed procedure: generation of posterior density draws by predictive
resampling and posterior summarisation by WASABI-SuP. Let $d$ denote the
data dimension, $n$ the number of observations, $M$ the number of posterior
density draws, $G$ the number of levels used for level marginalisation,
$L$ the number of WASABI-SuP particles, and $K_{\max}$ the largest number
of active clusters among the posterior sub-partitions. For convenience, 
write $R$ for the total number of sub-partitions entering WASABI-SuP where $R = M$ for fixed-level analysis and $R = MG$ for level-marginal analysis.

\paragraph{Predictive resampling for posterior draws.}
Our density posterior is generated using the multivariate copula recursion
of \citet{fong2023martingale}. As shown therein, constructing the predictive
density $p_n$ requires an initial $O(n^2d)$ computation, after which
evaluating $p_n$ at a single location requires $O(nd)$ operations.
With $N-n$ forward predictive-resampling steps, generating one posterior
density value at a fixed location requires $O((N-n)d)$ additional operations.
Since clustering requires the density only at the $n$ observed locations,
a direct implementation generating $M$ posterior density draws therefore
requires $O\!\left(n^2d + Mn(N-n)d\right)$
operations after sharing the initial fit across draws. Importantly, the forward predictive-resampling paths are independent and can be parallelised across draws. \citet{fong2023martingale} provide detailed runtime comparisons against MCMC-based inference, demonstrating improved computational savings.

For a given density draw and level, identifying the active set requires $O(n)$ threshold comparisons. Once the required neighbourhood graph is
available, restricting it to the active observations and finding connected
components is linear in the number of vertices and edges. For the reciprocal $k$-nearest-neighbour
graph, the graph contains at most $O(nk)$ edges, so this stage is $O(nk)$ per draw--level pair once the neighbour relations have been constructed. Across the level-marginal posterior, the graph-processing
cost is therefore $O(MGnk)$ after neighbour construction.

\paragraph{WASABI-SuP.}
A VI$_{\mathrm{IA}}$ loss between two sub-partitions can be evaluated in $O(n)$ time: the active--inactive contribution and the contingency counts required by
the restricted VI term are accumulated in a single pass through the allocation vectors. Hence, assigning all $R$ posterior draws to their
closest among $L$ particles in the $\mathcal N$-update requires $O(nRL)$ operations. The modified SALSO search used for the particle updates has
the same asymptotic order as its full-partition counterpart, since the additional active--inactive term is linear in $n$. Following the SALSO
complexity used by \citet{balocchi2025understanding}, the combined regional search costs $O(nR K_{\max}^2)$.
Consequently, one outer WASABI-SuP iteration has complexity
$O\left(
nR\left(L+K_{\max}^2\right)
\right)$,
and $I$ outer iterations require
$O\left(
I nR\left(L+K_{\max}^2\right)
\right)$.

\section{Theoretical guarantees for level-marginal WASABI-SuP}
\label{ssec:level-marginal-theory}
We state the following results conditionally on density-posterior contraction
and stability of the level-set clustering construction. This separates the
posterior-summarisation theory from the particular density-inference mechanism
used to generate the draws. Throughout this section, the level distribution $H$ is treated as
fixed rather than data-dependent.

For the theoretical analysis, it is useful to distinguish the
computational sample-graph clustering from the population level-set
clustering. Let
\[
\widetilde\psi_{n,\lambda}(f)\in\mathcal S_n
\]
denote the sample graph sub-partition used computationally, including
any $n$-dependent graph tuning parameters. This is the map denoted by
$\psi_\lambda(f)$ in Section~\ref{sec: graph}, with its dependence on
$n$ made explicit here.

In contrast, let
\[
\psi^\circ_{n,\lambda}(f_0)\in\mathcal S_n
\]
denote the sub-partition of the observed indices induced by the
connected components of the population level set
\[
L_\lambda(f_0)
=
\{y:f_0(y)\geq\lambda\}.
\]
Thus, observation $i$ is active if $f_0(y_i)\geq\lambda$, and two
active observations belong to the same block of
$\psi^\circ_{n,\lambda}(f_0)$ if they lie in the same connected
component of $L_\lambda(f_0)$.

Accordingly, the level-marginal posterior defined in
Section~\ref{sec:wasabi-lambda} can be written as
\begin{equation}
\Pi_n^\Lambda
=
\int
\int_\Lambda
\delta_{\widetilde\psi_{n,\lambda}(f)}
\,H(d\lambda)\,
\Pi_n(df\mid y),
\label{eq:Pi-Lambda-theory}
\end{equation}
while the corresponding true-density level-marginal target is
\begin{equation}
P_{0,n}^\Lambda
=
\int_\Lambda
\delta_{\psi^\circ_{n,\lambda}(f_0)}
\,H(d\lambda).
\label{eq:P0-Lambda}
\end{equation}
The dependence of $P_{0,n}^\Lambda$ on $n$ arises because it is a
distribution on sub-partitions of the observed sample. Throughout this
section, $H$ is treated as a fixed probability distribution on
$\Lambda$.

\subsection{Level-marginal posterior consistency}

\begin{assumption}[Uniform density contraction]
\label{ass:density_contraction}
There exist sequences $\epsilon_n\downarrow0$ and $K_n\to\infty$ such
that
\[
K_n\epsilon_n\to0
\]
and
\begin{equation}
\Pi_n\!\left(
\left\{
f:\|f-f_0\|_\infty>K_n\epsilon_n
\right\}
\mid y
\right)
\xrightarrow{P}0.
\notag
\end{equation}
\end{assumption}

Assumption~\ref{ass:density_contraction} is analogous to the
sup-norm contraction condition used in Bayesian level-set clustering
\citep{buch2024bayesian}. It is deliberately stated independently of
the density-inference mechanism. For the recursive Gaussian-copula
martingale posterior used in Section~\ref{sec:mpd_intro}, existing
theory establishes predictive convergence and weaker forms of
consistency \citep{fong2023martingale}, but does not directly establish
the sup-norm contraction assumed here. The result below should
therefore be interpreted conditionally on
Assumption~\ref{ass:density_contraction} for that construction.

Let
\begin{equation}
\mathcal F_n
=
\left\{
f:\|f-f_0\|_\infty\leq K_n\epsilon_n
\right\}.
\notag
\end{equation}

\begin{assumption}[Uniform level-marginal stability]
\label{ass:level_marginal_stability}
Let $d_n:\mathcal S_n\times\mathcal S_n\to[0,1]$ be a bounded
non-negative loss. Assume that
\begin{equation}
\Delta_n
:=
\sup_{f\in\mathcal F_n}
\int_\Lambda
d_n\!\left(
\widetilde\psi_{n,\lambda}(f),
\psi^\circ_{n,\lambda}(f_0)
\right)
H(d\lambda)
\xrightarrow{P}0.
\notag
\end{equation}
\end{assumption}

Assumption~\ref{ass:level_marginal_stability} is the level-marginal
analogue of fixed-level stability: densities uniformly close to $f_0$
must induce sample-graph sub-partitions close to the population
level-set sub-partitions, on average over $\lambda\sim H$. It therefore
accounts jointly for perturbation of the density and approximation of
the population level-set components by the sample graph construction.

For probability measures $P$ and $Q$ on $\mathcal S_n$, define the
transport cost
\begin{equation}
\mathsf T_{d_n}(P,Q)
=
\inf_{\gamma\in\Gamma(P,Q)}
\int
d_n(\theta,\theta')\,
\gamma(d\theta,d\theta'),
\notag
\end{equation}
where $\Gamma(P,Q)$ denotes the set of couplings of $P$ and $Q$.

\begin{theorem}[Consistency of the level-marginal clustering posterior]
\label{thm:level_marginal_consistency}
Under Assumptions~\ref{ass:density_contraction}
and~\ref{ass:level_marginal_stability},
\begin{equation}
\mathsf T_{d_n}\!\left(
\Pi_n^\Lambda,
P_{0,n}^\Lambda
\right)
\xrightarrow{P}0.
\notag
\end{equation}
If $d_n$ is a metric, then $\mathsf T_{d_n}$ is the corresponding
$1$-Wasserstein distance.
\end{theorem}

\begin{proof}
Conditional on the observed data, draw
\[
f\sim\Pi_n(\cdot\mid y),
\qquad
\lambda\sim H
\]
independently, and define
\[
\Theta
=
\widetilde\psi_{n,\lambda}(f),
\qquad
\Theta_0
=
\psi^\circ_{n,\lambda}(f_0).
\]
By \eqref{eq:Pi-Lambda-theory} and \eqref{eq:P0-Lambda},
the marginals of $(\Theta,\Theta_0)$ are
$\Pi_n^\Lambda$ and $P_{0,n}^\Lambda$, respectively. Hence this joint
law defines a valid coupling, and
\begin{align}
\mathsf T_{d_n}\!\left(
\Pi_n^\Lambda,P_{0,n}^\Lambda
\right)
&\leq
\int
\int_\Lambda
d_n\!\left(
\widetilde\psi_{n,\lambda}(f),
\psi^\circ_{n,\lambda}(f_0)
\right)
H(d\lambda)\,
\Pi_n(df\mid y).
\notag
\end{align}
Define
\[
r_n(f)
=
\int_\Lambda
d_n\!\left(
\widetilde\psi_{n,\lambda}(f),
\psi^\circ_{n,\lambda}(f_0)
\right)
H(d\lambda).
\]
Splitting the posterior integral over $\mathcal F_n$ and
$\mathcal F_n^c$, and using $0\leq d_n\leq1$, gives
\begin{align}
\mathsf T_{d_n}\!\left(
\Pi_n^\Lambda,P_{0,n}^\Lambda
\right)
&\leq
\int_{\mathcal F_n}
r_n(f)\,\Pi_n(df\mid y)
+
\int_{\mathcal F_n^c}
r_n(f)\,\Pi_n(df\mid y)
\notag\\
&\leq
\sup_{f\in\mathcal F_n}r_n(f)
+
\Pi_n(\mathcal F_n^c\mid y)
\notag\\
&=
\Delta_n
+
\Pi_n(\mathcal F_n^c\mid y).
\notag
\end{align}
The first term converges to zero in probability by
Assumption~\ref{ass:level_marginal_stability}, and the second by
Assumption~\ref{ass:density_contraction}. Hence
\[
\mathsf T_{d_n}\!\left(
\Pi_n^\Lambda,P_{0,n}^\Lambda
\right)
\xrightarrow{P}0.
\]
\end{proof}

For the symmetric VI$_{\mathrm{IA}}$ loss with $\tau>0$, one may take
\begin{equation}
d_n^{\mathrm{VI}_{\mathrm{IA}}}(\theta,\theta')
=
\frac{
\mathrm{VI}_{\mathrm{IA}}(\theta,\theta')
}{
(\tau+1)\log_2 n
},
\qquad n\geq2,
\label{eq:normalised_iavi}
\end{equation}
which lies in $[0,1]$ by Proposition~\ref{prop:ia_vi_lin_basic}. Alternatively, the tighter upper bound in \eqref{eq:iavi_upper_bound} may be used for normalisation.
Theorem~\ref{thm:level_marginal_consistency} therefore gives loss-based transport convergence under Assumption~\ref{ass:level_marginal_stability} with this choice of $d_n$. Since VI$_{\mathrm{IA}}$ is not assumed to satisfy the triangle inequality, we do not interpret this transport cost as a Wasserstein distance. For a normalised BI$_{\mathrm{IA}}$ loss satisfying the corresponding metric conditions, the theorem instead gives ordinary Wasserstein convergence.

\subsection{Oracle bound for metric ground losses}

For the empirical approximation, let
$\widehat\pi_{M,G}^{\Lambda}$ denote the weighted empirical
level-marginal posterior defined in \eqref{eq:hat-Pi-Lambda}, based on $M$
density draws and a $G$-point level grid. Let $d_n$ be a bounded metric on
$\mathcal S_n$, and let
\begin{equation}
\widehat q_L^\Lambda
\in
\operatorname*{arg\,min}_{q\in\mathcal Q_L}
W_{d_n}\!\left(
\widehat\pi_{M,G}^{\Lambda},q
\right)
\notag
\end{equation}
be the empirical $L$-particle WASABI-SuP summary. Define
\begin{equation}
q_{0,n,L}^\star
\in
\operatorname*{arg\,min}_{q\in\mathcal Q_L}
W_{d_n}\!\left(
P_{0,n}^\Lambda,q
\right)
\notag
\end{equation}
to be the best $L$-particle approximation to the true-density
level-marginal target.

\begin{theorem}[Oracle bound for empirical level-marginal WASABI-SuP]
\label{thm:level_marginal_oracle_bound}
For any bounded metric $d_n$,
\begin{align}
&
W_{d_n}\!\left(
P_{0,n}^\Lambda,
\widehat q_L^\Lambda
\right)
-
W_{d_n}\!\left(
P_{0,n}^\Lambda,
q_{0,n,L}^\star
\right)
\leq
2
W_{d_n}\!\left(
\widehat\pi_{M,G}^{\Lambda},
P_{0,n}^\Lambda
\right).
\label{eq:level_marginal_oracle_bound}
\end{align}
\end{theorem}

\begin{proof}
By the triangle inequality,
\begin{align*}
W_{d_n}\!\left(
P_{0,n}^\Lambda,\widehat q_L^\Lambda
\right)
&\leq
W_{d_n}\!\left(
P_{0,n}^\Lambda,\widehat\pi_{M,G}^{\Lambda}
\right)
+
W_{d_n}\!\left(
\widehat\pi_{M,G}^{\Lambda},\widehat q_L^\Lambda
\right).
\end{align*}
By optimality of $\widehat q_L^\Lambda$,
\[
W_{d_n}\!\left(
\widehat\pi_{M,G}^{\Lambda},\widehat q_L^\Lambda
\right)
\leq
W_{d_n}\!\left(
\widehat\pi_{M,G}^{\Lambda},q_{0,n,L}^\star
\right).
\]
Applying the triangle inequality once more,
\[
W_{d_n}\!\left(
\widehat\pi_{M,G}^{\Lambda},q_{0,n,L}^\star
\right)
\leq
W_{d_n}\!\left(
\widehat\pi_{M,G}^{\Lambda},P_{0,n}^\Lambda
\right)
+
W_{d_n}\!\left(
P_{0,n}^\Lambda,q_{0,n,L}^\star
\right).
\]
Combining these inequalities and using symmetry of $W_{d_n}$ gives
\begin{align*}
W_{d_n}\!\left(
P_{0,n}^\Lambda,\widehat q_L^\Lambda
\right)
&\leq
2W_{d_n}\!\left(
\widehat\pi_{M,G}^{\Lambda},P_{0,n}^\Lambda
\right)
+
W_{d_n}\!\left(
P_{0,n}^\Lambda,q_{0,n,L}^\star
\right).
\end{align*}
Rearranging proves \eqref{eq:level_marginal_oracle_bound}.
\end{proof}

The approximation term on the right-hand side satisfies
\begin{align}
W_{d_n}\!\left(
\widehat\pi_{M,G}^{\Lambda},
P_{0,n}^\Lambda
\right)
&\leq
W_{d_n}\!\left(
\widehat\pi_{M,G}^{\Lambda},
\Pi_n^\Lambda
\right)
+
W_{d_n}\!\left(
\Pi_n^\Lambda,
P_{0,n}^\Lambda
\right).
\notag
\end{align}
The first term is the Monte Carlo and grid-approximation error in
approximating the level-marginal posterior, whereas the second is the
statistical error controlled by
Theorem~\ref{thm:level_marginal_consistency}. Consequently, under any
joint asymptotic regime for which
\[
W_{d_n}\!\left(
\widehat\pi_{M,G}^{\Lambda},\Pi_n^\Lambda
\right)
\xrightarrow{P}0
\]
and
\[
W_{d_n}\!\left(
\Pi_n^\Lambda,P_{0,n}^\Lambda
\right)
\xrightarrow{P}0,
\]
the excess approximation error satisfies
\[
W_{d_n}\!\left(
P_{0,n}^\Lambda,\widehat q_L^\Lambda
\right)
-
W_{d_n}\!\left(
P_{0,n}^\Lambda,q_{0,n,L}^\star
\right)
\xrightarrow{P}0.
\]

\subsection{Oracle bound for VI$_{\mathrm{IA}}$ loss compression}
\label{app:ia_vi_oracle_bound}

Because VI$_{\mathrm{IA}}$ is not assumed to be a metric, the preceding Wasserstein
argument does not apply directly. Instead, let
\[
\ell_n(\theta,\theta')
=
d_n^{\mathrm{VI}_{\mathrm{IA}}}(\theta,\theta')
\in[0,1]
\]
denote the normalised VI$_{\mathrm{IA}}$ loss in
\eqref{eq:normalised_iavi}. Since this differs from the original VI$_{\mathrm{IA}}$
loss only by a positive factor depending on $n$, it has the same
minimisers.

For a set of $L$ representative sub-partitions
$\Theta=(\theta_1,\ldots,\theta_L)\in\mathcal S_n^L$,
define the VI$_{\mathrm{IA}}$ compression risk
\begin{equation}
\mathcal R_n(P,\Theta)
=
\int
\min_{1\leq\ell\leq L}
\ell_n(\theta,\theta_\ell)
\,P(d\theta).
\label{eq:ia_vi_compression_risk}
\end{equation}
Because the particle weights are free, minimising the transport
compression objective over distributions supported on at most $L$
representatives is equivalent to minimising
\eqref{eq:ia_vi_compression_risk}; the optimal particle weights are
the probabilities assigned to the corresponding minimum-loss regions.

Let
\begin{align}
\widehat\Theta_L^\Lambda
&\in
\operatorname*{arg\,min}_{\Theta\in\mathcal S_n^L}
\mathcal R_n\!\left(
\widehat\pi_{M,G}^{\Lambda},\Theta
\right), \qquad
\Theta_{0,n,L}^\star
\in
\operatorname*{arg\,min}_{\Theta\in\mathcal S_n^L}
\mathcal R_n\!\left(
P_{0,n}^\Lambda,\Theta
\right).
\notag
\end{align}

\begin{theorem}[VI$_{\mathrm{IA}}$ compression oracle bound]
\label{thm:ia_vi_oracle_bound}
The empirical and population minimisers satisfy
\begin{align}
&
\mathcal R_n\!\left(
P_{0,n}^\Lambda,\widehat\Theta_L^\Lambda
\right)
-
\mathcal R_n\!\left(
P_{0,n}^\Lambda,\Theta_{0,n,L}^\star
\right)
\leq
2
\sup_{\Theta\in\mathcal S_n^L}
\left|
\mathcal R_n\!\left(
\widehat\pi_{M,G}^{\Lambda},\Theta
\right)
-
\mathcal R_n\!\left(
P_{0,n}^\Lambda,\Theta
\right)
\right|.
\label{eq:ia_vi_oracle_bound}
\end{align}
Consequently, under any joint asymptotic regime with $n\to\infty$,
$M\to\infty$, and, when $H$ is approximated by an increasingly fine grid,
$G\to\infty$, with $L$ fixed, if
\begin{equation}
\sup_{\Theta\in\mathcal S_n^L}
\left|
\mathcal R_n\!\left(
\widehat\pi_{M,G}^{\Lambda},\Theta
\right)
-
\mathcal R_n\!\left(
P_{0,n}^\Lambda,\Theta
\right)
\right|
\xrightarrow{P}0,
\label{eq:ia_vi_uniform_convergence}
\end{equation}
the empirical VI$_{\mathrm{IA}}$ WASABI-SuP summary is asymptotically as good,
in VI$_{\mathrm{IA}}$ compression risk, as the best $L$-particle approximation of
$P_{0,n}^\Lambda$.
\end{theorem}

\begin{proof}
Write
$\widehat P
=
\widehat\pi_{M,G}^{\Lambda}$,
$P
=
P_{0,n}^\Lambda$,
$\widehat\Theta
=
\widehat\Theta_L^\Lambda$,
$\Theta^\star
=
\Theta_{0,n,L}^\star$.
Adding and subtracting the empirical risks gives
\begin{align*}
\mathcal R_n(P,\widehat\Theta)
-
\mathcal R_n(P,\Theta^\star)
={}&
\left\{
\mathcal R_n(P,\widehat\Theta)
-
\mathcal R_n(\widehat P,\widehat\Theta)
\right\}
+
\left\{
\mathcal R_n(\widehat P,\widehat\Theta)
-
\mathcal R_n(\widehat P,\Theta^\star)
\right\}
+
\left\{
\mathcal R_n(\widehat P,\Theta^\star)
-
\mathcal R_n(P,\Theta^\star)
\right\}.
\end{align*}
The middle term is non-positive by optimality of
$\widehat\Theta$. Therefore,
\begin{align*}
\mathcal R_n(P,\widehat\Theta)
-
\mathcal R_n(P,\Theta^\star)
&\leq
\left|
\mathcal R_n(P,\widehat\Theta)
-
\mathcal R_n(\widehat P,\widehat\Theta)
\right|
+
\left|
\mathcal R_n(\widehat P,\Theta^\star)
-
\mathcal R_n(P,\Theta^\star)
\right|
\\
&\leq
2
\sup_{\Theta\in\mathcal S_n^L}
\left|
\mathcal R_n(\widehat P,\Theta)
-
\mathcal R_n(P,\Theta)
\right|.
\end{align*}
This proves \eqref{eq:ia_vi_oracle_bound}. The final statement follows
immediately from \eqref{eq:ia_vi_uniform_convergence}.
\end{proof}

\begin{remark}[Relation to fixed $\lambda$]
If $H=\delta_{\lambda_0}$, then
\[
\Pi_n^\Lambda
=
\Pi_{n,\lambda_0},
\qquad
P_{0,n}^\Lambda
=
\delta_{\psi^\circ_{n,\lambda_0}(f_0)}.
\]
Theorem~\ref{thm:level_marginal_consistency} therefore reduces to the
corresponding fixed-level consistency result under sup-norm density
contraction and fixed-level stability. Thus, the level-marginal
formulation contains the fixed-$\lambda$ setting as a special case.
\end{remark}

\begin{remark}[Interpretation of $P_{0,n}^\Lambda$]
The target
\[
P_{0,n}^\Lambda
=
\int_\Lambda
\delta_{\psi^\circ_{n,\lambda}(f_0)}
\,H(d\lambda)
\]
is the distribution of the true-density level-set sub-partition of
the observed data when $\lambda\sim H$. It is therefore a probability
distribution over flat sub-partitions, rather than a cluster tree.

Accordingly, a level-marginal WASABI-SuP particle represents a typical
sub-partition arising from a collection of density--level
combinations, and its weight represents the posterior mass assigned
to the corresponding region. A particle need not correspond to the
clustering at any single value of $\lambda$.
\end{remark}

Note that these guarantees are conditional on the stated density-contraction and
stability assumptions and treat the level distribution $H$ as fixed. In particular, the uniform sup-norm contraction condition imposed is an assumption of the present theory; existing predictive-convergence results for the Gaussian-copula recursive predictive used in our experiments may not establish this condition. The data-adaptive calibration of $H$ used in our experiments is therefore a
practical extension not covered directly by the present asymptotic result.

\section{Additional details for the synthetic study}
\label{app:sim2d}

This section provides the data-generating settings, implementation and evaluation details, and additional results for the experiments in Section~\ref{sec:sim2d}. All experiments were run locally on an Apple M4 Pro machine with
24\,GB unified memory.

\subsection{Data-generating settings}
\label{app:sim2d-scenes}

\paragraph{Ring.}
The first dataset in Section~\ref{sec:sim2d} has $n=250$ observations with $90$ clutter observations.
It contains a single $K=1$ annular cluster with observations having angle uniformly distributed on $[0,2\pi)$ and radius $r\sim\mathcal N(R,\sigma_r^2)$, $R=6.6$, $\sigma_r=0.44$.
An independent homogeneous clutter with rate $100$ is added over area $[-10,10]^2$. 

\paragraph{Four rectangles.}
The second example contains $K=4$ clusters and a homogeneous clutter component with rate $80$. It has $n=480$ observations with $86$ clutter observations. The cluster component intensities range from $3$ to $15$ with different size, aspect ratio, and orientation. Each signal component is uniform on its support.

\begin{figure*}[t]
    \centering
\includegraphics[width=\textwidth]{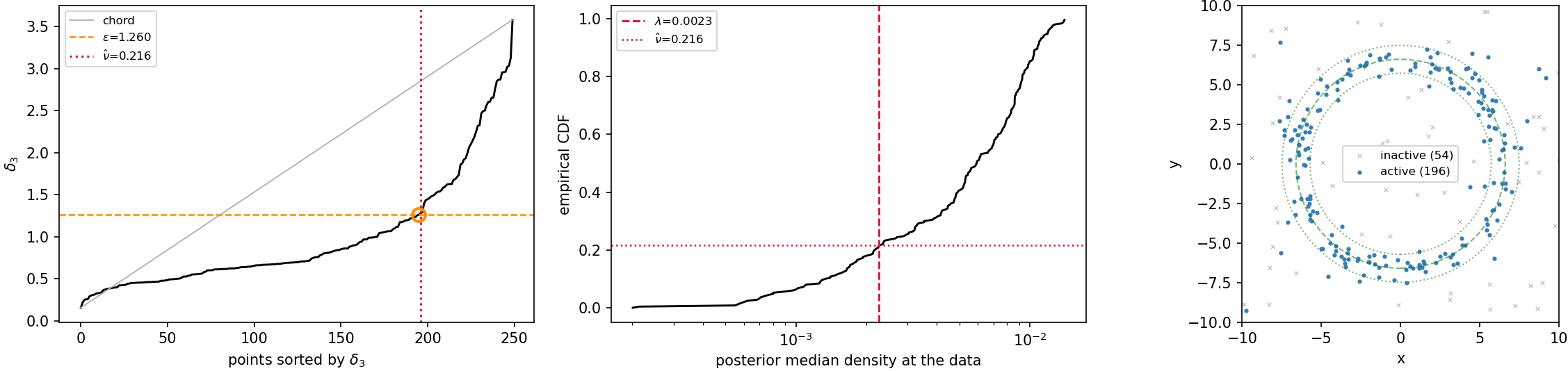}
\includegraphics[width=\textwidth]{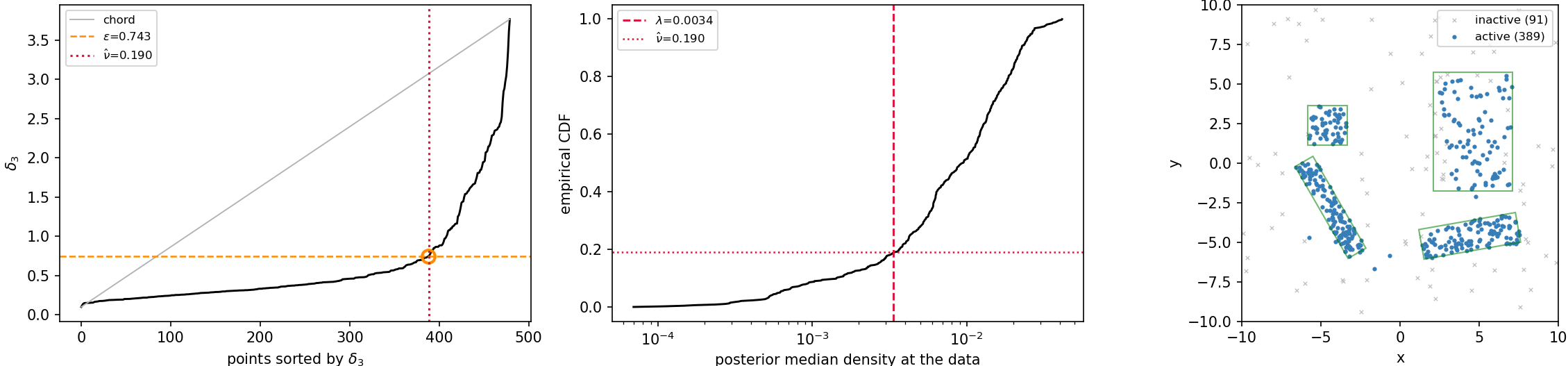}
    \caption{
    Data-driven calibration of the density level.
    \textbf{Top:} ring example.
    \textbf{Bottom:} four-rectangle example.
    In each case, the left panel shows the sorted $3$-nearest-neighbour distances, with 
    the orange markers denoting the detected knee $\varepsilon$, defined as the point of maximum perpendicular distance from the chord (grey lines). 
    The middle panel transfers the same empirical rank $\hat\nu$ to the posterior median MPD density, giving $\lambda=q_{\hat\nu}\{\hat f(x_i)\}$. The right panel shows the resulting reference active and inactive sets under $\lambda$. 
    Green dotted lines indicate the true component boundaries for visual reference and are not used in the   calibration.
    }
    \label{fig:level_selection}
\end{figure*}

\subsection{Baseline methods}
\label{app:sim2d-baselines}
\paragraph{$k$-means.}
We use Lloyd's algorithm implemented in \texttt{scikit-learn}, with $10$ random restarts. The number of clusters is selected using the gap statistic over $K\in\{1,\ldots,12\}$.
For each candidate $K$, $B^*=50$ reference datasets are drawn uniformly over the bounding box of the observations. We select the smallest $K$
satisfying
$\operatorname{Gap}(K)
\geq
\operatorname{Gap}(K+1)-s_{K+1}$,
where
$s_K
=
\operatorname{sd}^{*}(K)
\sqrt{1+\frac{1}{B^*}}$.
Since $k$-means returns a full partition, it has no inactive label and assigns every clutter observation
to a cluster.

\paragraph{DBSCAN.}
We use the \texttt{scikit-learn} implementation. The minimum number of
neighbours is set to $\texttt{min\_samples}=2d=4$,
where $d=2$. The neighbourhood radius $\varepsilon$ is selected from the sorted four-nearest-neighbour distance curve, with the knee defined as the
point having maximum perpendicular distance from the chord connecting the first and last points of the curve. 

\paragraph{BNPmix.}
We fit a Pitman--Yor location-scale mixture of bivariate Gaussian distributions using \nolinkurl{BNPmix::PYdensity}. The sampler is run for
$6000$ MCMC iterations, with the first $2000$ discarded as burn-in. The clustering reported in the main text is the package's VI-based point estimate obtained using \nolinkurl{BNPmix::PYdensity} with $\texttt{max\_k}=25$.
The routine considers complete-linkage and average-linkage cuts of the posterior dissimilarity matrix together with the best-scoring sampled partition, and selects the candidate with the smallest posterior expected VI. For both examples, the selected solution is the average-linkage candidate. As with $k$-means, BNPmix produces a full partition and therefore
assigns every background observation to a mixture component.

\paragraph{Proposed method}\label{app:sim2d-proposed}

Posterior density draws are generated using the martingale-posterior
predictive-resampling construction in Section~\ref{sec:mpd_intro}, with the
recursive Gaussian-copula predictive update. The bandwidth parameter is fixed at $\rho=(0.95,0.95)$.
Each predictive-resampling path is continued for $500$ forward samples, and the procedure is repeated independently to obtain $M=100$ posterior density
draws evaluated at the observed locations.

\subsection{Data-driven calibration of the density level}
\label{app:level-calibration}

The level-set map requires a density threshold $\lambda$ determining which observations are active. A direct specification of $\lambda$ through an assumed noise proportion would require prior knowledge of the true clutter fraction that is generally unavailable. Hence, here we calibrate $\lambda$ directly from the observed data by combining the standard
DBSCAN $k$-distance heuristic \citep{schubert2017dbscan} with the quantile parametrisation of density levels used by \citet{buch2024bayesian}. Importantly, DBSCAN is used only to determine a data-adaptive density rank; its distance threshold is not itself used as the MPD density level. Figure~\ref{fig:level_selection}
illustrates the procedure for both examples.

\paragraph{$k$-distance calibration.}
For each observation $x_i$, let $\delta_k(x_i)$ denote its distance to the
$k$th nearest neighbour. Following the usual DBSCAN heuristic, we set $\mathrm{MinPts}=2d$, which gives $\mathrm{MinPts}=4$ in the present two-dimensional examples, and this corresponds to the third-nearest-neighbour distance, $\delta_3(x_i)$ since the query point itself counts towards
$\mathrm{MinPts}$. As shown in the
left panels of Figure~\ref{fig:level_selection}, we  compute $\delta_3(x_i)$ for every observation and sort the values,
  $  \delta_{3,(1)}
    \leq
    \delta_{3,(2)}
    \leq
    \cdots
    \leq
    \delta_{3,(n)}$.
The knee of this curve is detected as the point of maximum perpendicular
distance from the chord joining its endpoints. If the knee occurs at rank $j^\star$, we set
$\varepsilon=\delta_{3,(j^\star)}$
and define the corresponding empirical low-density proportion as
\begin{equation}
    \hat\nu
    =
    \frac{1}{n}
    \sum_{i=1}^{n}
    \mathbbm{1}
    \left\{
        \delta_3(x_i)>\varepsilon
    \right\}.
    \label{eq:knee_nu}
\end{equation}
Thus, $\hat\nu$ is the empirical proportion of observations lying beyond the $k$-distance knee.

Motivated by the DBSCAN--level-set correspondence of
\citet{buch2024bayesian}, we transfer this density rank to the martingale-posterior density, as shown in
the middle panels of Figure~\ref{fig:level_selection}. Let
\[\hat f(x_i)
=
\operatorname{median}_{m=1,\ldots,M}
\left\{
    f^{(m)}(x_i)
\right\}\]
denote the posterior median density at observation $x_i$. The fixed density
level is then
\begin{equation}
    \lambda
    =
    q_{\hat\nu}
    \left(
        \{\hat f(x_i)\}_{i=1}^{n}
    \right),
    \label{eq:knee_lambda}
\end{equation}
where $q_{\hat\nu}$ denotes the empirical $\hat\nu$-quantile. Thus the density threshold has the same empirical rank $\hat\nu$ as that identified by the $k$-distance diagnostic. The right panels show the resulting reference active and inactive sets.

For the ring example, the knee occurs at $\varepsilon=1.2598$, giving
$\hat\nu=0.216$ and hence $\lambda=0.002267$. This yields $54$ inactive observations out of $250$, as shown in the top-right panel of
Figure~\ref{fig:level_selection}. Although the true clutter
fraction is $90/250=0.36$, it is not used in calibration; clutter points in
locally dense regions need not lie below the selected density level. For the
four-rectangle example, the corresponding values are $\varepsilon=0.7434$ and $\hat\nu=0.190$, yielding $91$ inactive
observations out of $480$ (bottom-right panel), compared with the generative
clutter fraction $86/480\approx0.179$, which is close to the estimated clutter fraction.

\subsection{Graph construction}

We use a common specification across the synthetic examples. Given the reference active set $A_\lambda(\hat f)$, let $d^{A}_{k}(x_i)$ denote the distance from active observation $x_i$ to its $k$-th nearest neighbour within that active set. We set
 $   \delta
    =
    q_{0.85}
    (
        \{
            d^{A}_{k}(x_i):
            i\in A_\lambda(\hat f)
        \}
    )$.
The same choice of neighbour order and quantile is used across datasets. We set $k=5$ in the  experiment, and as discussed in \cite{buch2024bayesian}, the resulting clustering is relatively insensitive
to this choice. 
The resulting $\lambda$ and $\delta$ are computed once from the reference posterior density and subsequently held fixed across posterior draws. For each draw $f^{(m)}$, the active set $A_\lambda(f^{(m)})=\{i:f^{(m)}(x_i)\geq\lambda\}$ is allowed to vary, and the connected components of the corresponding $\delta$-graph yield one posterior sub-partition draw.

\subsection{Normalised losses}\label{app:Normalised losses}
For VI$_{\mathrm{IA}}$, we set $\tau=1$ and normalise by the upper bound in \eqref{eq:bound deriv};
for BI$_{\mathrm{IA}}$, we set $a=b=1$ and
$m_{ai}=m_{ia}=0.5$ (This is the boundary choice satisfying the metric condition $a\leq 2m$ of \citet{buch2024bayesian}) and normalise by $\binom{n}{2}$, with details given in \cite{buch2024bayesian}.

Following \citet{buch2024bayesian}, we rescale BI$_{\mathrm{IA}}$ by its maximal pairwise loss. 
For $a=b$ and $m_{\mathrm{ai}}=m_{\mathrm{ia}}=m$, each unordered pair contributes at most
$\max(a,2m)$, giving
\[
\max_{\theta,\hat\theta} \mathrm{BI}_{\mathrm{IA}}(\theta,\hat\theta)
=
n(n-1)\max\!\left(m,\frac{a}{2}\right).
\]
Hence we use
\[
\widetilde{\mathrm{BI}}_{\mathrm{IA}}
=
\frac{\mathrm{BI}_{\mathrm{IA}}}
{n(n-1)\max(m,a/2)}.
\]
For $a=1$ and $m=0.5$, this reduces to
$\widetilde{\mathrm{BI}}_{\mathrm{IA}}=\mathrm{BI}_{\mathrm{IA}}/\binom{n}{2}$.

For VI$_{\mathrm{IA}}$, let $b=|A\cap\hat A|$. We normalise VI$_{\mathrm{IA}}$ with the upper bound in \eqref{eq:bound deriv}. Since
$\mathrm{VI}(\rho_B,\hat\rho_B)\le\log_2 b$, with both bounds jointly attainable,
\[
D_{n,\tau}
=
\max\!\left\{
\tau\log_2 n,\;
\max_{1\le b\le n}
\left[
\tau\left(1-\frac{b}{n}\right)\log_2 n+\log_2 b
\right]
\right\}.
\]
We therefore define
$\widetilde{\mathrm{VI}}_{\mathrm{IA}}
=
{\mathrm{VI}_{\mathrm{IA}}}/{D_{n,\tau}}$,
so that both normalised losses lie in $[0,1]$.

\begin{figure}[tp]
    \centering
    \includegraphics[width=\textwidth]{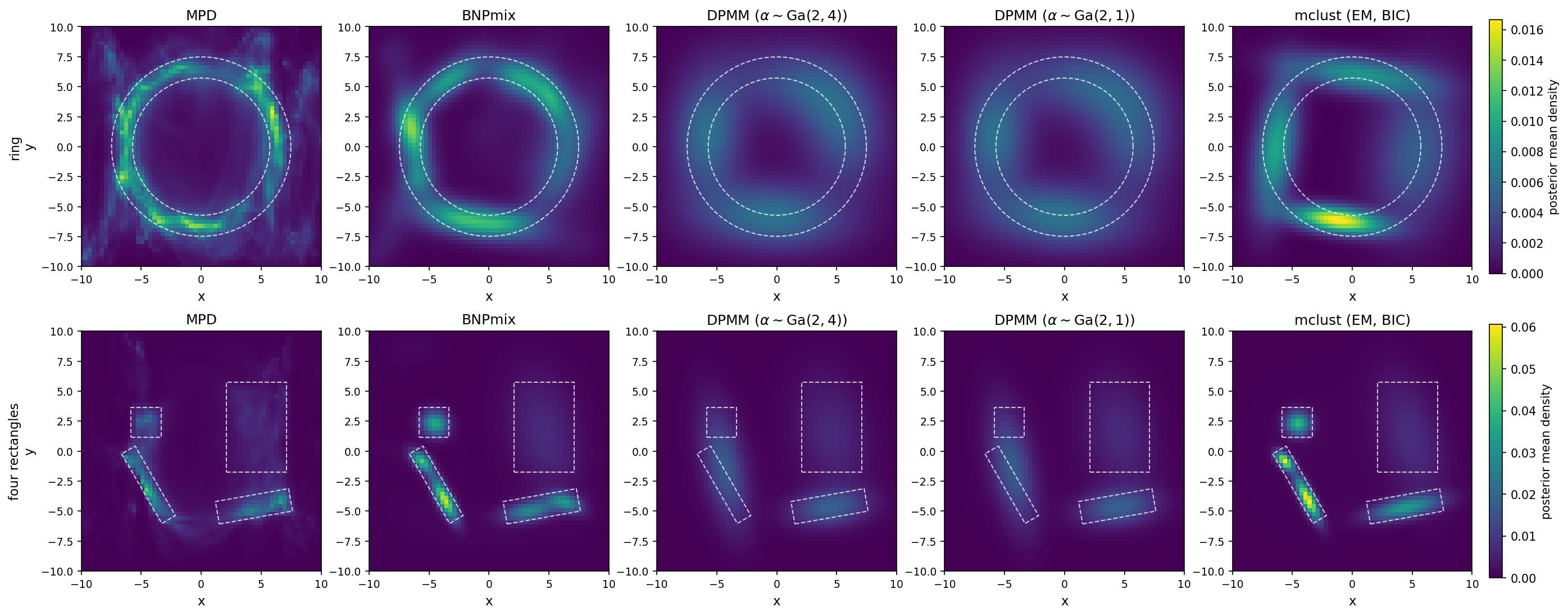}
    \caption{Density estimates for the ring (top) and four-rectangle (bottom)
    examples. Columns show the MPD posterior mean, the BNPmix posterior mean,
    Gaussian DPMM posterior means under concentration priors
    $\alpha\sim\mathrm{Ga}(2,4)$ and $\alpha\sim\mathrm{Ga}(2,1)$, and the
    \texttt{mclust} plug-in density estimate fitted by EM with BIC-based model
    selection. Dashed curves indicate the underlying high-density structures.
    Panels within each row share a common colour scale.}
    \label{fig:density_ballet_comparison}
\end{figure}

\subsection{Comparison with BALLET-style Bayesian density clustering baselines}
\label{app:ballet_comparison}

To provide a comparison of our predictive-resampling approach with the BALLET
framework of \citet{buch2024bayesian}, we compare the density estimates on the
\emph{ring} and \emph{four rectangles} examples. We report the density stage
separately to distinguish differences arising from density estimation from
those introduced by the subsequent level-set clustering and posterior
summarisation. BALLET is not tied to a particular Bayesian density model: it
takes posterior density draws as input and maps them to level-set clusterings.
Accordingly, we consider posterior densities obtained from both BNPmix and a
Gaussian Dirichlet process mixture model (DPMM), the latter being one of the
density models used in the BALLET experiments \citep{buch2024bayesian}. For the DPMM, we use the \texttt{dirichletprocess} package
\citep{ross2026dirichletprocess}, standardising the observations before fitting
\texttt{DirichletProcessMvnormal} with its default normal--inverse-Wishart base
measure. We additionally examine sensitivity to the concentration prior,
comparing the default $\alpha\sim\mathrm{Ga}(2,4)$ with $\alpha\sim\mathrm{Ga}(2,1)$. All density estimates are evaluated on the same
grid, with a common colour scale within each example.

Figure~\ref{fig:density_ballet_comparison} shows that MPD and BNPmix preserve
the main geometric features of both examples, whereas the Gaussian DPMM yields
smoother posterior mean densities, particularly for the annular and rectangular
structures. Increasing the prior mass on larger values of the concentration
parameter leads to more occupied mixture components, but has little visible
effect on the resulting density estimates in these examples.

For additional reference, we also include the plug-in density estimate from
\texttt{mclust} \citep{scrucca2023model}, fitted by EM with BIC-based model
selection. As shown in Figure \ref{fig:density_ballet_comparison}, the
\texttt{mclust} estimate also captures the principal high-density regions: for
the ring it recovers the low-density centre but represents the annulus through
a small number of broad, elongated Gaussian components, while for the
four-rectangle example it produces sharper local peaks with smooth Gaussian
contours around the rectangular regions.
The standard \texttt{mclust} fit provides a single plug-in density estimate
rather than posterior density draws. Repeated bootstrap refitting could, in
principle, generate an empirical distribution of fitted densities that could be
passed through the same level-set and summarisation pipeline, although this
would represent bootstrap rather than Bayesian posterior uncertainty and is not
pursued here. 

Finally, different Bayesian density models or more extensive
hyperparameter tuning may yield improved density estimates; see
\citet{buch2024bayesian} for a broader comparison of alternative Bayesian
density models.

\section{Results on robustness}
\label{app:robustness}
\begin{figure}[t]
    \centering
    \includegraphics[width=\linewidth]{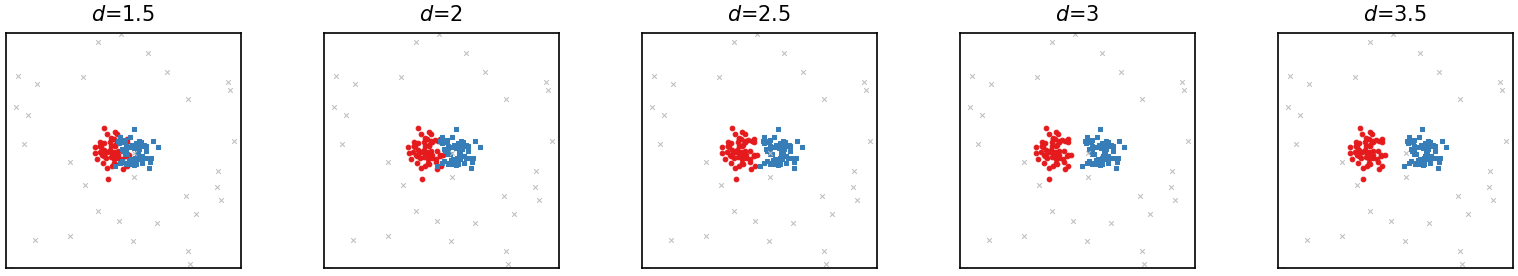}
    \caption{ Dataset and ground truth clustering in the separation experiment.
    Red and blue points denote the two Gaussian
    components and grey crosses denote background clutter. 
    }
    \label{fig:separation-data}
\end{figure}
\subsection{Additional settings and results for varying cluster separation}
\label{app:robustness-separation}

\paragraph{Settings}Figure~\ref{fig:separation} follows a single Monte Carlo realisation across increasing separation $d$, with the underlying dataset shown in Figure~\ref{fig:separation-data}. Common random numbers are used throughout: the component and clutter draws are fixed, and only the two Gaussian component centres move as $d$ increases across $d\in\{1.5,2,2.5,3,3.5\}$.  

The martingale posterior over densities uses copula predictive resampling with the bandwidth parameter fixed at $\rho = 0.95$ as suggested in \cite{hahn2018recursive}. Alternatively, $\rho$ can be optimised as discussed by \citet{fong2023martingale}. 
We use $M=100$ density draws, with each predictive-resampling path continued
for $500$ forward steps. The level is calibrated from the
$k$-distance knee, giving $\hat\nu$, and is then marginalised over
$\nu\sim\mathrm{Unif}[\hat\nu\pm0.10]$ using an equally weighted
$G=10$ point grid, with $\lambda(\nu)$ the $\nu$-quantile of the pointwise
median density. This yields $MG=1000$ level-marginal sub-partitions per dataset.
The connectivity radius $\delta$ is the $0.85$-quantile of the $5$-nearest-neighbour distance computed within the active set. 
WASABI-SuP is run for $L = 1,\dots,5$ using $3$ restarts of at most $10$ iterations. Reported losses are the empirical compression objective of~\eqref{eq:wasabi_sup_objective} under $D = \mathrm{VI}_{\mathrm{IA}}$ with equal penalties $m_{ai} = m_{ia} = 1$ with equal atom weights $1/(MG)$.

\begin{figure*}[tp]
    \centering    \includegraphics[width=\textwidth]{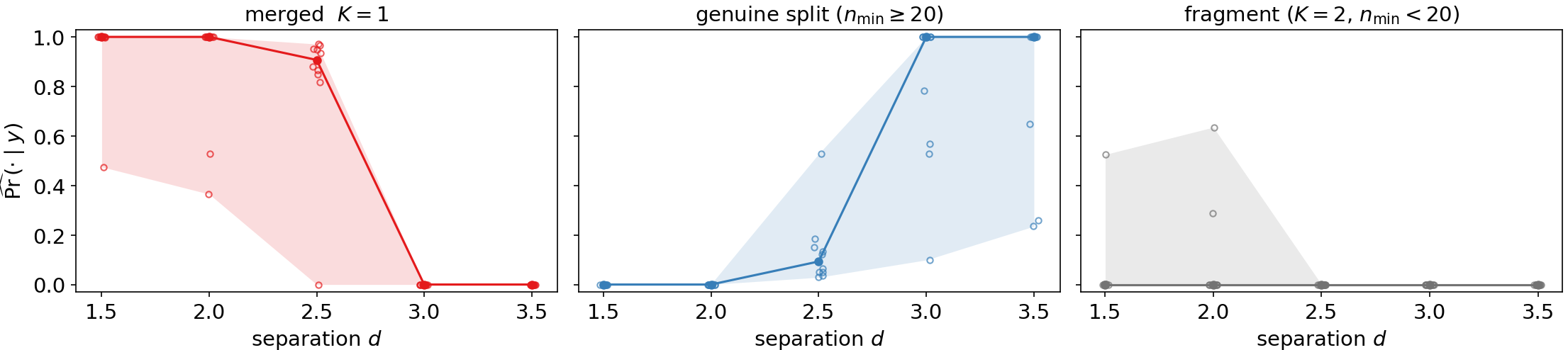}
    \caption{Posterior structural behaviour across additional Monte Carlo
    realisations.
    Posterior probabilities of a merged configuration ($K=1$), a genuine
    two-cluster split ($K=2$, $n_{\min}\geq20$), and a small-fragment
    configuration ($K=2$, $n_{\min}<20$) as the separation $d$ increases.
    Solid curves show the median across realisations, shaded regions show the
    full range, and circles show individual realisations.
    }
    \label{fig:separation-mc-state}
\end{figure*}

\paragraph{Behaviour analysis  across ten example datasets}
Here we further generate nine more example datasets under the same design as Figure~\ref{fig:separation-data}, sharing the same component and clutter
settings and differing only in the realised measurements. All settings of the proposed method are unchanged throughout.

To characterise the posterior structure, Figure~\ref{fig:separation-mc-state}
shows the empirical posterior probabilities $\widehat{\Pr}(\cdot\mid y)$ across
separations, which is computed as the fraction of the $MG=1000$ level-marginal sub-partitions satisfying each configuration: one merged cluster ($K=1$), a genuine two-cluster split ($K=2$ with smallest active-cluster size $n_{\min}\geq20$), and a two-cluster with a small fragment ($K=2$ with $n_{\min}<20$), where the total number of data points is $n = 180$. We separate the latter two to distinguish clear two-cluster structure from
small clutter fragments. We can see that at
$d=1.5$ and $d=2$, posterior mass is concentrated on the merged case, whereas at $d=3$ and $d=3.5$ it is concentrated on the two-cluster split. The intermediate value $d=2.5$ shows the greatest
uncertainty between these two clustering structures, consistent with the two particles displayed for the main-text realisation in Figure~\ref{fig:separation}(c)--(d).

Figure~\ref{fig:separation-mc-loss} decomposes the $L=1$ VI$_{\mathrm{IA}}$ loss into its active--inactive and clustering-structure components. The active--inactive term shows no pronounced dependence on $d$, indicating that the noise-related uncertainty is relatively robust to changes in cluster
separation. In contrast, the structural term is numerically zero away from $d=2.5$ and becomes consistently positive at $d=2.5$. Thus, the increase in total loss at the intermediate separation is driven by posterior ambiguity
between merged and split clustering structures, rather than by uncertainty in the active/inactive status.

\begin{figure*}[tp]
    \centering
    \includegraphics[width=\textwidth]{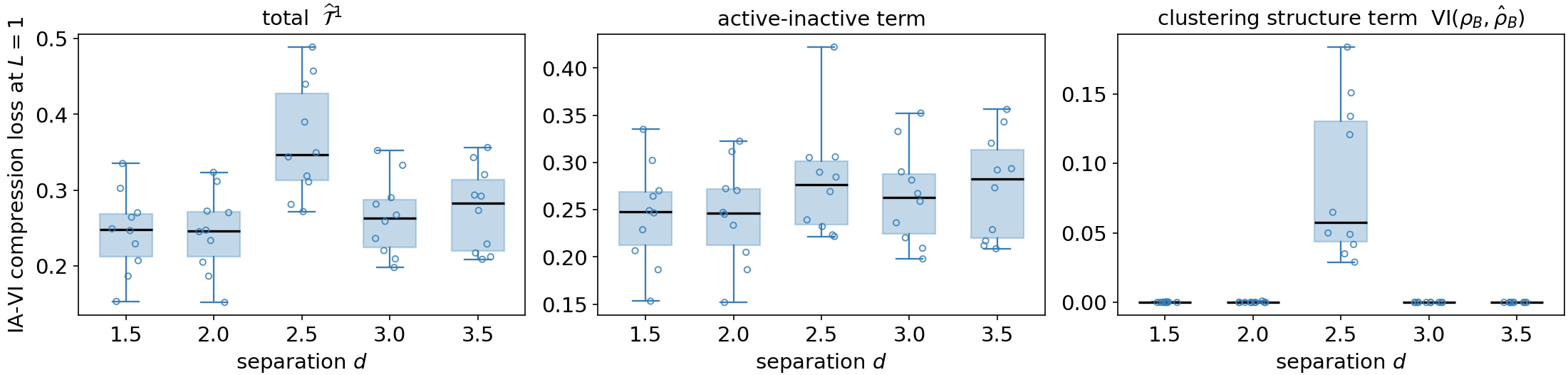}
    \caption{VI$_{\mathrm{IA}}$ uncertainty across additional Monte Carlo realisations.
    Total $L=1$ VI$_{\mathrm{IA}}$ compression loss (left), its active/inactive component
    (centre), and the restricted-VI structural component on the common active
    set (right), as a function of separation $d$.
    Boxes show the interquartile range, whiskers the full range, and circles
    the individual Monte Carlo realisations.
    }
    \label{fig:separation-mc-loss}
\end{figure*}

\begin{figure}[htp]
    \centering
    \includegraphics[width=0.9\linewidth]{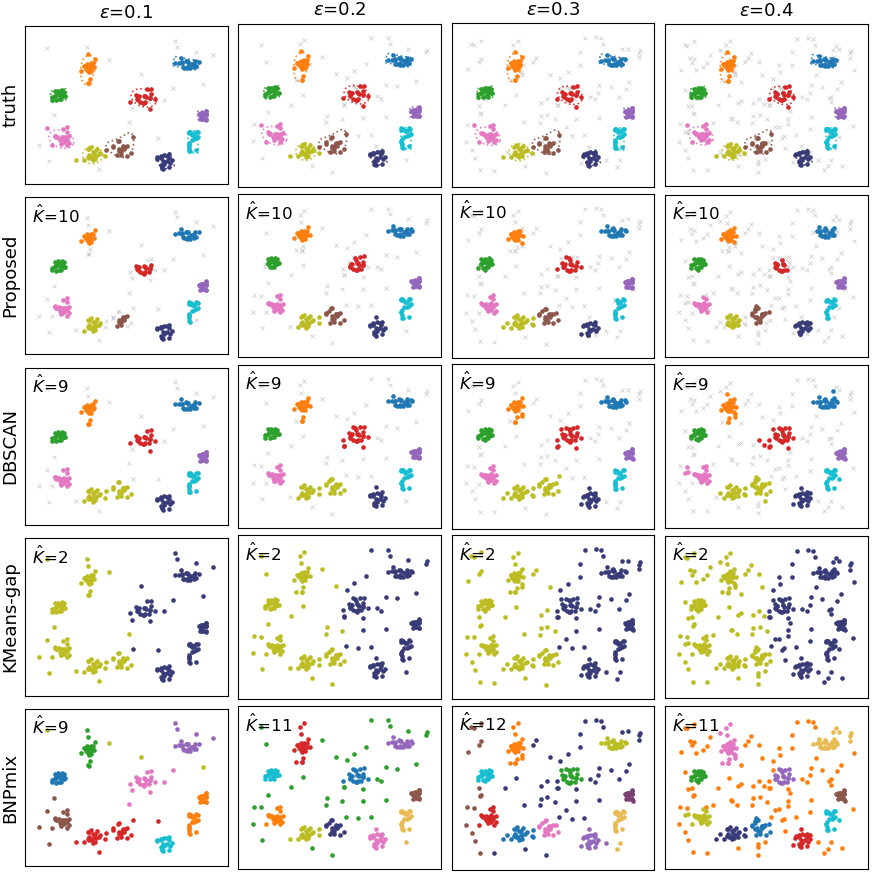}
    \caption{Clustering results under increasing clutter contamination,
    $\epsilon\in\{0.1,0.2,0.3,0.4\}$.
    Columns correspond to increasing clutter levels, while rows show the
    ground truth, proposed method, DBSCAN, gap-selected $k$-means, and BNPmix,
    respectively.
    The component observations are held fixed while clutter points are added
    sequentially across columns.
    Grey crosses denote inactive or unassigned observations.}
    \label{fig:noise_level_data}
\end{figure}

\subsection{Robustness to increasing clutter}\label{app:Robustness to increasing clutter}

We next examine robustness as the proportion of inactive observations increases.
As shown in Figure~\ref{fig:noise_level_data}, the component observations are held
fixed while clutter points are added sequentially, yielding contamination levels
$\epsilon\in\{0.1,0.2,0.3,0.4\}$.
We use the same settings as in Appendix~\ref{app:sim2d-baselines}.
For the proposed level-marginal method, we use $M=100$ density draws and
marginalise the level over
$\nu\sim\mathrm{Unif}[\hat\nu-0.10,\hat\nu+0.10]$
using an equally weighted $G=10$ point grid.
DBSCAN, gap-selected $k$-means, and BNPmix use the same configurations as in
Appendix~\ref{app:sim2d-baselines}.

Figure~\ref{fig:noise_level_data} shows the resulting point estimates.
The proposed method preserves the main clustering structure while progressively
classifying more observations as inactive as clutter increases.
DBSCAN also recovers most clusters, but consistently merges a pair of neighbouring
clusters.
The full-partition methods behave differently because they cannot leave
observations inactive.
Gap-selected $k$-means collapses the data into two broad groups throughout.
At $\epsilon=0.4$, BNPmix happens to form one estimated cluster containing a
large proportion of the clutter while leaving the main signal groups visually
separated; however, this behaviour is not consistent across contamination levels.
At lower contamination levels, clutter is instead largely absorbed into the
estimated signal clusters.

\section{Additional results for the real harbour data}
\label{app:harbour}

This section provides the additional WASABI-SuP summaries underlying the
real-data analysis in Section~\ref{sec:harbour}.  We first examine the choice
of the number of particles and then display the alternative VI$_{\mathrm{IA}}$
sub-partitions represented by the selected WASABI-SuP approximation.

\paragraph{Settings.}
We analyse frame~40 from the 141-frame harbour point-cloud dataset provided by \cite{gan2026pivot}. After preprocessing, the selected frame contains $n=1269$ observations. 
A full video illustrating tracking and dynamic clustering results obtained using variational inference is available at \url{https://runzegan.github.io/projects/pivot/}.  We use the level-marginal posterior with $\rho=0.999$. The calibration procedure in Appendix~\ref{app:level-calibration} yields a central level of $\lambda_0=1.2291\times10^{-6}$, corresponding to an inactive fraction of $\hat\nu=0.028$.
For this real-data example, we assess local sensitivity directly on the
density-level scale using three equally weighted levels centred on
$\lambda_0$, with offsets $\pm5\times10^{-7}$. These levels correspond,
under the reference density, to inactive fractions
$\nu_g\in\{0.009,\;0.028,\;0.043\}$.
Thus the grid varies the active/inactive decision over only a small proportion
of observations, unlike the synthetic experiments with substantial
homogeneous background clutter; consequently, the precise inactive fraction
is less consequential to the overall clustering structure here. We use
$10$ posterior density draws and $3$ levels, yielding $30$ level-marginal sub-partitions with equal weights.

\subsection{Expected VI$_{\mathrm{IA}}$ contribution}
\label{app:EVIC contribution}

We extend the expected VI contribution diagnostic of
\citet{balocchi2025understanding} to VI$_{\mathrm{IA}}$. Let
$B=A\cap\hat A$ be the common-active set of sub-partitions
$\theta$ and $\hat\theta$, and let
$\delta_i^{\mathrm{VI}}(\rho_B,\hat\rho_B)$ denote the standard pointwise VI
contribution on $B$, such that
$\mathrm{VI}(\rho_B,\hat\rho_B)
=
\sum_{i\in B}
\delta_i^{\mathrm{VI}}(\rho_B,\hat\rho_B)$.
The corresponding pointwise VI$_{\mathrm{IA}}$ contribution is
\begin{equation}
\delta_i^{\mathrm{IA\text{-}VI}}(\theta,\hat\theta)
=
\begin{cases}
\displaystyle \frac{\tau\log_2 n}{n},
& i\in (A\cap\hat I)\cup(I\cap\hat A),\\[2mm]
\delta_i^{\mathrm{VI}}(\rho_B,\hat\rho_B),
& i\in B,\\
0,
& i\in I\cap\hat I.
\end{cases}
\notag
\end{equation}
Thus
$\mathrm{VI}_{\mathrm{IA}}(\theta,\hat\theta)
=\sum_{i=1}^n
\delta_i^{\mathrm{IA\text{-}VI}}(\theta,\hat\theta)$.

For the $L=1$ VI$_{\mathrm{IA}}$ summary $\hat\theta_1$, the empirical expected
contribution of observation $i$ is
\begin{equation}
\widehat{\operatorname{EVIC}}_i(\hat\theta_1)
=
\sum_{m=1}^{M}
\eta_m\,
\delta_i^{\mathrm{IA\text{-}VI}}
\!\left(\theta^{(m)},\hat\theta_1\right),
\notag
\end{equation}
where $\eta_m=1/M$ for equally weighted posterior draws. By construction,
\[
\sum_{i=1}^{n}
\widehat{\operatorname{EVIC}}_i(\hat\theta_1)
=
\sum_{m=1}^{M}
\eta_m\,
\mathrm{VI}_{\mathrm{IA}}
\!\left(\theta^{(m)},\hat\theta_1\right).
\]

For the harbour analysis, the empirical level-marginal posterior consists of
$M$ density draws evaluated over a $G$-point level grid. Hence the EVIC
specialises to
\begin{equation}
\widehat{\operatorname{EVIC}}_i(\hat\theta_1)
=
\sum_{m=1}^{M}\sum_{g=1}^{G}
\eta_{m,g}\,
\delta_i^{\mathrm{IA\text{-}VI}}
\!\left(
\theta^{(m,g)},\hat\theta_1
\right),
\notag
\end{equation}
where $\eta_{m,g}=1/(MG)$ for the equally weighted grid used here.
Consequently,
\begin{equation}
\sum_{i=1}^{n}
\widehat{\operatorname{EVIC}}_i(\hat\theta_1)
=
\sum_{m=1}^{M}\sum_{g=1}^{G}
\eta_{m,g}\,
\mathrm{VI}_{\mathrm{IA}}
\!\left(
\theta^{(m,g)},\hat\theta_1
\right).
\notag
\end{equation}

\subsection{Selection of the number of WASABI-SuP particles}
\label{app:harbour-elbow}
Figure~\ref{fig:harbour-elbow} shows the approximation loss for
$L=1,\ldots,5$ under both BI$_{\mathrm{IA}}$ and VI$_{\mathrm{IA}}$. For VI$_{\mathrm{IA}}$, the elbow criterion selects $L=3$. 
Applying the same elbow rule under BI$_{\mathrm{IA}}$ selects $L=2$.  Thus, the exact number of representative particles depends on the geometry induced by the sub-partition loss, but both criteria favour representing the posterior with
more than one sub-partition.

\begin{figure*}[t]
    \centering
\includegraphics[width=0.9\textwidth]{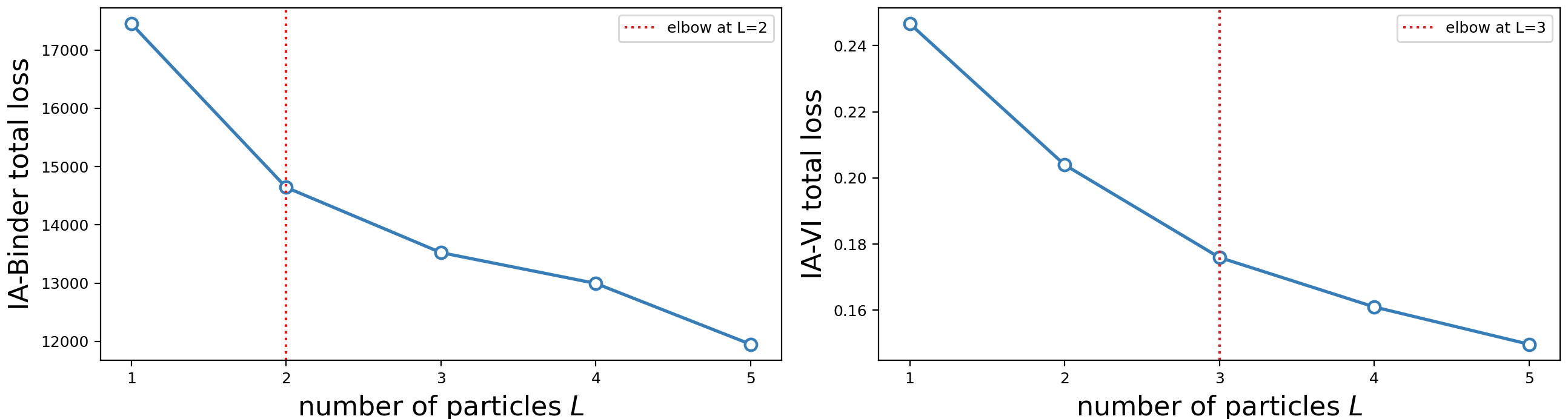}
    \caption{
    Selection of the number of WASABI-SuP particles for the harbour
    data.
    Total approximation loss against the number of particles $L$ for
    BI$_{\mathrm{IA}}$ (left) and VI$_{\mathrm{IA}}$ (right).  The dotted vertical lines indicate
    the values selected by the elbow criterion, $L=2$ for BI$_{\mathrm{IA}}$ and
    $L=3$ for VI$_{\mathrm{IA}}$.  Absolute loss magnitudes are not comparable across the
    two loss functions.
    }
    \label{fig:harbour-elbow}
\end{figure*}

\subsection{VI$_{\mathrm{IA}}$ particle summaries}
\label{app:harbour-particles}
The selected $L=3$ approximation is shown in
Figure~\ref{fig:harbour-l3}.  The three particles have weights
$(0.400,\;0.333,\;0.267)$, and contain
$25$, $24$, and $26$ active clusters, respectively.  The weights are
relatively balanced, indicating that the posterior variation is not captured by one dominant clustering structure.

Several spatially separated groups are preserved across the three particles, while differences are concentrated in a smaller number of
closely spaced local clusters.  These differences include both changes in cluster allocation and changes in whether individual observations are treated as
active or inactive.  These multimodal representations therefore complement the $L=1$ IA-EVIC map in
Figure~\ref{fig:harbour-main}: the IA-EVIC map identifies the observations responsible for uncertainty around the single point estimate, whereas the
$L=3$ particles explicitly display alternative sub-partition structures
that account for substantial posterior mass.

\begin{figure*}[t]
    \centering
    \includegraphics[width=\textwidth]{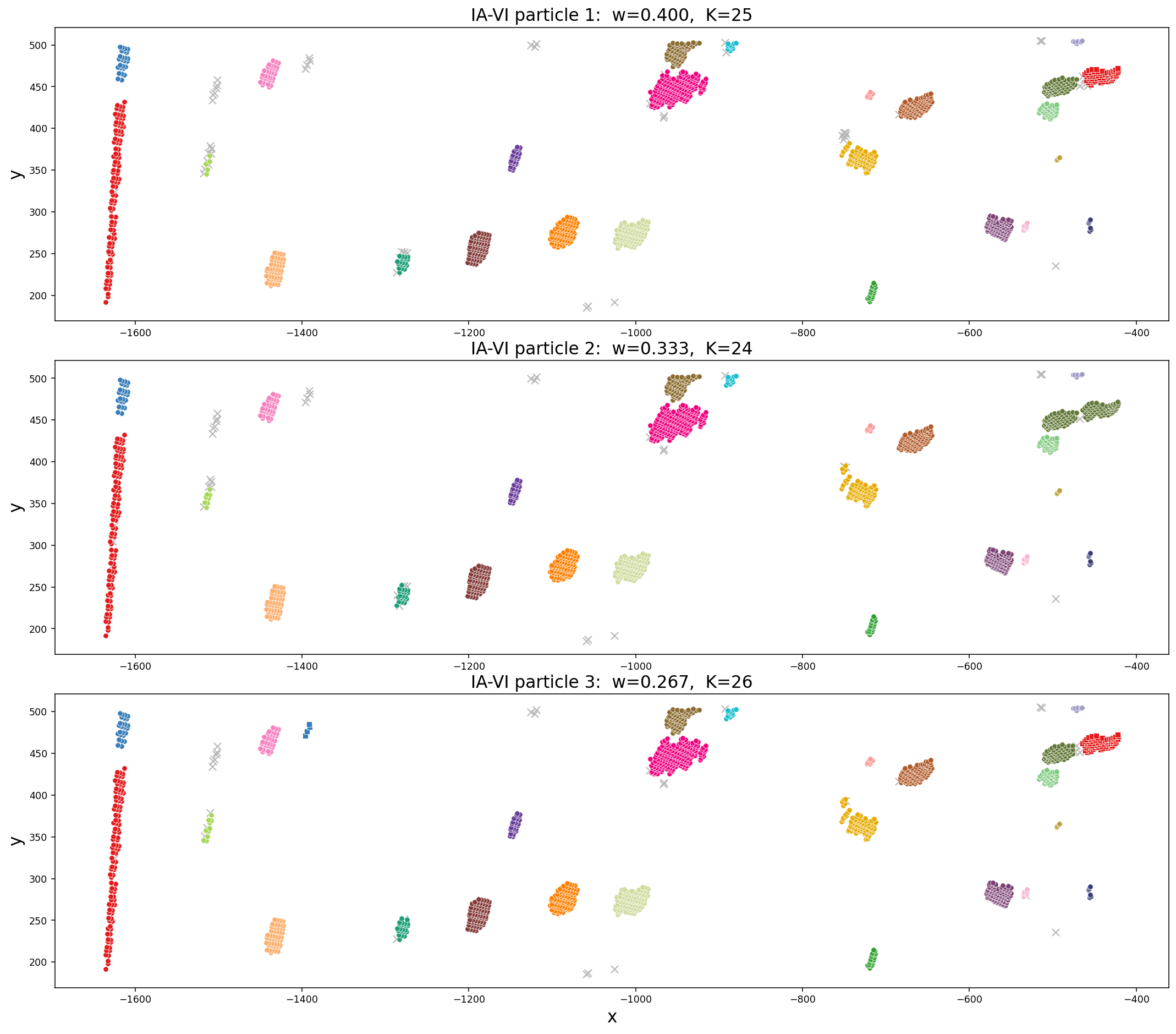}
    \caption{
   Three-particle VI$_{\mathrm{IA}}$ WASABI-SuP summary for the harbour
    data.
    }
    \label{fig:harbour-l3}
\end{figure*}

\subsection{Comparison with BI$_{\mathrm{IA}}$}
\label{app:harbour-binder}
As a sensitivity analysis, we also construct the WASABI-SuP approximation
using BI$_{\mathrm{IA}}$.  Figure~\ref{fig:harbour-elbow} shows that the corresponding
elbow occurs at $L=2$, compared with $L=3$ under VI$_{\mathrm{IA}}$.   This difference is expected because VI$_{\mathrm{IA}}$ and BI$_{\mathrm{IA}}$ induce different geometries on the space of sub-partitions.  Importantly, both losses indicate that the $L=1$ representation discards non-negligible posterior variation,
supporting the use of multiple representative sub-partitions for this real scene. Figure~\ref{fig:harbour-l2} shows the $L=2$ BI$_{\mathrm{IA}}$ approximation.
The two particles have weights
$(0.567,\;0.433)$; both particles contain $25$ active clusters.  The particles represent alternative cluster allocations and
active/inactive status.

\begin{figure*}[t]
    \centering
\includegraphics[width=\textwidth]{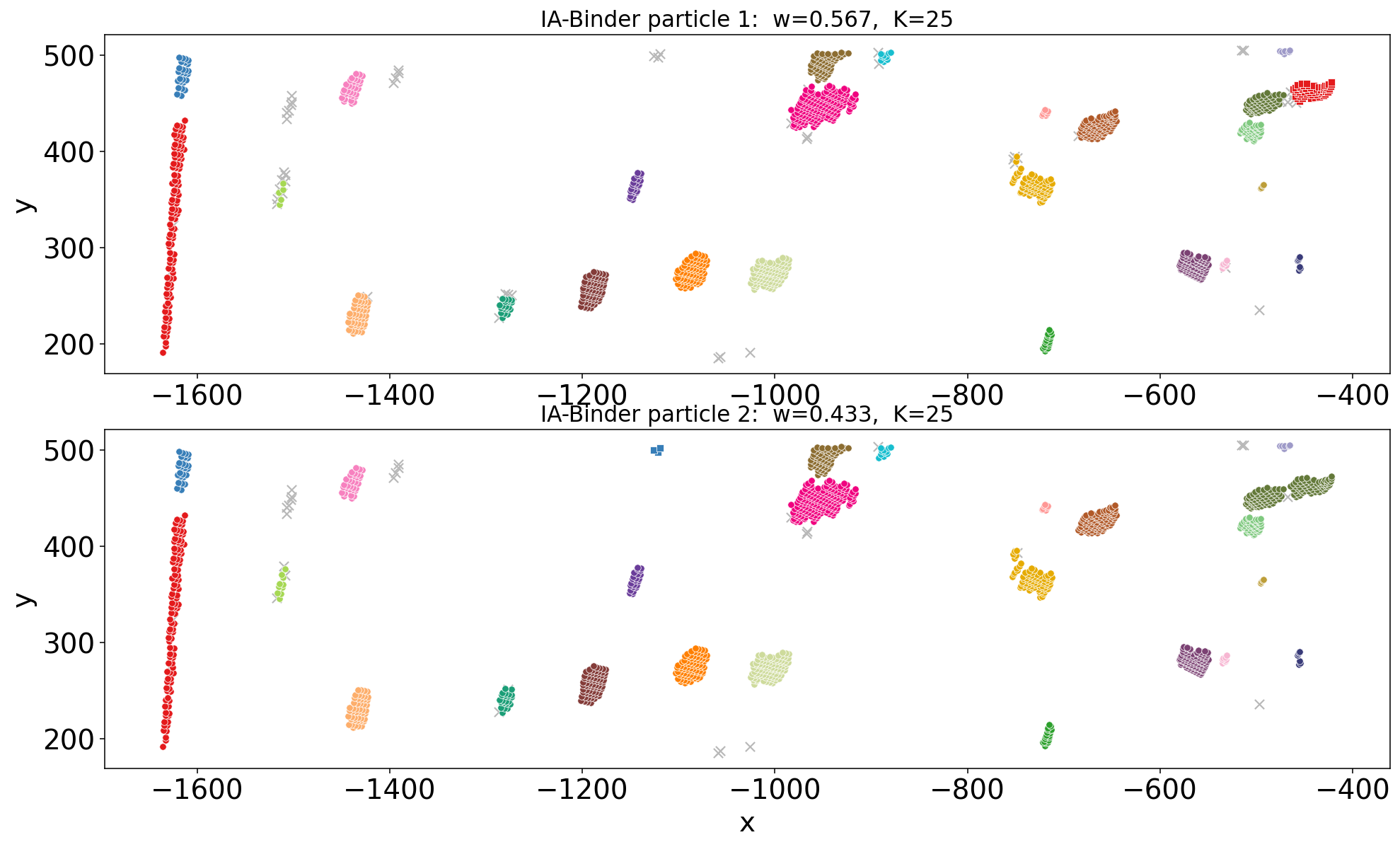}
    \caption{
    Two-particle BI$_{\mathrm{IA}}$ WASABI-SuP summary for the harbour data.
    Crosses denote
    inactive observations.
    }
    \label{fig:harbour-l2}
\end{figure*}

\end{document}